\documentclass{article}

\usepackage[margin=1.0in]{geometry}

\usepackage{amsmath, amsthm, amssymb, amsfonts}
\usepackage{euscript}
\usepackage{algorithm}
\usepackage{algpseudocode}

\usepackage{authblk}

\usepackage{graphicx}

\usepackage{url}
\usepackage{natbib}

\usepackage{tocbibind}
\usepackage[toc]{appendix}

\usepackage{tabularx}
\usepackage{booktabs}
\usepackage{longtable}
\usepackage{multirow}

\usepackage{verbatim}

\newtheorem{theorem}{Theorem}
\newtheorem{proposition}{Proposition}

\DeclareMathOperator*{\argmax}{arg\,max}
\DeclareMathOperator*{\argmin}{arg\,min}

\usepackage{fancyhdr}
\usepackage[linktocpage]{hyperref}

\author{Christopher J.\ Hazard}
\author{Michael Resnick}
\author{Jacob Beel}
\author{Jack Xia}
\author{Cade Mack}
\author{Dominic Glennie}
\author{Matthew Fulp}
\author{David Maze}
\author{Andrew Bassett}
\author{Martin Koistinen}
\affil{Howso Incorporated}

\title{A Theory of the Mechanics of Information: \\ Generalization Through Measurement of Uncertainty \\ (Learning is Measuring)}

\begin{document}

\maketitle

\begin{abstract}
Traditional machine learning relies on explicit models and domain assumptions, often limiting flexibility and interpretability.  We introduce a ``model-free'' framework of instance-based learning using surprisal (information theoretic uncertainty) to directly analyze and perform inferences from raw data, eliminating distribution modeling, reducing bias, and enabling efficient updates including direct edits and deletion of training data.  By quantifying relevance through uncertainty, the approach enables generalizable inference across tasks including generative inference, causal discovery, anomaly detection, and time series forecasting.  Our approach emphasizes traceability, interpretability, and data-driven decision making, offering a unified, human-understandable framework for machine learning, can be adapted to virtually all machine learning tasks, achieving competency across all and state-of-the-art performance on anomaly detection and some supervised tasks.  The mathematical foundations create a ``physics'' of information, which enable these techniques to apply effectively to a wide variety of complex data types, including missing data.  Our results suggest that this may be a viable foundation upon which to build scalable machine learning and artificial intelligence that can maintain human understandability of the underlying mechanics.
\end{abstract}

\tableofcontents

\section{Introduction}

Traditional machine learning methods often rely on explicit models of data distributions, which can obscure trade-offs between accuracy, interpretability, and computational efficiency.  These models also require domain-specific assumptions, limiting their adaptability to novel or complex tasks. This paper introduces a model-free framework that operates directly on data using surprisal, a fundamental concept from information theory, to quantify uncertainty and relevance.  Unlike conventional methods that require explicit modeling of data distributions, our approach operates directly on raw data, eliminating the need for domain-specific assumptions and reducing the risk of model bias.  This shift allows for more flexible and generalizable inference across diverse tasks, such as causal discovery and anomaly detection.

Advances in information theory and probability theory over the past two decades have provided essential insights toward understanding relationships among data.  Despite these advances, there remains a significant gap between the theoretical foundations of probabilistic inference and the practical needs of machine learning.  Most existing approaches to machine learning require careful tuning of hyperparameters, specification of loss functions, and selection of models, which can be time-consuming and require significant expertise. Further, most machine learning algorithms are designed for a specific objective function, rather than providing a general framework for learning and inference.  In this paper, we propose a new approach to machine learning that addresses these limitations by providing a unified, probabilistic framework for learning and inference that operates directly on the data.

Surprisal, or self-information, quantifies the information associated with observing a particular data element.  It is mathematically defined as the negative logarithm of the probability of an event, making it a natural tool for measuring uncertainty in data.  In our techniques, we frequently determine how surprising it would be if you could correctly use one particular data point in place of another because they were sufficiently similar.  If two data points are very similar, it would be unsurprising that you could use the second data point in place of the first to make a prediction about the first.  Surprisal offers a way to easily transform between distance and uncertainty, and surprisal is a flexible tool that enables many types of inferences and results that carry uncertainty characterizations with it.

Performing inference directly from the data is far from a new concept.  k-Nearest Neighbors (kNN) is one of the oldest, simplest, and most robust algorithms for pattern classification and regression models~\citep{HastieStatisticalBook}. It is a simple technique that is easily implementable~\citep{AlpaydinVoting}.  Outputs from the technique are usually traceable back to the exact data that influenced each decision.  This traceability enables detailed analysis of the decision inputs and characterization of the data local to the decision.  The terms kNN, \emph{case based reasoning} (CBR), and \emph{instance-based learning} (IBL) have been used to describe such algorithms, and which term is used depends on whether the focus is on the algorithm (kNN), reasoning (CBR), or retrieval of data in high dimensional spaces (IBL).  The techniques we describe here relate to kNN, CBR, and IBL, as they are all related to finding the relevant data.

k-Nearest Neighbors was previously a dominant machine learning technology~\citep{coomansAlternative, BreimanTree, AltmanIntroduction, AlpaydinVoting} but has increasingly been marginalized for niche applications and relegated to specific topologies given the growing size of data and the computational complexity of finding the nearest $k$ points~\citep{Raikwalperformance, SchuhMitigating, ArabicArticles2008}.  Recently, these techniques have seen a resurgence for use in vector databases to find embeddings for neural networks~\citep{pan2024survey}.  The curse of dimensionality has also been known to adversely affect kNN~\citep{HastieStatisticalBook, IndykApprox, SchuhImproving, TaoQuality} and the selection of a distance function can be challenging~\citep{SuryaDist}.  Additionally features may have to be scaled or standardized to prevent distance measures from being dominated by one of the features.  The accuracy of kNN can be severely degraded by the presence of noisy or irrelevant features, or if the feature scales are not consistent with their relevance.  Finally, kNN requires a value of parameter k. If k is too small, the model may have low bias but be sensitive to noisy points and have too high of variance.  If k is too large, the neighborhood may include points from other classes and may have too little variance.
 
Our techniques search for relevant data based on the idea that learning is a process of measuring and characterizing uncertainty in the relationships among the data, rather than optimizing a specific objective function.  Nearest neighbor searches using uncertainty have been studied, both with stochastic distances~\citep{mason2021nearest} and with uncertainty around the data points~\citep{agarwal2016nearest}.  We improve on these techniques by using surprisal as a measure of uncertainty and by measuring the uncertainties among the data.  We develop a framework for machine learning that is flexible, scalable, and robust to outliers and missing data.  Our approach can be applied to a wide range of machine learning tasks, including supervised and unsupervised learning, reinforcement learning, and time series forecasting.  It also has the potential to provide a new foundation for machine learning that is more general, flexible, powerful, and human-understandable than existing approaches.  In the following sections, we will describe the theoretical foundations of our approach, demonstrate its effectiveness, and discuss potential future directions.  We emphasize that the primary goal is not to create a technique that is the best at one specific task.  Rather, our aim is to create a system that meets the following criteria: it is of sufficiently high performance that it is practically usable for nearly \emph{every} task without having to build different models and systems for every task; every step is human understandable and traceable back to the data, with units of measurement and uncertainty every step of the way; the system turns the process of learning and inference into managing data and deriving values in understandable ways rather than attempting to understand and manage models; and that it opens new research to unify fundamental understanding of machine learning to improve to state-of-the-art and beyond without needing to build models.

The software implementing the techniques described in this paper is available at \url{https://github.com/howsoai/}, with the primary interface Python module at \url{https://github.com/howsoai/howso-engine-py}, and example notebooks at \url{https://github.com/howsoai/howso-engine-recipes}.  Throughout this paper we will generally refer to the algorithms and techniques, but in empirical results or when discussing implementation, we will refer to it as Howso or the Howso Engine.  This work is the first comprehensive treatment of a research effort that started in 2011, with earlier iterations of some of the core ideas starting from kNN and iterating toward the findings described here~\citep{hazard2019natively, banerjee2023surprisal}.

This paper adopts a nontraditional structure to emphasize the interdisciplinary nature of our framework.  We begin with an intuitive summary, followed by foundational concepts, then progressively build toward advanced applications and scale, and follow with empirical results. Related work and future directions are largely integrated inline to maintain focus on the core contributions.  The appendices contain relevant proofs, a numerical example of deriving a prediction, and details of empirical evidence.

\section{Intuitive Summary of Techniques}
Due to the novelty of techniques, we will begin with an intuitive summary of our system before introducing more technical formalisms.

Imagine trying to predict the acidity of a particular grape from a particular farm.  Having acidity measurements of 3 grapes of similar size picked from the same vine on the same day would likely be a very informative prediction of the acidity of the first grape.  Having 3 apples from the same farm may offer little information about the grape's acidity, but knowing the latitude and longitude of the apple trees they were picked from may be very informative for predicting the street address of the farm.

For a given inquiry, any given piece of data has some probability of being informative.  Most data will have a virtually zero probability of being informative most of the time.  If we know the probabilities of data elements being informative, we can use those probabilities to find the data most likely to be relevant.  From these probabilities, we can compute the best estimate as well as the uncertainty and distribution around that estimate.

Our system estimates all of these uncertainties, and we do so using robust and assumption minimizing techniques around the average uncertainty, the mean of the absolute differences between expected and actual.  To estimate the uncertainties of features or columns of data, our system uses all other features including the feature itself to find the relevant data, makes a prediction, and finds the average difference between the prediction and actual value.  In the case of measuring grapes, this means it uses the acidity of the grape in question, in addition to that grape's other properties to predict its own values.  This ``overfitting'' (bootstrap self-prediction) is intentional as it is a proxy for the potentially relevant features that are not included in the data.  We use these uncertainties to compute the probability that one value is informative for another value, and then use these results to compute the probabilities that each feature predicts every other feature.  Combining these probabilities, we compute the overall probability that one record is informative for another.  There is no model besides the data and its uncertainties; the database is queried with these uncertainties to find any relevant result.  Every computation of probability of every inference can be understood, and every data element used in an inference can be traced back to its source.

Because the system is a database, features, records, etc.\ can be included or excluded in queries, or added to, edited in, or removed from the data.  Rather than using traditional cross validation techniques from data science, we can just leave a record out and predict it.  Sampling records with replacement provides good estimates of the quality of inference on unseen data that is similar to existing data.  Our system can compute how much individual features contribute to a prediction without including the biases included from models via traditional feature importance techniques.  Imagine features $A$, $B$, and $C$ are being evaluated in how they predict $T$.  We can measure $A$ predicting $T$, $A$ and $B$ predicting $T$, $A$ and $C$ predicting $T$, etc.\ and even predicting $T$ without any features, and determine each feature's contribution.

Further, our system can use the same approach to determine how features contribute to reducing the uncertainty in predicting another feature, including uncovering causal relationships.  Causality deals with the relationship between two or more events or variables where one event influences the occurrence of another. For instance, moving a lever can open a valve and allow water to flow.  This cause-and-effect relationship is asymmetric and not reversible.  Our system can detect this asymmetry by analyzing how the uncertainty in one variable changes when another variable is included in the prediction. By doing so, we can identify potential causal relationships among variables, which can have significant implications for understanding the underlying mechanisms driving a system and predicting the effects of changes to the system.  Further, our system can detect asymmetries in both uncertainty and predictability, revealing likely causal relationships and hidden influences (confounders), and the strength of those relationships.  This is in contrast to correlational techniques that treat relationships as two-way and can only find associations.

Many more machine learning and AI capabilities are possible and effective with this system.  Anomalies can be identified by finding data records whose probability of being informative to their neighbors is much less than their neighbors' probabilities of being informative for each other.  Anything in the system can be conditioned, that is, narrowed or focused to specific types of data, by specifying values.  This is akin to looking up records in a database but by probability instead of by values or value ranges.  Generative outputs are created by drawing from the probability distributions.  Reinforcement learning is implemented by finding the data records that are most informative, selecting from those records which best maximize rewards, and then drawing an output from that probability distribution.  Larger data sets can be addressed by automatically determining which records are sufficiently duplicative that their probability masses can be combined.

Time series capabilities extend this base of capabilities.  By deriving and including additional features which represent the lagged values (values from an earlier time step) and rates (change in value over time), the aforementioned capabilities can be leveraged for nearly any time series workflow, including forecasting and hypothetical modeling.  Time series strengthens the quality of causal results by assessing whether one change occurs before another.

Throughout this work, we will frequently leverage information theory, which provides notationally and computationally convenient techniques to handle these probabilities.

\section{Foundational Statistics of Uncertainty}
\label{sec:foundational_statistics}
Because our system is built around uncertainty, we begin by describing what forms of uncertainty we employ and how we characterize it.

When assessing uncertainty of a data element, it is common to assess the uncertainty of measurement, or measurement error.  This is classically performed by repeatedly performing some measurement and characterizing some moment about the mean of those measurements, typically variance.  Often, this uncertainty is assigned to the reliability of a measuring device or process.  Another form of uncertainty of a data element is the uncertainty of attempting to predict the value via a model.  This is typically called a residual or prediction uncertainty, and this residual is typically assigned to a particular prediction or a model if the residuals are aggregated.

In addition to where the uncertainty is measured, uncertainty is also often characterized by its most likely origin~\citep{hullermeier2021aleatoric}.  \emph{Epistemic} uncertainty is uncertainty that stems from not having enough knowledge or data, which can be reduced by obtaining more knowledge or data.  \emph{Aleatoric} uncertainty is the irreducible part of uncertainty that comes from underlying stochasticity that cannot be resolved by collecting more data, such as the outcome of a fair coin flip.  However, in many practical scenarios, the distinction between these two types of uncertainty is blurred.  For instance, a coin flip's randomness (aleatoric) could become predictable with advanced tools like high-speed cameras and physics simulations (reducing epistemic uncertainty).  Our framework aims to identify epistemic sources of uncertainty as described in Section~\ref{ssec:guided_feature_discovery}.

Beyond traditional categories, we introduce \emph{substitutability deviation} (often just \emph{deviation} for short), a novel measure of uncertainty that reflects how likely one data element can be replaced by another without compromising predictive or informative value.  This concept is critical for tasks like nearest-neighbor searches, where similarity between data points determines relevance.  It is not strictly epistemic or aleatoric, but rather represents the combination of the two given the data.  As an example of deviations, when using human height in reasoning about physiological information or clothing size, we often make measurements to the nearest centimeter or half inch and declare people to be the same height if within that tolerance.  A person's posture, shoe sole depth, and spine compression from sitting and standing can all impart differences on the person's height.\footnote{Undoubtedly, any data set with self-reported human heights is likely to have some skew in the rounding.}  In this case, the deviation for this kind of data may be roughly a centimeter or half inch.

The foundation of our techniques rely upon mean absolute error (MAE), the absolute value of the difference between the estimated and the observed.  Though standard deviation is widely used for measuring uncertainty, MAE can be considered to be a more suitable measure of risk and uncertainty than standard deviation in many situations.  MAE is more robust to outliers, better captures the potential for extreme losses, and is easier for humans to interpret due to approximately half of the errors lie on either side of the MAE~\citep{taleb2007blackswan}.

Building on these concepts of uncertainty, we use maximum entropy distributions whenever possible to characterize the local uncertainty of our data.  Maximum entropy refers to the idea of selecting a probability distribution that makes the fewest assumptions about the underlying data, given the constraints.  This approach is useful in situations where the data and information are limited, such as when assessing uncertainty locally, zoomed in to a small amount of data.  The maximum entropy distribution of the absolute values of residuals is the exponential distribution, as the residuals are described by a mean and are nonnegative.  This distribution is widely used in statistics and has several desirable properties, including being robust to outliers and easy to interpret.  When dealing with differences between two values, if the absolute value of the difference is used as the mean absolute deviation of the distribution of values, the Laplace distribution (which is two back-to-back exponential distributions) is the maximum entropy distribution.  We will rely on this relationship throughout our work.

To determine the surprisal of a particular observation, we use inspiration from cumulative residual entropy~\citep{rao2004cumulative}, which integrates the entropy over a range.  But instead of integrating, we treat the probability mass under the residual as the basis for surprisal.  We derive this surprisal from the cumulative density function (CDF) of the exponential distribution.  Given two observed values, $x_1$ and $x_2$, we assume the expected absolute difference between the values, $\delta$, is exponentially distributed and write this well-known CDF with regard. Using this $\delta$ as the deviation, this expression becomes
\begin{equation}
\label{eq:prob_value_with_observation}
P(\delta \leq |x_1 - x_2|) = 1 - e^{-\frac{1}{\delta} |x_1 - x_2|}.
\end{equation}
We wish to flip the question.  Given the expected difference between the values $\delta$, we would like to determine the probability that the difference between our observations $x_1$ and $x_2$ is less than the actual difference.  If the actual difference is less than the observed difference, that means the data is at least as relevant as it has been observed.  We find this by negating Equation~\ref{eq:prob_value_with_observation} as
\begin{equation*}
P(|x_1 - x_2| < \delta) = e^{-\frac{1}{\delta} |x_1 - x_2|}.
\end{equation*}
As surprisal is the negative logarithm of a probability, we can express the surprisal that we would observe from using the value of $x_1$ in place of $x_2$ as
\begin{align}
I(|x_1 - x_2| < \delta) &= -\ln( P(|x_1 - x_2| < \delta) ) \nonumber\\
& = -\ln\left( e^{-\frac{1}{\delta} |x_1 - x_2|} \right) \nonumber\\
\label{eq:cumul_resid_entropy}
&= \frac{|x_1 - x_2|}{\delta}.
\end{align}
In light of this, we observe that to find the surprisal of one value given another, we can simply divide the absolute value of the difference by the deviation.

We note that mean absolute error and mean absolute deviation (MAD) are closely related concepts.  MAE refers to the expected difference between actual and predicted when making a prediction or measurement, whereas MAD refers to the expected difference between values and their mean.  When using our techniques, this distinction can become blurred.  For example, determining the deviation to find the probability that a data point is informative for another could be considered MAE if trying to predict that value or MAD if characterizing the data.  We try to use the most appropriate term between MAE or MAD for each context, but a different way of viewing a given technique may suggest the other term.

\section{Definitions}
Here we define common variables throughout the formulae below.  This is meant as a lookup table for the remainder of the work.  Further, we define \emph{conviction}, $\rho$, as the ratio of the expected surprisal to observed surprisal as
\begin{equation}
\label{eq:conviction}
\rho = \frac{E I}{I}.
\end{equation}

\begin{itemize}
\item $\mathcal{F}$ is the set of all available features.
\item $\mathcal{C}$ is the set of all available cases or records or instances that are in the data set.
\item $F$ is a set of features, $J$ is used when a second set of features is needed.
\item $i$ and $n$ are individual cases, which may or may not be in the set of available cases.
\item $j$ and $t$ are individual features, with $t$ generally indicating a feature designated as a target for the particular operation.  In our implementation, we often refer to features being used as input or contextualization as ``context features'' and features being used as output via prediction, generation, or inference as ``action features'', as whether a feature is an input or output depends on what is being done rather than a built-in aspect of a model.
\item $x$ is a vector of values, with the subscript representing the feature or features.  It may be subscripted with $i$ or $n$ to indicate a specific record, either trained or not, or may be omitted if there's only one vector of values being considered (e.g., $x_F$).
\item $w_i$ is the weight of a given case, when applicable.  This is the probability mass that the particular case represents in the data when data reduction techniques are applied.  Weights are generally omitted from most of the formulae for simplicity, but may be included in nearly any formula involving case probability masses.  See Sections~\ref{ssec:surprisal_of_case} and~\ref{ssec:case_weights}.
\item $f_j$ is the function that estimates the value for feature $j$.
\item $C_{i,F}$ is the influential cases with respect to case $i$ with regard to features $F$.
\item $\mathcal{P}(C)$ denotes the power set of set $C$.
\item $I(x)$ is the self-information of $x$, also known as surprisal.\footnote{We use $I$ instead of $s$ to reduce the number of letters introduced.}
\item $I(x \mid y)$ is the information of $x$ given $y$.
\item $S_i$ is the surprisal contribution of $i$, indicating the expected surprisal of $i$ given its influential cases $C_i$.
\item $\sigma_i$ is the surprisal contribution conviction of $i$, which is the surprisal contribution contextualized as a conviction.
\item $\kappa$ is the value of a particular nominal class.  It may be used with a subscript, e.g., $\kappa_1$, to indicate one class versus another.
\item $\delta_{j}(x_{i,F})$ is the deviation, the mean absolute error of $j$ for the values of $x_i$ over set of features $F$ which indicates the uncertainty of whether the values can be considered the same for inference purposes.  This value is at least as large as the measurement error but not greater than the residual.  Note that $F$ does not exclude $j$.  If $x_j$ is a nominal value, then $\delta_j(x_F)$ or $\delta_{j, \kappa}(x_F)$ may be used to represent the deviation for the nominal class of $x_j$ or $\kappa$ respectively.  Further, for performance purposes, $E\left(\delta_{j}(x_{n,F}) \mid n \sim \mathcal{C}\right)$ is sufficient to yield strong results from most data sets.  See Section~\ref{ssec:deviations} for computing these values.
\item $d$ is the difference between two actual values, which may not be known.  When used by itself it is used as a variable, but in context of different observations it will be used as a function of those two data elements for a given feature, e.g.\ as $d_j(x_{i,j}, x_{n,j})$.  In many of our formulae, the real value for $d$ is unknown and so it is treated as a distribution.
\item $q_{j,t}$ is the probability that feature $j$ is informative for predicting feature $t$, which can be a function of $F$ or $x_F$, and if not, then it is the expected value of the probability.  See Section~\ref{ssec:deviations} for computing these values.
\item $\phi$ is a set of constraints used to subset the data.  For example, for a given query, $\phi$ could be $\left\{ x_j \in \{\kappa_1, \kappa_2\}, 0 \leq x_t \leq 5 \right\}$.  See Section~\ref{ssec:constraining_and_goals}.
\item $\Omega_j(x_F)$ is a goal function on feature $j$ that accepts a vector of values and evaluates to a value representing how close it is to a specific goal, where smaller values are closer to the goal.  Without loss of generality, goal features may maximize by simply negating a particular value or may attempt to achieve a value as close as possible to another by minimizing the difference.  If written with a set of features as $\Omega_F$, it means a collection of goal functions.  See Section~\ref{ssec:constraining_and_goals}.
\item $\rho$ is conviction, a ratio of expected surprisal to observed surprisal, which may be used as an input.
\item $k$ denotes a count of cases, typically indicating the cases that comprise the appropriate statistical bandwidth for a query.
\item $E$ denotes an expected value over the relevant domains.  $E_{LK}$ denotes the expected value assuming an uncertainty using the {\L}ukaszyk-Karmowski metric.  Based on the domain, this is either as defined by Equation~\ref{lk_difference} for Laplace distributions, or as $E_{LK}\left( |x_{i,j} - x_{n,j}| \mid x_{n,F} \right) =  \left|x_{i,j} - x_{n,j}\right| + \frac{2 \delta_j^2}{\left|x_{i,j} - x_{n,j}\right| + 2 \delta_j}$ for exponential distributions~\citep{lukaszyk2003metryka}, the maximum entropy distribution for distributions over positive real numbers given a mean, though solutions for other distributions could be used if appropriate.
\item $P$ denotes a probability, potentially conditioned.
\item $\eta$ is a generic variable used to represent a real value.
\item $\xi$ is a generic variable used to represent a nominal value.
\item $Q_\eta$ is a function that returns the quantile of the set passed in at quantile $\eta$.
\item $e$ is Euler's number.
\end{itemize}

\section{Inference via Surprisal as Distance}
\label{sec:data_relevancy_inference}

We begin by describing the core operations around determining data relevancy and performing discriminative and generative inference.

\subsection{Continuous and Ordinal Features Expected Differences}
The {\L}ukaszyk-Karmowski metric finds the expected difference between two values, given the distributions around those values~\citep{lukaszyk2003metryka,lukaszyk2004lkmetric}.  If there is uncertainty around a measurement, this means that even identical measurements still yield a small but nonzero expected distance due to that uncertainty.  This is written in their original notation as the expected distance $d$ between uncertainty distributions $X$ and $Y$ as
\begin{equation*}
d(X, Y) = \int_{-\infty}^\infty \int_{-\infty}^\infty |x-y| f(x) g(y) \, dx\, dy.
\end{equation*}
We assume the Laplace distribution for uncertainty.  This is a maximum entropy distribution given a mean absolute deviation, not unlike the Gaussian distribution being the maximum entropy distribution when given the variance.  We assume both probability density functions of the value distributions are parameterized with the same uncertainty, $\delta_{j}(x_F)$, as derived in Appendix~\ref{subsec:lk_laplace_derivation}.  Note that we write it as a conditional expected difference given that the uncertainty is characterized by the uncertainty around $x_{n,F}$.  Adopting the notation described above, this expected difference is
\begin{equation}
\label{lk_difference}
E_{LK} \left( |x_{i,j} - x_{n,j}| \mid x_{n,F} \right) =  \left|x_{i,j} - x_{n,j}\right| + \frac{1}{2} e^\frac{-|x_{i,j} - x_{n,j}|}{\delta_{j}(x_{n,F})} \left( 3 \delta_{j}(x_{n,F}) + |x_{i,j} - x_{n,j}| \right).
\end{equation}
When using this expected difference, it reduces the distinctiveness of differences that are within the uncertainty, but the effect is insignificant on differences that are far greater than the uncertainty.

\subsection{Continuous and Ordinal Features Marginal Surprisal}
Given that we have an expected distance between two values, we would like to determine how surprising it would be if one value turned out to be informative for the other.  In other words, we want to determine how surprised we would be if we could use one value in place of another given their uncertainties.  The uncertainties are the mean absolute deviation around the values, the primary parameter for the Laplace distribution.

As discussed in Section~\ref{sec:foundational_statistics}, the Laplace distribution is symmetric, which allows us to simplify the distribution of these mean absolute deviations to the exponential distribution.  We parameterize this Laplace error distribution as $L\left(x_{i,j}, \delta_j(x_{i,F}) \right)$, so the exponential distribution of the distribution of the unknown actual difference between values becomes $Exp\left( \frac{1}{\delta_j(x_{i,F})}\right)$.  We write the deviation for feature $j$ as $\delta_j$, and write the expected difference between observed value $x_{i,j}$ and another known value $x_{n,j}$ given a set of features $F$ as $E \left( \left|x_{i,j} - x_{n,j}\right| \mid F\right)$.  The features are relevant to determine what deviation is appropriate for the $x_{n,F}$.

From these aforementioned values, we can rewrite Equation~\ref{eq:cumul_resid_entropy} to express the surprisal that $x_{n,j}$ is informative to $x_{i,j}$ for the unknown difference $d_j(x_{i,j}, x_{n,j})$ as
\begin{equation}
\label{eq:cont_resid_entropy}
I\left(d_j(x_{i,j}, x_{n,j}) < E\left( \left| x_{i,j} - x_{n,j} \right| \mid x_{n,F} \right) \right) = \frac{E\left( \left| x_{i,j} - x_{n,j} \right| \mid x_{n,F} \right)}{\delta_j(x_{n,F})}.
\end{equation}

Equation~\ref{eq:cont_resid_entropy} includes the conditional surprisal that the distribution of the uncertainty is a Laplace distribution with the {\L}ukaszyk-Karmowski metric~\citep{lukaszyk2003metryka,lukaszyk2004lkmetric}.  To determine the marginal surprisal, we should remove this extra base surprisal.  This base surprisal of a case compared to itself, $I\left(x_{n,j} \mid F, x_{n,j}, L(x_{n,j}, \delta_j(x_{n,F})) \right) = \frac{0 + \frac{1}{2} e^0 \left( 3 \delta_j(x_{n,F}) + 0\right)}{\delta_j(x_{n,F})} = 1.5$ nats, and should be subtracted from the resulting surprisal.  We write this surprisal more compactly leaving out the unknown difference between values.  The complete formula for marginal surprisal for a continuous feature is thus
\begin{equation}
\label{eq:cont_resid_entropy_marginal}
I\left(x_{i,j} \mid F, x_{n,j} \right) = \frac{  \left|x_{i,j} - x_{n,j}\right| + \frac{1}{2} e^\frac{-|x_{i,j} - x_{n,j}|}{\delta_j(x_{n,F})} \left( 3 \delta_j(x_{n,F}) + |x_{i,j} - x_{n,j}| \right) }{\delta_j(x_{n,F})} - 1.5.
\end{equation}
The difference between surprisal of an observation versus marginal surprisal of an observation can be illustrated with a coin toss.  If the coin is believed to be fair and tossed in normal conditions (e.g., excluding landing on the edge), then $1$ bit of surprisal is expected.  However, if the coin lands on its edge, approximated to be a 1 in 6000 chance with an American nickel~\citep{murray1993probability}, a person would be surprised.  The probability of heads is thus $\frac{2999.5}{6000}$, tails is $\frac{2999.5}{6000}$, and edge is $\frac{1}{6000}$, the surprisal of observing a heads would be $1.00024$ bits, a tails $1.00024$ bits, and an edge $12.55075$ bits.  Thus, the marginal surprisal of the coin landing on its edge over the baseline expectation of a normal coin flip is $12.55075 - 1.00024 = 11.55051$ bits.

\subsection{Nominal Features Marginal Surprisal}
Finding the marginal surprisal of a nominal feature follows similar logic to that of continuous features.  First we find how surprising it would be if a given value $x_{n,j}$ were informative for the case we are trying to use it for, $x_i$, given the relevant deviations.  To obtain the marginal surprisal, we must also subtract out the minimum surprisal of all of the possible values, treating that as the baseline surprisal.  Typically this will yield zero marginal surprisal for the most probable match.  However, it may also yield zero marginal surprisal if several values are statistically indistinguishable.  Because nominal features' deviations are themselves probabilities, the marginal surprisal for nominal features can be expressed as
\begin{equation}
I(x_{i,j} \mid F, x_{n,j}) = -\ln \delta_j(x_{n,F}) - \min_\kappa \left( -\ln \delta_{j, \kappa}(x_{n,F}) \right).
\end{equation}

\subsection{Surprisal of One Case Given Another}
\label{ssec:surprisal_of_case}
For set of features $F$, the surprisal of case $i$ given case $n$ is
\begin{equation}
I(i \mid F, n) = \sum_{j \in F} I(x_{i,j} \mid F, x_{n,j}).
\end{equation}
This can also be written as
\begin{equation}
\label{eq:case_surprisal}
I(x_i \mid F, x_n) = \sum_{j \in F} I(x_{i,j} \mid F, x_{n,j}),
\end{equation}
which can be transformed to the probability of case $x_n$ being informative for case $x_i$ as
\begin{equation}
P(x_i \mid F, x_n) = e^{-I(x_i \mid F, x_n)}.
\end{equation}
If case weights, $w_n$ are used, then the weighted probability can be written as
\begin{equation}
\label{eq:weighted_case_probability}
P(x_i \mid F, x_n) = e^{-w_n \cdot \sum_{j \in F} I(x_{i,j} \mid F, x_{n,j})},
\end{equation}
which is equivalent to $w_n$ occurrences of the case each having the same probability given that the probability is exponentiated by the weight.  We note that independence of features is not assumed, and that lack of independence will be addressed by feature influence probabilities as will be discussed in Section~\ref{ssec:deviations}.

When computing the surprisal of one case given another with regard to a specific target, we need to consider the probability that a given feature, $j$, is informative to the target feature being predicted, $t$.  The probability of case $x_n$ influencing inference for case $x_i$ using features $F$ when performing inference to predict feature $t$ is
\begin{equation}
P(x_i \mid F, t, x_n) = P\left(x_i \mid F, x_n\right) \cdot q_{j,t}(x_F)
\end{equation}
where
\begin{equation}
\label{eq:targeted_prob_case_given_another}
P\left(x_i \mid F, x_n\right) = e^{- I(x_i \mid F, x_n)}.
\end{equation}
This means that the expected case surprisal for predicting target $t$ is
\begin{equation}
\label{eq:targeted_case_surprisal}
I(x_{i,t} \mid F, x_n) = \sum_{j \in F} q_{j,t} I(x_{i,j} \mid F, x_{n,j}).
\end{equation}
Equation~\ref{eq:targeted_case_surprisal} can be transformed like Equation~\ref{eq:case_surprisal} to the probability of case $x_n$ being informative for case $x_i$ for feature $t$ as
\begin{equation}
P(x_{i,t} \mid F, x_n) = e^{-I(x_{i,t} \mid F, x_n)}.
\end{equation}
Adding surprisals is the logarithm equivalent of multiplying probabilities, and so multiplying a surprisal by a probability of influence is the same as raising the surprisal's corresponding probability to the probability of influence.  This yields the intended behavior that two features which have the same informativeness and probability of influence contribute the same as multiplying together two probability terms, as the former is the square of the probability.  From this perspective, Equation~\ref{eq:targeted_case_surprisal} closely resembles the conjunction of probabilities of independent distributions that form the core of the no-flattening theorem by \cite{lin2017does}, suggesting that there may be a relationship between our work and how hierarchical neural network architectures behave.

With fixed positive deviations and fixed nonnegative feature weights, the unsmoothed continuous dissimilarity is a weighted L1 pseudometric, and is a metric on coordinates assigned positive weight. The unshifted LK expected absolute difference retains symmetry and the triangle inequality under compatible uncertainty distributions, but has positive self-distance. Subtracting the LK self-distance produces a useful zero-baseline dissimilarity, but does not generally preserve the triangle inequality. Context-dependent, sparse, or directional deviations may additionally make the operational dissimilarity asymmetric. We therefore use “distance” operationally while reserving “metric” for the fixed-scale unsmoothed subset.

Under narrow considerations without LK smoothing, surprisals can be a distance metric.  With LK smoothing and at larger scales of difference they function approximately as one as proven in Appendix~\ref{subsec:surprisal_as_distance}.  Real-world needs such as context-dependent, asymmetric, and sparse deviations additionally push the behavior further from being a true metric.  As surprisal is sufficiently similar to distance metrics for practical use, we will use ``distance'' in the general sense and reserve ``metric'' for the fixed-scale unsmoothed subset.  Further, rectifying asymmetries by taking the minimum oer average enables these surprisals to be adapted for use in many distance-based machine learning techniques such as distance-based clustering techniques (e.g., HDBSCAN~\citep{mcinnes2017hdbscan}), and visualization techniques (e.g., UMAP~\citep{mcinnes2018umap}).

\subsection{Influential and Similar Cases}
\label{ssec:influential_cases}

The set of influential cases $C_{i,F}$ is the set of records that have minimal surprisal with regard to record $i$ for set of features $F$ and comprise the statistical bandwidth that is influential to $i$.  This set can be of size $k_{min}$ to $k_{max}$. Let $x_{(k)}$ denote the $k$-th case in $\mathcal{C}$ ordered by increasing surprisal $I(x \mid i)$.  The set $C_{i,F}$ is constructed by incrementally adding cases $x_{(k)}$ for $k \in [1, k_{max}]$ until a termination criterion is met.  The accumulation stops at $k$ if $k \ge k_{min}$ and the accumulated probability $\sum_{j=1}^k P(x_{(j)} \mid i) > 1.25$, provided that either the marginal probability $P(x_{(k)} \mid i) / \sum_{j=1}^{k-1} P(x_{(j)} \mid i) < t$ or the relative decrease in probability from the most similar case $(P(x_{(1)} \mid i) - P(x_{(k)} \mid i)) / P(x_{(1)} \mid i) \ge 0.2$, and $P(x_{(k)} \mid i) \neq P(x_{(k+1)} \mid i)$ to ensure the inclusion of equidistant cases. The set of influential cases is represented as
\begin{equation}
C_{i,F} = \{ x_{(k)} \in \mathcal{C} : 1 \le k \le \hat{k} \}
\end{equation}
where $\hat{k}$ is the smallest $k \in [k_{min}, k_{max}]$ satisfying the aforementioned termination criteria.  Note that by default, $k_{min}=1$, $k_{max}=\infty$, and $t = \frac{1}{e^3}$, and that ``similar cases'' are determined by specifying a fixed $k$.

\subsection{Discriminative Inference (React)}
\label{ssec:discriminative_predictions}
Discriminative inference, often contextualized as a prediction, produces an expected value for a feature $j$ given the context of features $F$ and feature values $x_F$.  Our techniques take into account the influential cases and the probabilities that they are informative to predicting $j$ given $x_F$.  For continuous and ordinal features, the expected value for case $i$ is found over the support of all of the influential cases, $C_{i,F}$, by finding the average value given the probabilities of influence as
\begin{align}
\label{eq:discr_predict}
f_t(x_{i,F}) & = E\left( x_t \mid x_{i,F} \right) \nonumber \\
& = \frac{1}{\sum_{n \in C_{i,F}} P(x_{i,t} \mid F, x_{n,F})} \sum_{n \in C_{i,F}} P(x_{i,t} \mid F, x_{n,F}) \cdot x_{n,t}.
\end{align}
If case weights, $w_n$, are used then this equation becomes
\begin{equation}
\label{eq:weighted_discr_predict}
f_t(x_{i,F}) = \frac{1}{\sum_{n \in C_{i,F}} P(x_{i,t} \mid F, x_{n,F})^{w_n} } \sum_{n \in C_{i,F}} P(x_{i,t} \mid F, x_{n,F})^{w_n} \cdot x_{n,t}.
\end{equation}
For ordinal features, the predicted value is rounded to the nearest ordinal.

For nominal features, the expected value is the mode of the probability masses of the values of the influential cases.  We combine the probabilities of each of the influential cases and accumulate their masses by each value, $\xi$, finding the value with the largest probability as  
\begin{align}
f_t(x_{i,F}) & = E\left( x_t \mid x_F \right) \nonumber \\
&= \argmax_\xi \sum_{n \in \left\{ n^{\prime} \in C_{i,F} \mid x_{n^{\prime},t} = \xi \right\} }  P(x_{n,t}).
\end{align}
Like continuous features, if case weights are used, then the expected value becomes
\begin{equation}
f_t(x_{i,F}) = \argmax_\xi \sum_{n \in \left\{ n^{\prime} \in C_{i,F} \mid x_{n^{\prime},t} = \xi \right\} }  P(x_{n,t})^{w_n}.
\end{equation}

\subsection{Prediction (React) Residuals}
We use the definition of residuals as the mean absolute expected difference between the predicted and actual value.  The residual for feature $j$ can be computed for an existing data element $x_i$ for features $F$ as
\begin{equation}
r_j(x_{i,F}) = E\left( \left|f_j(x_{n,F}) - x_{n,j} \right| \mid n \sim C_{x_{i,F}} \setminus \{i\} \right)
\end{equation}
if the feature $j$ is continuous.  Determining the expected residual across many cases can be computationally expensive, and we have found that using the weighted mean absolute deviation of a prediction matches the actual error more closely than the average of the residuals found by leave-one-out predictions for many variants of nearby cases (e.g., $C_{i}$, as well as fixed values such as the 30 nearest cases).  The weighted mean absolute deviation of the prediction can be expressed as
\begin{equation}
r_j(x_{i,F}) \approx  \frac{1}{\sum_{n \in C_{i,F}} P(x_{i,t} \mid F, x_{n,F})} \sum_{n \in C_{i,F}} P(x_{i,t} \mid F, x_{n,F}) \cdot \left| x_{n,t} - E\left( x_t \mid x_{i,F} \right) \right|.
\end{equation}

If $j$ is nominal, then the residual can be computed as
\begin{equation}
r_j(x_{i,F}) = E\left( 1 - P\left(x_{n,t} \mid C_{x_{i,F}}, F, x_{i,F}\right) \mid n \sim C_{x_{i,F}} \setminus \{i\} \right).
\end{equation} 
We note that if $i \in \mathcal{C}$, then these formulae are computing the residual of a record already in the data set, and so it must be excluded as to not influence itself.  If $i \notin \mathcal{C}$, then nothing will be removed.  Further, modifying these formulae to account for case weights can be done in the same manner as described in Section~\ref{ssec:discriminative_predictions}.

\subsection{Generative Inference (React)}
\label{ssec:generative_react}
When performing a generative inference, the default assumption is that it should follow the marginal distribution of the prediction.  However, this sharpness of the marginal distribution can be controlled by changing the conviction, $\rho$ from the default value of 1 which follows the marginal distribution.  Given a desired conviction $\rho$, a generative value for a continuous feature $j$ given a context $x_{i,F}$ can be drawn from the Laplace distribution from the relevant uncertainty as 
\begin{equation}
g_j(x_{i,F}, \rho) \sim L\left( x_{n,j}, \frac{r_j(x_{n,j})}{\rho} \right), n \sim \left\{ n \in C_{i,F} : P\left(x_{i,j} \mid F, x_n \right) \right\}.
\end{equation}
For nominal features, drawing generative for the context $x_{i,F}$ can be written as
\begin{equation}
g_j(x_{i,F}, \rho) =
\begin{cases} 
x_{n,j}, n \sim P(x_{n,t} \mid C_i, F, x_{i,F}) & \text{if } \eta < \frac{r_j(x_{i,F})}{\rho}, \eta \in U(0, 1) \\
x_{n,j}, n \sim \mathcal{C} & \text{if } \frac{r_j(x_{i,F})}{\rho} \leq \eta < 1 - \left(1 - \frac{r_j(x_{i,F})}{\rho} \right)^2 \\
x_{n,j} \sim U\left(\left\{x_{n,j} \mid n \in \mathcal{C}\right\}\right) & \text{if } 1 - \left(1 - \frac{r_j(x_{i,F})}{\rho} \right)^2 \leq \eta
\end{cases}.
\end{equation}
Like residuals, modifying these formulae to account for case weights can also be done in the same manner as described in Section~\ref{ssec:discriminative_predictions}.

\section{Influential Features and Causal Discovery}

In this section, we describe how our methods determine how features contribute to predictions and reductions in uncertainty, and link those results to the discovery of causal relationships.  Because the system is a database rather than a model, techniques such as leave-one-out validation, where an individual record is held out and predicted, are very efficient.  Performing bootstrap sampling with leave-one-out evaluation (selecting cases to be held out randomly with replacement) is an effective way to measure performance aspects of predicting a feature in our system, as validated by experiments performed with extra holdout data.  We leverage our system's efficiency at leaving out records and leaving out features in queries throughout the techniques described in this section.

\subsection{Prediction Contributions}
\label{ssec:prediction_contributions}

Prediction contributions quantify the change in an action feature when each feature (feature prediction contributions) or case (case prediction contributions) is included versus not omitted from the query.  When feature prediction contributions are applied in a robust fashion, this means that the samples are collected over the power set of all possible combinations of features.  Prediction contributions are the Banzhaf values~\citep{pmlr-v180-karczmarz22a}, which are a more robust alternative to the SHAP feature importance measure~\citep{lundberg2017unifiedapproachinterpretingmodel}.  Our implementation addresses many of SHAP's common issues~\citep{kumar2020problems}.  Because features are actually removed during queries, the resulting values are a ground truth of feature importance, the same as removing the feature and retraining~\citep{hooker2019benchmark}, which eliminates inductive biases.  For example, if feature $f_1$ and feature $f_2$ are highly correlated, a gradient boosted random forest model may remove or reduce the influence of $f_2$, when with respect to the data, both $f_1$ and $f_2$ are equally informative.  Prediction contributions ensure that such relationships are surfaced.

Prediction contributions can be computed either directionally or via absolute values.  Directional values show whether the inclusion of a feature increases or reduces a value, and so the sum of the directional prediction contributions over the features will approximate the mean value of a feature.  Absolute value prediction contributions indicate the magnitude of impact a feature has on a particular prediction.  Prediction contributions can be computed across an entire data set, on a conditioned subset of a dataset, or on the data for a particular prediction.  For simplicity, we write the unconditioned prediction contributions across features, though any aspect of the formulae can be conditioned, and they may be rewritten to apply to cases from some set of cases.

The directional prediction contributions of feature $j$ predicting a target feature $t$ are obtained across the set of features $\mathcal{F}$ in the manner of computing Banzhaf values~\citep{pmlr-v180-karczmarz22a} as
\begin{equation}
DPC_{j,t} = E\left( f_t(x_{F \cup \{j\}}) - f_t(x_F) \mid F \sim \mathcal{P}(\mathcal{F} \setminus \{j, t\}) \right).
\end{equation}
We note that the primary difference between SHAP values and Banzhaf values is that Banzhaf values do not include the inverse binomial weighting.  The corresponding absolute value prediction contributions are obtained as
\begin{equation}
\label{eq:abs_pred_contrib}
APC_{j,t} = E\left( \left| f_t(x_{F \cup \{j\}}) - f_t(x_F) \right| \mid F \sim \mathcal{P}(\mathcal{F} \setminus \{j, t\}) \right).
\end{equation}

\subsection{Accuracy Contributions}
\label{ssec:accuracy_contributions}

The techniques in Equation~\ref{eq:abs_pred_contrib} can be applied to the uncertainties as well, in the form of mean absolute error.  Accuracy contributions measure how much the inclusion of feature $j$ reduces the uncertainty when predicting a given target feature $t$.  Computing this robustly over the power set of all feature coalitions can be written as
\begin{equation}
\label{eq:acc_contrib}
AC_{j,t} = E\left( \left|f_t(x_F) - x_t \right| - \left| f_t(x_{F \cup \{j\}}) - x_t \right| \mid F \sim \mathcal{P}(\mathcal{F} \setminus \{j, t\}) \right).
\end{equation}

\subsection{Robust Residuals}
\label{ssec:robust_residuals}

To characterize the uncertainty of prediction and accuracy contributions, we examine the mean absolute error of predicting target feature $t$ over the power set of possible features.  These residuals will be at least as large as the residual of predicting the target feature $t$, and represent useful information of the uncertainty of prediction and accuracy contributions.  To signify that these residuals are computed over the power set, we call these \emph{robust} residuals, and denote them as
\begin{equation}
\label{eq:robust_residuals}
rr_{t} = E\left( r_t(x_{n,F}) \mid F \sim \mathcal{P}(\mathcal{F} \setminus \{t\}), n \in \mathcal{C} \right).
\end{equation}

\subsection{Information of Accuracy Contribution}

Using the same derivation from Equation~\ref{eq:cont_resid_entropy_marginal}, we can transform the uncertainties of the accuracy contributions from Equation~\ref{eq:acc_contrib} with regard to their differences from zero into surprisals using the uncertainty characterized by robust residuals from Equation~\ref{eq:robust_residuals}.  Accuracy contributions are measured across the power set of features, so to keep the uncertainties on the same scale, robust residuals fits better than residuals or deviations.\footnote{There are a few circumstances where other residuals still give slightly better results than robust residuals, particularly for feature influence probabilities as derived in Section~\ref{ssec:deviations}.  This is an area for future investigation.}  We call this value the information of accuracy contribution and compute it as
\begin{equation}
\label{eq:iac}
IAC_{j,t} = \frac{ AC_{j,t} + \frac{1}{2} e^\frac{-AC_{j,t}}{rr_t} \left( 3 rr_t + AC_{j,t} \right) }{rr_t} - 1.5.
\end{equation}
These $IAC$ values can be assembled into a matrix comparing all features to all other features.  Causal relationships require asymmetries in the entropy relationships between features, in particular, see Definition 8 in the work of \cite{simoes2023causal}.  These asymmetries can be found as
\begin{equation}
\label{eq:iaac}
IAAC_{j,t} = IAC_{t,j} - IAC_{j,t},
\end{equation}
and may be computed for every pair of features to characterize causal entropy apparent in the data.  These entropies may be left as entropies or transformed into probabilities for assessing those relationships with the strongest causal signal.  The values for $IAAC_{j,t}$ above a given threshold can be used as edges to characterize a portion of the underlying causal graph.  This technique relates to knockout features~\citep{candes2018panning}.  The difference is that knockout features are copies of the features with added noise to give a baseline to detect causal relationships, whereas our techniques attempt to assess the uncertainty directly and compare it from the opposite direction.

\subsection{Deviations and Feature Influence Probabilities}
\label{ssec:deviations}

Determining what neighbors are relevant in a high dimensional space is a challenging problem.  Solving this problem involves contextualizing what is locally relevant based on the data.  We extend the ideas of finding relevant features~\citep{hinneburg2000nearest} and data~\citep{beyer1999nearest} into probability space, achieving dynamic behaviors that resemble attention mechanisms~\citep{soydaner2022attention}.

Deviations, as described in Section~\ref{sec:foundational_statistics}, is the uncertainty measured in mean absolute error that surrounds each data element, measuring the probability that one value is representative of another.  For example, the difference of a centimeter is irrelevant when measuring the height of a redwood tree, as wind can move branches and other factors including time of day can affect the angle of the pine needles, whereas a centimeter may be relevant when measuring the height of an early-growth sapling.  We compute these deviations by sampling the data with replacement and determining the uncertainty in predicting each value given that the value itself was used to find the relevant data.  Using the value being predicted to find the relevant data, and then determining the mean absolute error of predicting the value, deliberately underestimates the residual with the intent of approximating features absent from the data that might help with the predictability.  In other words, deviations are an estimate of the lower bound of the expected residual of predicting a particular data element.  If the data does not have extremely long tails and more consistent uncertainty, the expected deviation as measured across the data set can be estimated via sampling with replacement as
\begin{equation}
\delta_{j} = E\left( r_j(x_{n, \mathcal{F}} \mid n \in \mathcal{C}) \right).
\end{equation}
For data sets that have long tails, highly variable uncertainty, or nominal classes of varying similarity, we can compute and store deviations that are locally relevant to the data as
\begin{equation}
\delta_{j}(x_{i,F}) = E\left( r_j(x_{n, \mathcal{F}} \mid n \in C_{i,F}) \right).
\end{equation}
Using local deviations or deviations for pairs of nominal classes can improve results but incurs additional performance and storage costs.  Empirically, we have found that using a single value for a feature's deviation works very well for many data sets.  In the case of nominal features, sometimes different nominal classes are unlikely to be confused with one another.  For example, if dealing with a text corpus using tokens or words, it is unlikely that the word ``tree'' would be confused with the word ``square'' in virtually any context, so storing a deviation between those two tokens would neither have sufficient statistical support nor would it be computationally efficient.  Conversely, the words ``tree'' and ``bush'' may be confused on occasion, for example, with regard to coffee plants.  Instead, we have implemented what we call a \emph{sparse deviation matrix} for nominal classes, where only the deviations that have sufficient statistical evidence are maintained, and anything that is not statistically significant we bundle together for the deviations of nominal values not found in the sparse deviation matrix.  The same technique is used for handling missing data values as well.

To determine how informative a feature is when making a particular prediction, we leverage a variation of its information of accuracy contribution, Equation~\ref{eq:iac}.  The reduction in expected absolute error one feature has on another is indicative of the probability that one feature is informative to predicting the other.  This information of accuracy contribution can be transformed back into the probability of a feature $j$ influencing feature $t$ using its deviation as
\begin{equation}
\label{eq:full_feature_influence_prob}
q_{j,t} = 1 - e^{-IAC_{j,t}}.
\end{equation}
We have found improved results when $IAC$ uses LK smoothing via the exponential distribution, finding the marginal surprisal by subtracting the 1 nat that accompanies the baseline surprisal given the exponential distribution.

However, Equation~\ref{eq:full_feature_influence_prob} assumes that all features are present.  Suppose we are using a some subset of features $F \in \mathcal{F}$ with values $x_F$ to predict feature $t$, meaning that features $\mathcal{F} \setminus F$ are inactive.  The influences from unused features relevant from the features used should be redistributed.  Imagine a simple data set where $x_t = x_1 + x_2$, where $x_1$ and $x_2$ are random numbers independently drawn from the same uniform distribution.  The feature influence probabilities $q_{1,t}$ and $q_{2,t}$ should both be $\frac{1}{2}$.  If we duplicated feature $x_2$ into 9 additional features, we would expect $q_{2,t} = \frac{1}{2} \cdot \frac{1}{10}$, and the same value for each of the other duplicated features.  We can approximate this for any combination of features by redistributing the indirect feature probability influences from each feature in $\mathcal{F} \setminus F$ to the features in $F$.  First, $q_{\mathcal{F} \setminus t,t}$ is normalized to obtain the set of probabilities of each feature predicting $t$ without using $t$.  Then, for each feature $j^\prime \in \mathcal{F} \setminus F$, we normalize $q_{F,j^\prime}$, and multiply each normalized probability by the respective $q_{j, j^\prime}$, and accumulate that onto each feature probability $j \in F$.  The final result should be normalized, but to ensure numerical precision after the additions, the resulting probabilities are normalized again.  In the previous example, this means that the feature influence probabilities of the duplicated features would be accumulated into the one feature that was used, and in more complex examples, the relevant probability mass is accumulated based on features' probabilities of influence.  Though this technique is not perfect at redistributing feature probability influences in complex graphs, empirically we have found this approximation to be sufficiently close to the ground truth and notably improves the quality of conditioned inferences.

When performing forms of inference that do not involve targets, such as assessing anomalousness via the total surprisal of the case or assessing how similar one case is to another, we average the feature influence probabilities across all of the features as targets.  For feature $j$, this is computed as $\frac{1}{|F|} \sum_{t \in F} q_{j,t}$.  The effect of this operation reduces the influence of features that do not add to any predictability of any other features.  Empirically, we have found this to improve the quality of anomaly detection by about 1-2 percentage points on area under the curve measures of precision and recall and ROC.

Putting all of the equations together, deviations and feature influence probabilities are computed using themselves recursively.  These deviations seem to frequently converge quickly to values that yield strong performance if a good initial heuristic value is used.  We start with the smallest gap between two values for the initial deviation for continuous features, $\frac{1}{|\mathcal{C}| + \frac{1}{2}}$ as a variant of Laplace's rule of succession for nominal features, and using probabilities of $1$ for each feature influence.  First we converge the deviations, and then compute the feature influence probabilities in one pass using initial feature influence probabilities of $1$.  As this is a nonlinear dynamics problem, we have found that some data distributions can be sensitive to those initial heuristically chosen values.  Determining how to best compute or converge these values is an open problem, and preliminary results suggest that improving this approach may be able to improve discriminative inference performance to or potentially past current state-of-the-art.  Future work may extend into dynamic probabilities of a feature predicting another feature as $q_{j,t}(x_F)$.

\subsection{Additional Inference Optimization for a Single Target}
\label{ssec:single_targeted}

Though the methods described in Sections~\ref{ssec:deviations} result in robust inference quality across nearly all types of machine learning for any combination of features, it is possible to obtain higher quality results to predict a single feature.  Here we describe one such way to boost this quality.  To implement our solution that optimizes for single targeted inference, the processes for obtaining deviations from Section~\ref{ssec:deviations} are maintained, though augmented with additional optionality to select best performance.  First, a Lebesgue parameter is introduced for combining feature surprisals beyond just $L_1$, including $L_{0.1}$, $L_{0.5}$, $L_1$, and $L_2$.  These additional topologies can better represent relationships among the data when much of the data is already measured from the respective topology.  For example, data that has significant spatial measurements in Euclidean geometry will benefit from using $L_2$.  The second change is that instead of choosing a dynamic bandwidth for each instance, a global $k$ is selected from the Fibonacci sequence among the range 3 to 144.  Third, when computing feature probabilities, using equal feature weights is explored in addition to those computed via accuracy contributions.  Our targeted optimization algorithm implements a grid search over these configurations with the deviations and feature probabilities otherwise being computed as described in Section~\ref{ssec:deviations}, though iterating the entire process, including the grid search, twice in order to improve convergence.  This single targeted algorithm can achieve cutting edge results.  Though the universal algorithm's results are quite strong and task independent, there is still a discrepancy.  See Section~\ref{ssec:supervised_results} for analysis of the empirical results.  Future work entails improving the methods of Section~\ref{ssec:deviations} to achieve higher inference quality without incurring the additional compute required for the grid search and while maintaining consistency on the probabilistic inference.

\subsection{Causal Direction and Guided Causal Discovery of New Features}
\label{ssec:guided_feature_discovery}

Because a feature's substitutability deviation is an estimate of a theoretical maximum of predictability, its residual is the measure of predictability given the data, and these values are on the same scale and units of measurement, we can use this information to estimate which features may benefit the most from having further information.  The missing certainty ratio (MCR) of residual to deviation measures the proportion of unexplained variability relative to the theoretical predictability limit.  MCR can be expressed for feature $j$ as
\begin{equation}
MCR_j = \frac{E(r_j)}{\delta_j}.
\end{equation}
During causal discovery we first identify feature pairs that exhibit strong asymmetric accuracy contributions (Equation~\ref{eq:iaac}).  Among these pairs we prioritize the feature with the larger MCR as the most promising candidate for augmenting the dataset with a new, informative feature.  As each new feature is discovered, the process may be repeated to evaluate whether and how the inclusion of a new feature fits in with regard to the rest of the system causally.  Additionally, when looking at causal relationships, the direction from larger to smaller MCR between the two features tends to indicate causal direction.

As an example, imagine we have a pond which is partly fed by a small river with a sluice gate.  Sluice gates are used to control water flow.  In this example, the sluice gate is used to help regulate the water level in the pond.  To illustrate this, see Figure~\ref{fig:sluice_gate} and imagine that the pond is to the left and the river is on the right.
\begin{figure}[ht]
\centering
\includegraphics[width=0.5\textwidth]{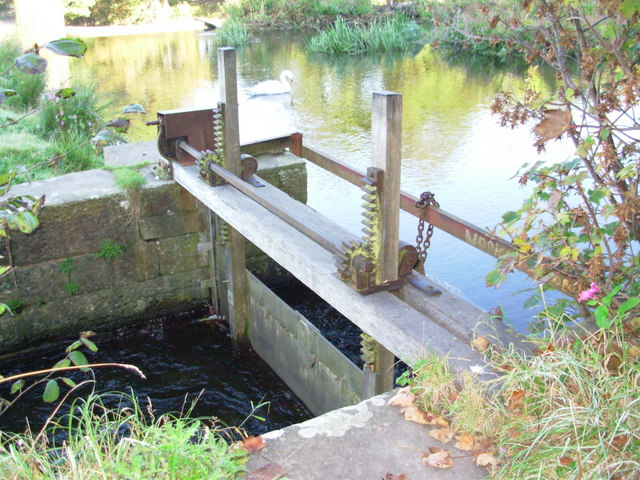}
\caption{Sluice gate (H Stamper / sluice gate / CC BY-SA 2.0, via Wikipedia).}
\label{fig:sluice_gate}
\end{figure}

Suppose initially we only have two features in the data, the water level of the pond measured in centimeters and the position of the crank as measured in degrees of rotation from the furthest closed position of the gate.  Further suppose we have collected data over a year, measured once per day.  When using these two features to predict themselves with regard to deviation, we find that they are both very predictable.  When using them to predict each other, we find an asymmetry in the accuracy contributions where the position of the crank reduces the uncertainty in predicting the water level more than the other way around.  This indicates that there is likely a causal relationship involving the crank and the water level.  The ratio of the residual of in predicting the position of the crank with and without the crank's current position is somewhat similar, though having the crank's information helps, so consider the ratio being a value larger than 1.  But the ratio of the residual of predicting the pond water level with and without the water's position (residual versus deviation) is still very large.  This indicates that there are more features that should be discovered, most likely involving a causal mechanism from the crank to the pond water level.

Investigating this further, suppose that we find there are some teeth missing on the gear and some chains that are used to manually fix the gate's position in place.  A new feature can be identified which is the position of the gate measured by its vertical position in centimeters.  Recomputing the residuals and deviations, we find that the ratio of the residual of the pond water level to the pond water level deviation is smaller but still large.  Given that we believe we have exhausted the obvious measurable features, we look to other features that might have a causal relationship with the pond water level.  Including the water level of the river seems plausible, since it will be higher when there are rains, and possibly some of the water in the river may be seeping into the pond through the soil.  With this new river water level, reassessing the residual-to-deviation ratio for pond level is much closer to 1, and the causal relationships discovered by the asymmetries in accuracy contribution show the appropriate directions of uncertainty.  But now, the ratio of residuals to deviations for the river water level is high, indicating there are more upstream causal relationships.  Thus, the sluice gate example demonstrates how a high MCR signifies a feature that would benefit from additional upstream measurements.

This incremental, data-driven approach to causal feature discovery facilitates systematic expansion of the causal graph, but greatly benefits from domain knowledge and controlled experiments.  Conditioning or constraining on particular subsets of the data can also help understand whether causal mechanisms change when systems are in different configurations or states.  Using a time series approach as described in Section~\ref{ssec:time_series} in conjunction with these techniques can help tremendously in discovering causal relationships, and we will show some empirical results from simulations in Section~\ref{ssec:empirical_causal_discovery}.

\section{Advanced Inference for Single Cases}

In this section we describe more complex ways to assess and characterize predictions and inferences for single cases at a time, such as when making individual predictions about a new data case or assessing the anomalousness of an existing case.

\subsection{Semistructured Data Types and Null Values}

In Section~\ref{sec:data_relevancy_inference} we primarily discussed nominal and continuous feature types.  However, any feature format that can be expressed as a probability can be included in inference.  Straightforward extensions include cyclic features like time of day, where a modulus operation can be applied to the difference operation of a continuous feature.  Continuous features can be rounded at any part during inference, and ordinal features can be implemented by a hybrid of continuous and nominal inferences.

More complex feature types can be supported.  For example, unordered sets can be compared based on the percentage of elements contained in them.  Strings can be compared by edit distance.  Trees and graphs can be compared by graph similarity metrics, such as graph edit distance.  And even formulae and code can be compared by first converting them into parse trees.  What is particularly powerful about these underlying techniques is that as long as there as a measure of difference and a way of combining values, the core principles can be applied.  For example, a discriminative result for comparing two sets may include values of the intersection of sets where the inclusion of an element is more probable than not.  Or a generative inference involving formulae may be a weighted recombination via genetic programming.

\begin{figure}[!ht]
\centering
\includegraphics[width=0.5\textwidth]{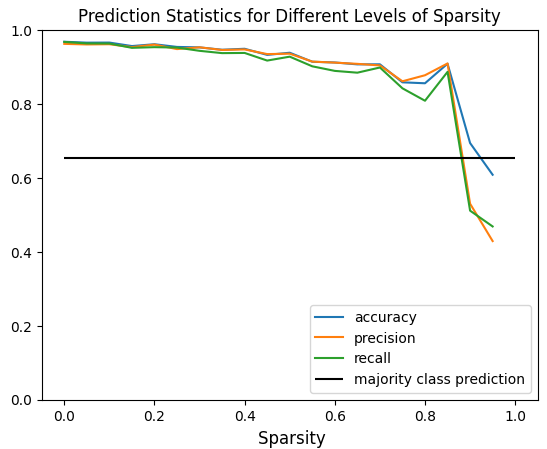}
\caption{An example of prediction performance where the horizontal axis represents the fraction of data elements that are null.}
\label{fig:nulls}
\end{figure}

We model missing data by using a special value indicating that no value was present, regardless of feature type.  We further keep deviations characterizing the probability of a value's being null or not null being informative as to other values.  For example, if missing data is correlated with a strong storm due to sensors being destroyed by wind, this surprisal can be computed and null sensor data may be informative to wind data.  This extra information can be handled efficiently without the need for storing an entirely separate feature.  Figure~\ref{fig:nulls} shows the prediction performance of our techniques applied to the breast cancer Wisconsin dataset~\citep{breast_cancer_wisconsin_dataset}, where the null percentage on the horizontal axis indicates the percentage of the elements in the entire data set (regardless of whether it is the target being predicted or any other element) that have been randomly selected and overwritten by null values.  The general performance stays strong in this small data set even when most of the elements are null, and this general robustness to sparsity is typical across data sets in our experimentation.

\subsection{Derived and Dependent Features}
\label{ssec:derived_and_dependent_feat}

Because every step of inference is computed sequentially, steps can be inserted in the middle of inference to derive intermediate results.  Suppose that one feature is known to be the sum of several other features, that a feature is known to be some complex function of other features, or that a feature's value is computed by a function specified by another feature.  Derived features can be particularly valuable in domains where the relationships between some features are well known and characterized.  This capability enables spreadsheet-like computations in the middle of inference, as well as opens the door for solving optimization problems.  Derived features can be interwoven between discriminative or generative outputs.

Many trade-offs are made when encoding and storing data.  A somewhat common way of storing extremely sparse data which has mutually exclusive combinations of features is known as data ``melting'' where data values are pivoted around identifying values to transform a data set with many columns into one with few columns and more rows.\footnote{For an example of this form of transform, see https://pandas.pydata.org/pandas-docs/stable/reference/api/pandas.melt.html.}  As an illustrative example, consider a healthcare data set that includes lab tests.  Suppose one patient had three tests performed and another patient had six tests performed.  One could have a column for the results of every test, but given the vast number of test types and units of measurement, as well as new tests being developed and released, the schema becomes complex to maintain.  Instead, the data is often organized where one column is the test name, one is the units of measurement, and a third is the value.  Instead of a patient visit having one row with all of the test data, there is one set of test data per row.  Performing any form of inference over data encoded this way can be slightly challenging, as similar values with different units of measurement can be confounding and it may take considerably more data for machine learning to learn those additional relationships.

To address the situation where features are reused for different purposes, we explicitly constrain these dependent features.  If feature $B$ is dependent on feature $A$, then once feature $A$ has a value, only data records that match $A$ will be considered for the influential cases.  If a feature is nominal or ordinal, then it can select only those that match, but if a feature is continuous, then it may consider only those matching the value range as a hard constraint rather than a similarity.  Continuing the aforementioned medical lab test example, when generating a set of sample values for lab test results, a lab test may be generated first, and then for generating subsequent features like units of measurement or measurement values, only cases that correspond to the specific lab test would be considered.

\subsection{Boundary Values}
Boundary values represent the smallest perturbation needed in an individual feature to change the outcome.  When dealing with nominal features, boundary values are often called counterfactuals.  For a given target feature that is being predicted $t$, we wish to find the smallest change in the feature $j$ that yields a different nominal class for the target.  This can be expressed as
\begin{equation}
BV_j(x_F, t) = \argmin_{\eta \in \mathbb{R}} \{ |\eta| : f_t(x_{F \setminus \{t\}}^{(x_j \leftarrow x_j + \eta)}) \neq x_t \}.
\end{equation}
When target feature $t$ is continuous, then we characterize the boundary value as the change in feature $j$ required to move the prediction to the edge of the mean absolute error of the prediction, $r_t(x_F)$, expressed as
\begin{equation}
BV_j(x_F, t) = \argmin_{\eta \in \mathbb{R}} \{ |\eta| : \left| f_t(x_{F \setminus \{t\}}^{(x_j \leftarrow x_j + \eta)}) - f_t(x_F) \right| > r_t(x_F) \}.
\end{equation}

\subsection{Boundary Cases}
Boundary cases are the individual records closest to a given record or prediction that have the most rapid change over some set of features.  They are cases that are just past the edges of ``plateaus'' in the data, finding the cases that best represent where the data changes most suddenly.  The top $k$ boundary cases are found relative to $x_{i,F}$ by varying feature set $J$ by computing
\begin{equation}
BC(x_{i,F \cup J}, F, J, k) = \underset{i \in \mathcal{C}}{\operatorname{argmax}}^{k} \frac{ I(i, n \mid F) }{ I(i, n \mid F \cup J) }.
\end{equation}

\subsection{Data Constraining and Goal Features}
\label{ssec:constraining_and_goals}

Inferences can be conditioned on any subset or combination of features by setting certain feature values.  Inferences can also be constrained to apply only to data meeting certain criteria.  For example, a prediction can be made using the constraint that it only infers from data where $0 \leq x_{1} \leq 5$.  A set of constraints, $\phi$, can be added to condition an inference or probability calculation.  For example, to obtain a discriminative prediction with constraints $\phi$, Equation~\ref{eq:discr_predict} becomes
\begin{equation}
f_t(x_{i,F} | \phi )  = E\left( x_t \mid x_F, \phi \right).
\end{equation}

One aspect of constraining data that can be particularly challenging is when constraints are relative to other data.  For example, a ``small'' fruit on a watermelon farm may be of a very different size than a ``large'' fruit on a blueberry farm.  We introduce \emph{goal features} as a way to further condition data that has already been found to be similar via conditioning or constraints.  Goal feature values are obtained by minimizing the function $\Omega_J(x_F)$ for the set of features $J$ from the influential cases of $x_F$, $C_{x_F}$.  The values obtained by $\Omega_J(x_F)$ can then be concatenated on to the values of $x_F$ to provide a new condition or constraint of $\Omega_J(x_F) \cup x_F$ yielding values for features in the set $J \cup F$.  We discuss reinforcement learning further in Section~\ref{ssec:reinforcement_learning_active_inference}, but goal features are an important part of reinforcement learning using our techniques as they determine what data would maximize the objective given a particular observed world state.

More formally, a goal value $GV$ for a given feature $j$ is obtained by finding the influential cases to case $i$ for features $F$, and finding the value that minimizes the distance to goal function, $g$.  This can be used in any part of inference in place of a given element $x_{i,j}$, and can be represented as
\begin{equation}
GV_j(x_{i,F}, \Omega) = x_{ \left( \argmin_{n \in C_{i,F}} \Omega_j(x_{n,j})\right), j}.
\end{equation}
Goal values can be used to condition queries, where a reasonable goal is chosen for the data most relevant to $x_F$.

Additionally, goal features can be used to select influential cases and to compute the relevant surprisals to the cases which best achieve the goal for the soft ``state'' of the other features.   The goal cases for a given feature can be represented as
\begin{equation}
GC_j(x_{i,F}, \Omega) = \argmin_{n \in C_{i,F}} \Omega_j(x_{n,j}).
\end{equation}

\subsection{Residual Conviction: Comparing Anomalousness of Features}
\label{ssec:residual_conviction}

When determining anomalousness, it is necessary to compare different features and combinations of features.  For example, an anomalous grape might be extraordinarily large, extremely sour, or unusually shaped.  But each of these three features has its own units of measurement.  And an unusually large grape might be comparable in size to an unusually small apple, making anomalous size contextual to the fruit.  To solve both the issues of contextuality of the anomalousness and units of measurement, we determine the ratio of how surprising we expected a feature to be to how surprising it actually was, all in context of the envelope of the expected uncertainty.  If we say that a grape is as large as it is sour, and the grape is the size of a very small apple, then we would expect the grape to be both very sour and large, thus anomalous.

Following the derivation above for residual entropy in Equation~\ref{eq:cont_resid_entropy} of assuming a Laplace distribution of the error, we can write the surprisal of the residual as
\begin{equation*}
RS_{i,j,F} = \frac{ E_{LK}\left( \left|f_j(x_{i,F}) - x_{i,j} \right| \mid \delta_{j}(x_{i,F}) \right) }{E_{LK}\left( r_j(x_{i,F}) \mid \delta_{j}(x_{i,F}) \right)}.
\end{equation*}
To convert this value to conviction, we assume that the expected surprisal is $1$ and take the reciprocal following Equation~\ref{eq:conviction}.  The residual conviction can be represented as
\begin{equation}
RC_{i,j,F} = \frac{E_{LK}\left( r_j(x_{i,F}) \mid \delta_{j}(x_{i,F}) \right)}{E_{LK}\left( \left|f_j(x_{i,F}) - x_{i,j} \right| \mid \delta_{j}(x_{i,F}) \right)}.
\end{equation}

Residual conviction can be used to diagnose anomalous values within a data set, specifically to diagnose what makes a particular data element different from and unexpected relative to the relevant comparable data.

\subsection{Surprisal or Distance Contribution of a Case}
\label{ssec:dist_contrib}

When including a new data point to an existing dataset, the data point contributes its own probability mass to the probability mass distribution in the location of the data point, which alters the structure of the data.  This can be thought of as analogous to adding a new thread to a tapestry, where the new strand becomes an integral part of the existing fabric; the thread's color, thickness, and strength will all impact the tapestry.  Similarly, adding a new note to a musical composition can influence the overall harmony and discord.  In this context, we can consider the ``distance contribution'' of a data point, which represents the extent to which it modifies the existing distribution, and thereby affects the detection of anomalies and outliers.  Because distance is surprisal, we can also refer to these values as ``surprisal contributions''.  By examining this contribution, we can gain a deeper understanding of how individual data points interact with and shape the overall dataset, and develop more effective methods for identifying the usefulness and validity of data points, as well as characterize unusual patterns.  Surprisal contributions are a foundation for various kinds of anomaly detection, and anomalies can be detected by using the surprisal contributions as a measure of how surprising a case is given the relevant data.

The surprisal contribution for case $i$ can be computed from $i$'s influential cases $C_{i,\mathcal{F}}$ as
\begin{equation}
\label{eq:surprisal_contribution}
S_i = E\left( I(x_i \mid x_n) \mid n \sim C_{i,\mathcal{F}} \right).
\end{equation}
Because this is an expected surprisal, surprisal contributions can be considered a form of marginal entropy of a data point with regard to the data.  The relevant surprisal contribution conviction, $\sigma_i$ may be computed by finding
\begin{equation}
\sigma_i = \frac{E\left( S_n  \mid n \sim C_{i,\mathcal{F}} \right)}{S_i}.
\end{equation}

\subsection{Similarity Conviction of a Case}
\label{ssec:sim_conviction}

Though surprisal contributions can be used for anomaly detection, the technique has two potential shortcomings.  The first is that computing local deviations to properly contextualize each anomaly to the influential data is somewhat more computationally expensive.  The second is that, even though surprisal is a strong way to communicate anomalousness, nats are not units of measurement intuitive to most users.  However, saying a result is ``twice as surprising as I'd expect'' is a statement that has intuitive value.  For both of these reasons, we primarily employ a surprisal ratio, which we call conviction as described in Equation~\ref{eq:conviction}.  Since this is determining how similar this case is to what we would expect it to be, we call this ratio ``similarity conviction''.  Similarity conviction can be expressed in context of Equation~\ref{eq:surprisal_contribution} as
\begin{equation}
SC_i = \frac{E\left( S_n \mid n \sim C_{i,\mathcal{F}} \right)}{S_i}.
\end{equation}
We have found empirically that a threshold of $SC_i \leq 0.5$ seems to offer the best performance across a set of benchmark anomaly detection data sets as measured by F1 score.  This value corresponds to a case being considered anomalous if it is twice as surprising as expected.  However, this value can be tuned based on the needs of an anomaly detection process based on the cost and value of false positives versus false negatives.  See Section~\ref{ssec:empirical_anomaly_detection} for further details.

\subsection{Synthetic Data via Differential Privacy Mechanisms}
\label{ssec:synth_data}

Synthetic data are fictitious data that is generated to serve some purpose.  Often, synthetic data is used to condition existing data based on exogenous knowledge, to remove bias or address fairness considerations, or to protect privacy.  The generative inference of our system described in Section~\ref{ssec:generative_react} can be used directly to build new data.  A new data set can be created feature-by-feature, case-by-case, by starting with a given feature, generating a value from the unconditioned marginal distribution, and then incrementally using the previous output to condition the next generative output.  The generated data set can be constrained or conditioned via the techniques mentioned in Section~\ref{ssec:constraining_and_goals}.

The generative algorithms can be modified to use differentially private mechanisms in order to improve privacy of the generated data.  To be used in a differentially private way, the queries built and feature attributes need to ensure appropriate sensitivity, values that reveal private data such as personally identifiable information must be substituted, and the privacy budget and other parameters must be applied in a corresponding manner to how the queries and noise mechanisms are being used.  Though formal proofs of differential privacy are beyond the scope of this document, we demonstrate how differentially private mechanisms can be substituted.

For continuous and ordinal features, we can generate values by applying the Laplace mechanism, definition 3.2 as shown by \cite{dwork2014algorithmic} and use sensitivity at least as large as that in definition 3.1 in the same work.  We note that this is closely related to the generative outputs as described in Section~\ref{ssec:generative_react}, though the differences are in the aggregation and sensitivity.  We compute and release the average value of the cases that are the result of the corresponding nearest neighbor query.  The sensitivity should always be at least the range divided by the number of cases; because the influence weights are probabilities, the sensitivity needs to be adjusted to accommodate.  Using the largest gap between consecutive values increases the sensitivity to account for the possibility that one value is extreme and all the rest are not.  Given desired epsilon $\epsilon$, continuous features may be generated as
\begin{equation}
g_j(x_{i,F}, \rho) \sim L\left( x_{n,j}, \frac{\Delta_j(n)}{\epsilon} \right), n \sim \left\{ n \in C_{i,F} : P\left(x_{n,j} \mid F, x_{i,F} \right) \right\}
\end{equation}
When using aggregation-based differential privacy, the equation becomes
\begin{equation}
g_j(x_{i,F}, \rho) \sim L\left( E(x_{n,j} \mid n \sim \left\{ n \in C_{i,F} \right\}), \frac{\Delta_j(n)}{\epsilon} \right).
\end{equation}

For nominal features we can use the \emph{noisy max algorithm} on the probabilities associated with the nominal classes from the kNN query as proven in claim 3.9 by by \cite{dwork2014algorithmic}.  However, to address privacy concerns of homogeneous data we include an extra nominal class in that set with probability zero.  If selected, instead of drawing from the results of the kNN query, it will sample from the global distribution using the noisy max method.  If there is a domain of all possible values (e.g., a list of every nominal possible), it will include another zero probability nominal class that, if selected, will draw from the uniform distribution over all possible values with no information from the data.  Given desired epsilon $\epsilon$, nominal features may be generated as
\begin{equation}
g_j(x_{i,F}, \rho) =
\begin{cases} 
x_{n,j}, n \sim P(x_{n,t} \mid C_i, F, x_{i,F}) & \text{if } \eta < 1 - \frac{e^\epsilon}{1 + e^\epsilon}, \eta \in U(0, 1) \\
x_{n,j}, n \sim \mathcal{C} & \text{if } 1 - \frac{e^\epsilon}{1 + e^\epsilon} \leq \eta < 1 - \left( \frac{e^\epsilon}{1 + e^\epsilon} \right)^2 \\
x_{n,j} \sim U\left(\left\{x_{n,j} \mid n \in \mathcal{C}\right\}\right) & \text{if } 1 - \left( \frac{e^\epsilon}{1 + e^\epsilon} \right)^2 \leq \eta.
\end{cases}
\end{equation}

Cases are then generated by starting from one feature randomly, synthesizing the data element, and then using that value to condition the subsequent element.  The process is repeated until the entire case has been synthesized from a randomized order of the features.  We can chain the results together to construct a synthetic case due to the adaptive query sequence proof by~\cite{dwork2006our}.  Additionally, each synthetic data element is obtained by finding similar data for each set of features values.  Due to the stochastic nature of feature selection and the stochastic nature of ordering results, this shuffling enhances the privacy of the output~\citep{feldman2022hiding, cheu2019distributed}.

Despite their widespread use, similarity-based privacy metrics suffer from many issues~\citep{ganev2025inadequacy}.  Similarity measures of privacy keep the synthetic data points away from the original data points, usually measured by some form of distance ratio.  One of the privacy concerns that similarity measures intend to address is whether an individual may identify their own data in the synthesized data.  The deviations in our technique address the uncertainties which would identify data as the same, and the distances are surprisals, which both help improve the meaningfulness of dissimilarity over most implementations which are more arbitrary distances.  However, we emphasize that the synthesis process should only be done once regardless of whether similarity measures are used; there is no model being produced, so the entire privacy budget should be spent at once.  This removes the feedback loop that an attacker could exploit to attempt to synthesize data points to surround and compromise an existing data point.  When coupled with the differentially private mechanisms, the noise mechanisms introduce plausible deniability, and similarity measures can be used to either reject data to prevent coincidental publication, validate the outcome by measuring the noise added, or both.

We describe these similarity measures as \emph{anonymity preservation}.  This is implemented as a ratio, a comparison of the measure of the similarity of a synthetic data point to its nearest neighbor to some function of the similarity between nonidentical cases among its influential cases.  The minimum distance between nonidentical neighbors is a test to ensure coincidental publishing does not occur can be written as
\begin{equation}
AP_{\min}(x_{i,F}) = \frac{ \min_{n \in C_{x_{i,F}}, I(i, n \mid F) \neq 0} I(i \mid F) }{\min_{i, n \in C_{x_{i,F}}, i \neq n \land I(n^{\prime}, n \mid F) \neq 0} I(n^{\prime}, n \mid F) } .
\end{equation}
A stronger variant can be written as the maximum distance between any two cases in the most influential data.  This is an ex post measurement of the kind of noise that is added by differentially private mechanisms, and can be written as
\begin{equation}
AP_{\max}(x_{i,F}) = \frac{ \max_{n \in C_{x_{i,F}}, I(i, n \mid F) \neq 0} I(i \mid F) }{\max_{i, n \in C_{x_{i,F}}, i \neq n \land I(n^{\prime}, n \mid F) \neq 0} I(n^{\prime}, n \mid F) }.
\end{equation}

Anonymity preservation can be aggregated in different ways, e.g., as an average, geometric mean, or minimum of the distance ratios, depending on whether the concern is measuring the amount of noise added versus assessing the worst case.  When anonymity preservation is configured such that it excludes based on the maximum distance between points, evaluates a single noninteractive one-time synthesis process, and selects bandwidth of the cases to be sufficient to address sensitivity concerns, then this technique becomes a stronger approach of evaluating privacy than existing similarity-based privacy metrics.  The minimum of the minima of all the distance ratios of all data points can be used to evaluate the privacy of the data, indicating a worst case analysis as a way of evaluating the outcome of the privacy enhancing process, rather than being a yes-no binary decision of privacy per record.  And because anonymity preservation can use the local sensitivity appropriate for a case, it is measuring the same characteristics of noise that should be added via differentially private mechanisms.

\section{Data Compression Via Probability Mass and Hierarchical Sharding}
\label{sec:compression_hierarchy}

Traditional machine learning techniques construct a model from input data.  That model is usually much smaller than the source data, and model construction is essentially a form of data compression with generalization.  Because the techniques described in this paper are applied directly to the source data, along with any computed uncertainty, a naive search of that data could be limited by memory and computational resources.  To scale up to larger data sets without limiting the types of queries that can be performed, we employ data compression techniques that robustly shard the data hierarchically.  The popularity of vector databases for retrieving embeddings has increased attention to the long history of nearest neighbor search techniques and yielded many algorithms with different trade-offs~\citep{pan2024survey}.  However, the vast majority of the literature deals with simple distance metrics, where the data instances do not have a specified weight independent of queries, and all features are present in every query.  In this section we discuss how we apply generalization to search on a representative subset of the data, employing case weighting and redundancy elimination techniques similar to \cite{morring2004weighted}, though using our surprisal-based techniques.  We also discuss techniques to use data reduction to shard data and direct queries to subsets of the data in order to improve search performance while minimizing loss in quality of the results.  Although the worst case for exact nearest neighbor search for all data points is only slightly better than $O(|\mathcal{C}|^2 \cdot |\mathcal{F}|)$~\citep{williams2018difference}, in practice, structure and redundancy within the data improves typical performance.  This section focuses on reducing redundancy in the data and making trade-offs to improve generalization and scalability, whereas the details of the search algorithms we employ will be discussed in Section~\ref{sec:algorithms_performance}.

\subsection{Case Weights, Redundancy Reduction, and Ablation}
\label{ssec:case_weights}

In instance-based learning, case weights refer to the importance or relevance of each data point.  By default each case trained has a weight of $1$.  However, if two cases are identical, instead of keeping separate cases, only one copy of the case can be kept in the database and the weight can be incremented to $2$ and beyond as more identical cases are added.  If a data set has significant redundancy due to identical cases, this can be beneficial to computational and storage requirements.

Real world data sets are rarely comprised of mostly populations of identical cases.\footnote{Even when records are identical, data provenance is typically different for the identical records.  Maintenance of provenance of case weights can be maintained via traditional data stores independently of the methods we describe using data provenance tools and databases.}  Instead of only accumulating weight for identical cases, we can instead choose to insert new cases when they are not sufficiently predictable and accumulate weight for cases when they are predictable.  Many techniques exist for reducing data for a particular purpose.  Notable examples date back to the 1960's~\citep{hart1968condensed}, and determining small data subsets that are most relevant is often termed core-sets~\citep{feldman2020core}.  Our work is most closely related to weighted core-sets.

As previously mentioned, when cases are initially trained they have a probability mass of 1 except in certain situations when dealing with rebalancing for bias, fairness, and other use cases covered in Section~\ref{ssec:rebalancing}.  Cases are trained up until a statistically significant threshold has been reached, typically at least 1000 cases.  Subsequent cases are evaluated in order to determine whether they should be trained, or whether their probability masses should be accumulated onto existing cases without storing the new data, something we refer to in our system as \emph{ablation}.  Trained cases are inserted into the database with a probability mass of 1, whereas ablated cases are rejected for training, but their influential cases are determined and the probability mass of 1 is distributed among the influential cases relative to their normalized influence weight.  That is, if case $i$ is being ablated, then the weights $w_n \forall n \in C_i$ will be updated with their respective influential probabilities from Equation~\ref{eq:weighted_case_probability} as
\begin{equation}
w_n \leftarrow w_n + \frac{P(x_i \mid \mathcal{F}, x_n) }{ \sum_{n^{\prime} \in C_i} P(x_i \mid \mathcal{F}, x_{n^{\prime}})} \forall n \in C_i. 
\end{equation}

Even with data ablation, at some point the data accrued may grow too large to hold in memory quickly.  At this point, a data reduction pass is begun.  Data reduction applies a similar ablation selection, except instead applies to already trained cases, removing them and accruing their probability mass to remove cases to a desired size.  By default, this size is $\frac{\mathcal{C}}{e}$ records.  We perform data reduction in batches where cases are selected for removal simultaneously.  Larger batch sizes can improve concurrency and performance, but smaller batch sizes may select better cases to remove.  Once the data has been sufficiently reduced, training with ablation is resumed until the data size becomes too large, and the process is repeated.\footnote{This process also permits the ability to reduce case weights to dampen the effects of prior learning to focus more on immediate data.  Appropriate tuning for this process is future work.}

Our techniques select for removal cases that are highly predictable and most central, preserving cases that define boundaries.  This selection is comprised of two criteria: prefer to remove cases that are both maximally equidistant to their neighbors and which add minimal marginal surprisal to the rest of the data.

For the first criteria, we prefer cases that are central within a crystal-like lattice of other cases, as that means an interpolation would yield an approximately correct value for any of the features.  As dimensionality of a data set increases, most of the records are on the periphery of the volume.  So therefore, the influential cases relative to the case in question should be approximately equidistant.  This can be assessed by measuring the entropy of the probabilities of influence of its influential cases.   We compute a case's influence entropy as
\begin{equation}
H_{influence}(i) = \sum_{n \in C_{i,F}} P(x_i \mid F, x_n) \cdot I(x_i \mid F, x_n).
\end{equation}
If a case's influence entropy is sufficiently high, it should be kept.  We aim to ablate cases that fall in the lowest $\frac{1}{e}$th quantile of influence entropy, meaning that they are the most irregular with regard to the local topology.  The first criteria can be written as cases not being ablated if
\begin{equation}
\label{eq:influence_entropy}
H_{influence}(i | F) < Q_{1 - \frac{1}{e}}\left( \left\{H(n) : n \in C_{i,F}\right\} \right).
\end{equation}

The second criteria is to select cases for ablation that do not add significant marginal surprisal to the overall data.  If a case adds more marginal surprisal than its neighbors, then it is more likely to stick out from its neighboring cases rather than being surrounded by them.  This criteria of not ablating a case can be written as
\begin{equation}
\label{eq:largest_influence}
\min_{x_n \in C_i} I(x_i \mid x_n) > \max_{x_n, x_{n^\prime} \in C_i, x_n \neq x_{n^\prime} } I(x_n \mid x_{n^\prime}).
\end{equation}
Cases are ablated unless they meet both the criteria of Inequation~\ref{eq:influence_entropy} and Inequation~\ref{eq:largest_influence}.  We note that using these criteria for data reduction and data ablation are independent of the features, and therefore tend to perform reasonably well across any task involving any subset of features.

\subsection{Hierarchical Data Sharding}
\label{ssec:hierarchical_sharding}

Data reduction and ablation together yield a much smaller working data set while maintaining utility.  However, if the data set is very large but also very nuanced, such as with language data sets, the relevant data needed to preserve utility may still be too large to query efficiently.  In this case we employ bottom-up hierarchical decomposition of the data in order to divide it into different portions to reduce the total data that must be searched.  As of this writing, fully automated hierarchical data sharding is still under development, but preliminary results show promise.

As data is trained sequentially and the first data reduction pass is performed, any of the cases rejected as being redundant can be moved to a set of shards.  Once the data is reduced, a clustering algorithm, such as that described in Section~\ref{ssec:clustering}, can be applied using surprisals as distances to determine the number of shards as well as which shard each case most closely represents.  The cluster label itself can be stored as a feature.  Then when additional records are trained and ablated, they can be routed to the appropriate shards based on the shards associated with the influential cases.  This process can be repeated to build hierarchical structures of the data, in some ways conceptually similar to other greedy hierarchical search techniques such as HNSW~\citep{malkov2018efficient}, though capable of a wider variety of tasks than just nearest neighbors search on a fixed set of features.

\section{Groups, Clusters, Series, and Anomalies}
\label{sec:groups_series_anomalies}

Data sets often have identifiers of transactions or events that are related to a person, organization, agent, piece of equipment, or other entity.  In relational databases, these are generally considered primary or foreign keys, but also can be identifiers such as personally identifiable information.  Sometimes these identifiers are used in conjunction; a product may have both a SKU and also a serial number, or a transaction may have a buyer and a seller.  Further, cases may have sequence numbers, timestamps, or simply be ordered in one or more ways.  In this section we discuss operations with data groups, how groups can be determined via clustering, how series data is group data that has an ordered sequence, data sets that have both groups and series, and how our system determines whether data is anomalous.

\subsection{Rebalancing, Fairness, and Data Volume Privacy}
\label{ssec:rebalancing}

As discussed in Section~\ref{ssec:case_weights}, each case trained has a default probability mass of 1.  However, sometimes this is not desirable and the data collected may have biases.  Sometimes the biases may be deliberate sampling biases for cost efficiency, when it is cheaper to run one kind of experiment than another, or accidental in that one group tends to respond more than another.  The biases may be systemic or historical such as social biases, or perhaps the distribution of the data is changing and historical data should be updated to better fit future trends.  A different reason case weight rebalancing may be desired is if a given identifier has a disproportionate volume of data, and the impact of that identifier should be adjusted to only count as the weight of a single entity, regardless of data volume corresponding to the identifier.  Rebalancing of weights can be used to enhance privacy, e.g., so that a patient who visits a clinic weekly would be afforded the same privacy as someone who visits a clinic yearly, or so that inferences performed on the data do not overly bias toward high volume identifiers.

The most straightforward way to rebalance case weights is when the rebalancing is desired for a single nominal feature, which does not necessarily need to be an identifier.  Each additional nominal feature used for rebalancing can be leveraged to adjust the weights in combination with other features by updating the weights by each feature sequentially.  Given a rebalancing feature $j$, the case weights can be updated by evening out the probability mass relative to each value
\begin{equation}
w_n \leftarrow w_n \cdot \frac{\sum_{i \in \mathcal{C}} w_i}{\sum_{i \in \mathcal{C}, x_{i,j}=x_{n,j}} w_i}
\end{equation}

As new cases are trained, they can be initialized with the respective weight, and occasionally renormalized if needed, depending on whether data ablation or reduction are being performed.

\subsection{Clustering}
\label{ssec:clustering}

\begin{algorithm}[H]
\caption{Similarity Conviction Clustering}
\label{alg:clustering}
\begin{algorithmic}[1]
\State Sort $\mathcal{C}$ by descending $SC_i \rightarrow SC_{list}$; $cluster\_map \gets\{\}$; $cluster\_id \gets 1$.

\State \textbf{Seed clusters ( $SC_i \ge 1$ )}
\While{$\text{SC\_list} \neq \emptyset$ \textbf{and} $SC_{\text{top}} \ge 1$}
    \State $i \gets \text{first}(\text{SC\_list})$; remove $i$ from list
    \State $cluster\_map[i] \gets cluster\_id$
    \State \Call{ExpandCluster}{$i, cluster\_id$}
    \State $cluster\_id \gets cluster\_id + 1$
\EndWhile

\State \textbf{Attach remaining cases}
\State $\displaystyle \max\!\sigma(c) \gets \max\{\sigma_i : i \in c\}$ for each formed cluster $c$
\State $changed\gets\text{True}$
\While{$changed$}
    \State $changed \gets \text{False}$
    \ForAll{unclustered $i$ sorted by ascending $\sigma_i$}
        \State $c\gets$ unique cluster ID among $C_i$ (if any)
        \If{$c$ exists \textbf{and} $\sigma_x < \text{clustering\_inclusion\_relative\_threshold}\,\max\!\sigma(c)$}
            \State $cluster\_map[i] \gets c$; update $\max\!\sigma(c)$; $changed \gets \text{True}$
        \EndIf
    \EndFor
\EndWhile

\State \textbf{Finalize}
\ForAll{$i \in \mathcal{C}$}
    \If{$i \notin cluster\_map$} $cluster\_map[i] \gets -1$ \EndIf
\EndFor
\State \Return $cluster\_map$
\end{algorithmic}
\end{algorithm}

\begin{algorithm}[H]
\caption{Similarity Conviction Clustering Subroutines}
\label{alg:clustering_subroutines}
\begin{algorithmic}[1]
\Function{ExpandCluster}{$i, cluster\_id$}
    \State $U\gets\{n \in C_i \mid n \text{ unclustered or } cluster\_map[n]\neq cluster\_id\}$
    \ForAll{$n \in U$ with $cluster\_map[n]$ defined}
        \If{$i \in C_n$} \Call{Relabel}{$n, cluster\_id$} \EndIf
    \EndFor
    \State $G \gets\{n \in U \mid SC_n \ge \text{clustering\_expansion\_threshold}\}$
    \ForAll{$n \in G$}
        \State $N_n \gets C_n$
        \If{neighbors of $n$ span multiple clusters}
            \If{any $z\in N_n$ is unclustered and  $z\in G$}
                \State remove $n$ from $G$
                \State \textbf{continue}
            \EndIf
            \State $c_n \gets \argmax_{c}\sum_{z \in N_n\cap c}P(z \mid n)$
            \If{$c_y\neq cluster\_id$}
                \State $cluster\_map[n]\gets c_y$; \textbf{continue}
            \EndIf
        \EndIf
        \State $cluster\_map[n]\gets cluster\_id$
        \State \Call{ExpandCluster}{$n, cluster\_id$}
    \EndFor
\EndFunction

\Function{Relabel}{$n, cluster\_id$}
    \ForAll{$z$ with $cluster\_map[z] = cluster\_map[n]$}
        \State $cluster\_map[z] \gets cluster\_id$
    \EndFor
\EndFunction
\end{algorithmic}
\end{algorithm}

Our approach to clustering takes advantage of all of the forms of density and surprisal stored in our system.  It leverages similarity conviction of each case, $SC_i$, surprisal contribution conviction of each case, $\sigma_i$, and the relevant neighbors of each case $C_i$.  It is similar to the leaders hierarchical clustering algorithm~\citep{vijaya2004leaders}, in that it starts by selecting cases that are the most similar to their relevant neighboring cases, relative to the data around them, and assigns each a cluster id.  As it iterates from the most similar unclustered cases to the least, it determines whether each new case has neighbors of another cluster, and determines whether different clusters should be joined or if unclustered cases should remain unclustered.  Algorithm~\ref{alg:clustering} describes the primary algorithm and Algorithm~\ref{alg:clustering_subroutines} describes the relevant subroutines.

\subsection{Group Anomalousness}
\label{ssec:group_anomalousness}

Often anomalousness is not just about an individual case, but sometimes can be about a group of cases.  It may be that a particular sensor identifier is broken and is producing inconsistent results.  Or it may be that a particular merchant is invoicing in a fraudulent manner where any record by itself would not seem anomalous, but as a group the cases are anomalous either because they're all slightly anomalous or they have an unusual probability density distribution.

Detecting anomalous cases as described in Sections~\ref{ssec:dist_contrib} and \ref{ssec:sim_conviction} is performed by calculating surprisal.  The additive nature of surprisal opens up a variety of powerful techniques.  This opens up a variety of techniques for measuring anomalousness of a group.  Summing the total surprisal of each group indicates which group as a whole may have been the most anomalous, which can be averaged by data volume by dividing the total surprisal of the group by its probability mass to find average group surprisal (AGS) as
\begin{equation}
AGS_{t = \kappa} = \frac{1}{ \sum_{i \in \mathcal{C}, x_{i,t} = \kappa} w_i } \sum_{i \in \mathcal{C}, x_{i,t} = \kappa} S_i.
\end{equation}

Sometimes in situations of fraud or other intelligent gaming of systems, all of the cases by themselves look very typical, but as a distribution their density is anomalous.  Various tools, such as Kullback-Leibler divergence can be employed to find these sorts of group anomalies.  For example, finding the surprisal of groups as potential inliers can be found via
\begin{equation}
S(t = \kappa) = D_{KL}( \{ S_i \forall i \in \mathcal{C} \text{s.t.} x_{i,t} = \kappa \} || E(S_i)).
\end{equation}

\subsection{Anomaly Detection}
\label{ssec:anomaly_detection}

As discussed in other sections, many factors can determine anomalousness, including whether a record is far from the rest of the data, far from similar data, is in a cluster with relatively little data, or whether there is an association between records where the density is anomalous.  Conviction is a consistent unit of measurement across different types of queries since it is the ratio of two surprisals, relating how surprising some data are to how surprising they were expected to be.  This allows us to compare different types of anomalousness on the same scale.

To perform general anomaly detection, we first cluster the data (described in Section~\ref{ssec:clustering}) and compute the distance contribution convictions (described in Section~\ref{ssec:dist_contrib}), similarity convictions (described in Section~\ref{ssec:sim_conviction}), and optionally case convictions related to density anomalousness (described in Section~\ref{ssec:group_anomalousness}) and optionally residual conviction (described in Section~\ref{ssec:residual_conviction}).  These are stored as additional features.  Once these are computed, clusters are sorted based on size, and those clusters below 15\% of the data are held out to compute its average group surprisal as described in Section~\ref{ssec:group_anomalousness}.  Because a record is as anomalous as its most anomalous aspect, the minimal conviction for each record is computed as the minimum of all the forms of conviction.  Any particular form of anomalousness can be inspected or computed for any record.  For example, the most anomalous records in a data set may be labeled as such because they are in an extremely anomalous small group, however, some of the records in that group may also be slightly anomalous in that their values are a bit different than the rest in the group which are tightly clustered.

\subsection{Series, Time Series, and Panel Data}
\label{ssec:time_series}

Instance based learning techniques have a long history with regard to time series dating back to the 1800's but expanding to trends in the late 1950's~\citep{HOLT20045}.  It still remains in use via techniques like exponential smoothing~\citep{de2011forecasting} as well as more complex kNN approaches~\citep{martinez2019methodology}.  Time series forecasting varies a little among differing fields, but often ``time series forecasting'' implies forecasting a single set of variables over time such as macroeconomic data for a single country.  ``Panel data'' and ``multi-series forecasting'' refer to data sets where there are multiple different identifiers, such as metrics around companies for stock pricing or evaluating health changes of patients over time.  Our techniques are flexible with regard to handling temporal data, provided at least one column represents time.  If there are multiple entities being handled temporally, then at least one column should represent an identity of which group with which the temporal sequence of data is affiliated.

When handling continuous features, the rate of change is often important.  This is computed as the rate of change in a value divided by the rate in change of the time feature.  If $j$ is the continuous feature and $t$ is the time feature, then this can be computed as a rate value for case $i$, $x_{i,\Delta j}$, as
\begin{equation}
x_{i,\Delta j} = \frac{x_{i,j} - x_{n,j}}{x_{i,t} - x_{n,t}} \text{ where } n = \argmax_{n^\prime \in \mathcal{C}, x_{n^\prime,t} < x_{i,t}} x_{n^\prime,t}.
\end{equation}
Special handling is done in the software to handle simultaneous events which would otherwise have a division by zero.  The rate calculations may be chained to represent the difference equivalents of second or third derivatives, or beyond.  For example, the second derivative may be approximated as
\begin{equation}
x_{i,\Delta^2 j} = \frac{x_{i,\Delta j} - x_{n,\Delta j}}{x_{i,t} - x_{n,t}} \text{ where } n = \argmax_{n^\prime \in \mathcal{C}, x_{n^\prime,t} < x_{i,t}} x_{n^\prime,t}.
\end{equation}
In practice, we have found that including the first rate is generally sufficient for most problems, and including the second rate has been sufficient for all of the practical problems we have tested, as deeper nuances of the changes are generally learned from the data.  We can imagine higher order rates being valuable for certain problems, especially in the physical sciences.  By default, our implementation assumes a single rate for any continuous feature and implements it via a derived feature as described in Section~\ref{ssec:derived_and_dependent_feat}.

In addition to extra rate features, we employ lag features to repeat values of previous records in the given series.  These extra features are simply the values of the previous steps leading up to the case, and may be used for any feature type.  Lags may be configured and optimized for the data.  The memory and performance costs of including an additional set of features as large as the original feature set must be weighed versus the gains.  In many data sets, the first several lags are generally sufficient to capture very rich behavior.  However, this can miss events that happened much earlier in the time series, which can sometimes require additional feature engineering for aggregation and lags on those additional features.  Future work on lags includes exploring treating each time series as its own data set, performing ablation and data reduction as described in Section~\ref{ssec:case_weights}, and then combining the probability mass of the series regardless of the particular lag.  The goal is to reduce a time series to only the events that are pertinent and compare those, instead of maintaining specific lags for all records.  This is inspired by the work of ~\cite{fountas2025humaninspired} that obtains extremely long context windows for LLMs via breaking apart the sequences by surprisal and including the relevant chain.  The inspiration is to treat a given series at least as temporary training data for the purposes of inference, using the probability mass of related sequences, and perform inference based on the subsequent values in the relevant time series.

The final additional features used in time series are features related to time and event progress throughout a series, time and event progress until a series ends, and a count of any synchronous events.  Additionally, features may be marked as stationary features, which means that they will have the same value across the time series, and may be synchronously updated if the predicted value changes.  An example of a stationary feature might be a city of birth, location that a company was originally incorporated, or the eventual outcome of a particular series such as which player ultimately won or lost a game.  These features are inferenced via probability mass across the entire time series.

The inference process for forecasting uses the additional features listed here and performs an inference on all nominal features and highest order rate features.  Then it uses rate features to compute the lower order rates and eventually the continuous values, propagating null values when missing values are appropriate.  The inference process can be generative or discriminative as described in Sections~\ref{ssec:discriminative_predictions} and~\ref{ssec:generative_react}.  The aggregation of a set of generative reacts can be performed to obtain a mean absolute deviation for each step in time to help characterize the uncertainty of a forecast.

\subsection{Reinforcement Learning and Relation to Active Inference}
\label{ssec:reinforcement_learning_active_inference}

Reinforcement learning can be implemented by combining generative outputs (Section~\ref{ssec:generative_react}) with time series (Section~\ref{ssec:time_series}) and goal features (Section~\ref{ssec:constraining_and_goals}).  To implement this, the features fall into several categories.  First are the features that characterize the world or system state in which the agent is acting.  The time series element includes the historical information pertinent to characterizing the state.  Next, goal features describe what goal to achieve, e.g., maximizing the score, minimizing cost, or attempting to achieve a specific value.  These goal features find the data that best achieve those goals within the data that best match the state, and updates the probability of influence based on achieving the goal.  From these data and influence weights, the system produces a generative output for what action should be taken given the context of the world state as described by other features.  Given a world state described by $x_F$ and goal function over action features $J$ on the state, $\Omega_J(x_F)$, we generate a set of actions conditioned on selecting the value from the cases that minimize the goal functions for conviction $\rho$ as
\begin{equation}
g_J\left(x_F, \rho \mid GC_J(x_{i,F}, \Omega_J) \right).
\end{equation}

Our system allows the user to describe the domain of each feature so that reinforcement learning may begin without any training data.  Initially, these first actions are drawn uniformly from the domain since there is no supporting data.  As data is obtained including the resulting reward, the rewards are stored in the goal features.  This combined data is trained into the system with all of these features including the state, action, and goal features as a sequence of time series data.  If the reinforcement learning is applied to games, then each game itself can be a series with a unique identifier.  Rewards can be implemented either as the final reward or incrementally, depending on what is best suited for the domain.

The conviction parameter provides control over exploration versus exploitation.  Since conviction is the ratio of expected surprisal to observed surprisal, it can even be fixed to a specific value and the system can still converge to a high performing agent.  As the system obtains more data, it will generate actions following the marginal distribution given the state and goals, which will converge to an appropriate set of actions that achieve the rewards for portions of the state that have significant data, but will still explore in sparse areas of the data because the marginal distribution will have more uncertainty.  We have found that conviction values in the realm of 2 to 5 tend to perform well, which means it will lean slightly toward exploiting knowledge, though dynamically adjusting this value based on domain knowledge can improve how quickly the system can learn.  An interesting flexibility of this system is that the performance can be controlled dynamically by changing how the goal features are parameterized.  For example, if a system has been trained to win at a particular game by maximizing a score, the goal features can be changed to seek a particular score, playing at a set level of difficulty, rather than maximizing it.  Of course, the data for a given level of difficulty may be more sparse than that which is maximized, but additional training can improve the agent's ability to perform across levels.  Data ablation and reduction, described in Section~\ref{ssec:case_weights}, can be quite important to keeping the data volume appropriate.

Given that all of the distances are surprisals, our techniques relate to the core ideas of active inference~\citep{sajid2021active, friston2009reinforcement}.  Free energy, which is the surprisal of a given sensory input given the observed state in active inference, corresponds directly to the surprisal between currently observed data and trained data in our system.  Further, our system can measure many variants of surprisal, such the surprisal it might take to achieve a particular score given the current world state.  We believe that there are rich overlaps between our work and active inference and believe it is a rich area for future work.

\section{Implementation}
\label{sec:algorithms_performance}

Here we discuss design and algorithmic aspects of our implementation.

\subsection{API, Architecture, and Usability}

The general design philosophy is that data should be stored similar to a database, but in manageable quantities that generalize the knowledge and reflect the distribution.  Each instance corresponds to a data set that can be trained, edited, updated, and inferred against.  The horizontally scaling implementation is built using Kubernetes,\footnote{\url{https://kubernetes.io/}} but also can be run locally as a shared compiled library using a Python interface.  The engine is written in the Amalgam language described in Section~\ref{ssec:amalgam}, and the Amalgam language includes a rich query engine that manages the storage of the data.

To begin, the data schema is described via what we call \emph{feature attributes}.  Feature attributes are a rich schema that describe the type of column of data, including: whether it is continuous, nominal, or ordinal; allowed values or the range for continuous numeric features; whether nulls or missing data are allowed; whether a feature has a specific date or time format; whether the feature is cyclic; whether a feature should be treated as numeric, string, code, etc.; how the feature is computed if applicable (e.g., the code to compute one feature from another if it is known and deterministic); as well as the feature's data format.  Due to the flexibility of the Amalgam language, feature attributes may be changed or corrected at runtime.  The first step in most workflows is to call a method in our system known as ``infer\_feature\_attributes'' which uses a rich set of heuristics to make a best guess, taking advantage of whatever information and metadata a data frame or data store has using the most appropriate relevant APIs.

The main API verbs are train, analyze, and react.  Train is a general verb that includes addition of new data, as well as editing and removal of existing data.  Any data that includes derived features, such as rates of change for time series data or dynamically computed features, are updated during training or editing time.  The train process accounts for data ablation, described in Section~\ref{ssec:case_weights} to accrue probability mass when the data itself is largely redundant.  Analyze describes the process that characterizes the uncertainty of the data, performing the computations described in Section~\ref{ssec:deviations}, and also includes data reduction to keep the data size to a manageable amount also described in Section~\ref{ssec:case_weights}.  The analyze process can also be computed dynamically or concurrently to training in horizontally scaling environments.

The react verb describes a desired inference based on a set of contextual input values, a set of actions or outputs, a set of details to report, whether the react should be generative or discriminative, and any relevant parameters.  The details include a wide variety of options, such as influential cases, boundary cases, boundary values, various forms of residuals and other uncertainties, as well as feature and case influence metrics.  These details are included in the react verb because it is typically much faster to compute relevant details along with an inference rather than computing them after the fact.  React comes in variants including react series, react group, react aggregate, and react into features.  React series is best suited for time series data as described in Section~\ref{ssec:time_series}, and react group corresponds to the capabilities described in all of Section~\ref{sec:groups_series_anomalies}.  React aggregate returns aggregate results, particularly aggregations of inference across all or part of the data.  React into features performs inference, potentially with additional details, and stores the results into new features.

\begin{figure}[!ht]
\centering
\includegraphics[width=0.8\textwidth]{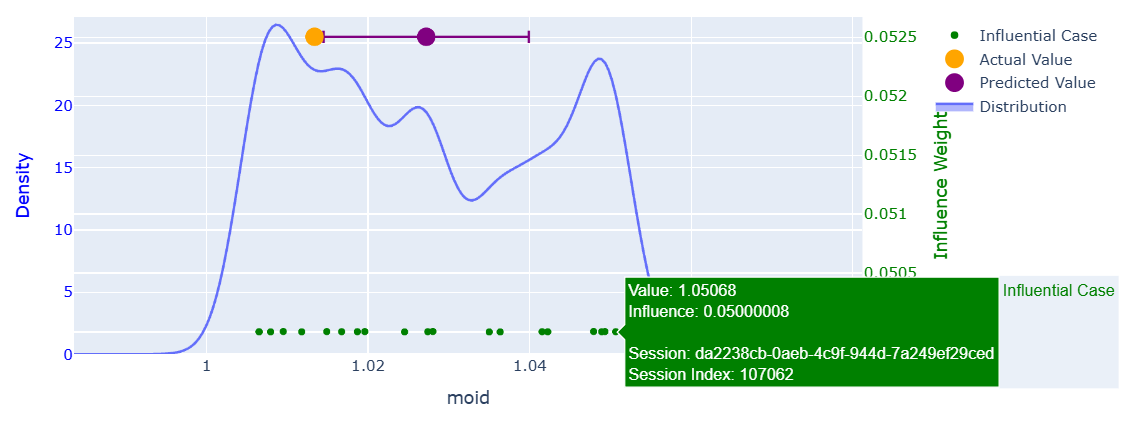}
\caption{An example of inspecting a prediction.}
\label{fig:learning_visualizing_prediction}
\end{figure}

Figure~\ref{fig:learning_visualizing_prediction} depicts a way of visualizing the data and uncertainty around a prediction, in this example using the Small-Body Database Lookup provided by NASA JPL\footnote{\url{https://ssd.jpl.nasa.gov/tools/sbdb_lookup.html}} to predict the minimum orbit intersection distance of the object and Earth.  In this particular example, we have held out the asteroid Hehe and are attempting to predict it, plotting the actual value in orange.  The prediction is the reddish-purple dot, and the mean absolute deviation of the prediction is the whisker plot around it, showing that the actual value is just beyond the 50th percentile of uncertainty, a reasonable estimate.  The blue line is the probability density of the prediction, indicating it is largely bimodal and so it is more likely that the prediction would be closer to one of the modes.  The green dots on the bottom show the cases used to make the inference, with one of the values selected.  Note that any information can be put in the green box, but for this particular plot it is just the value and the provenance of the data based on its training session UUID and the index it was trained within that session.  The records selected as being within the relevant statistical bandwidth for this prediction are relatively numerous, and none of the records stand out as being notably more influential than the others, and the influence weight of the selected record is only about 5\%.

\begin{figure}[!ht]
\centering
\includegraphics[width=0.9\textwidth]{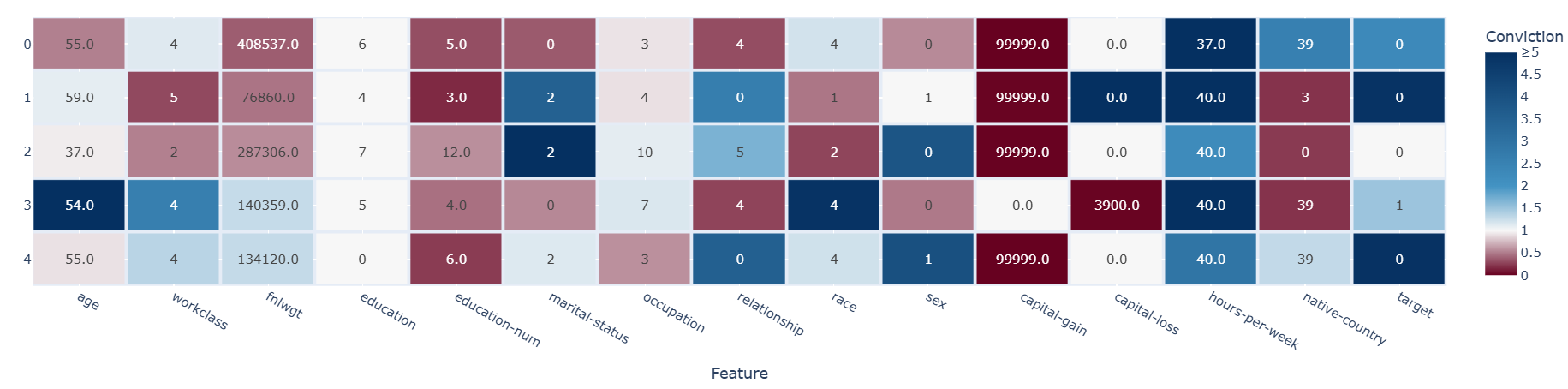}
\caption{An example of inspecting the most anomalous records.}
\label{fig:learning_visualizing_anomalies}
\end{figure}

Exploring anomalies can be done visually as is depicted in Figure~\ref{fig:learning_visualizing_anomalies}.  This depicts outlier anomalies in the Adult dataset~\citep{adult_dataset}.  Each value is predicted and compared to the expected uncertainty of the prediction as residual conviction (described in Section~\ref{ssec:residual_conviction}).  From this visual anomalous values are easy to discern, and the combinations of the amount of capital gains, the education, and the native country are driving the significance of the anomalous cases.

Many expected values such as prediction contributions (as discussed in Section~\ref{ssec:prediction_contributions}), accuracy contributions (as discussed in Section~\ref{ssec:accuracy_contributions}), and robust residuals (as discussed in Sections~\ref{ssec:robust_residuals}) may sometimes yield small negative values when used when computing unrelated values due to statistical sampling.  To ensure proper treatment of data and appropriate normalization, clamping these values to zero is an important practical matter in implementation.  However, it is additionally important to empirically validate that the unclamped values are sufficiently small even if negative in order to ensure correctness.

\subsection{Amalgam Programming Language}
\label{ssec:amalgam}

Amalgam is a domain specific language developed primarily for genetic programming and instance based machine learning, but also for simulation, agent based modeling, data storage and retrieval, the mathematics of probability theory and information theory, and game content and AI.  The primary repository for the language is located at \url{https://github.com/howsoai/amalgam}.  The language format is like LISP and Scheme in that it uses parenthesized list format with prefix notation and is geared toward functional programming, where there is a one-to-one mapping between the code and the corresponding parse tree.  The choice for implementing our techniques using Amalgam are due to its high performance and flexible query system, its ability to work with semistructured data and code safely via generative, comparative, and executable workflows, as well as its transactional nature of storing and updating data.  It also integrates easily within the Python ecosystem.

Whereas virtually all practical programming languages are primarily designed for some combination of programmer productivity and computational performance, Amalgam prioritizes code matching and merging, as well as a deep equivalence of code and data.  Amalgam uses entities to store code and data, with a rich query system to find contained entities based on their attributes or labels, and also has strong and easy-to-use concurrency capabilities.  The language uses a variable stack, but all attributes and methods are stored directly as labels in entities.  There is no separate class versus instance, but entities can be used as prototypes to be copied and modified.  Though code and data are represented as trees from the root of each entity, graphs in code and data structures are permitted and are flattened to code using special references.  Further, instead of failing early when there is an error, Amalgam supports genetic programming and code mixing by being extremely weakly typed, and attempts to find a way to execute code no matter whether types match or not.

Amalgam takes inspiration from many programming languages, but those with the largest influence are LISP, Scheme, Haskell, Perl, Smalltalk, and Python.  Despite being much like LISP, there is deliberately no macro system.  This is to make sure that code is semantically similar whenever the code is similar, regardless of context.  It makes it easy to find the difference between \verb|x| and \verb|y| as an executable patch, and then apply that patch to \verb|z| as \verb|(call (difference x y) {_ z})|, or semantically mix blocks of code \verb|a| and \verb|b| as \verb|(mix a b)|.  These capabilities are used for working with semistructured data, to find similarity by edit distance and to create generative outputs.  Amalgam is not a purely functional language. It has imperative and object oriented capabilities, but is primarily optimized for functional programming with relatively few opcodes that are functionally flexible based on parameters to maximize flexibility with code mixing and matching.

Genetic programming can create arbitrary code, so there is always a chance that an evolved program ends up consuming more CPU or memory resources than desired, or may attempt to affect the system outside of the interpreter.  For these reasons, there are many strict sandboxing aspects of the language with optional constraints on access, CPU, and memory.  Amalgam also has a rich permissions system, which controls what code is able to do, whether writing to the console or executing system commands.  These security capabilities and performance constraints enable custom code to be both created and executed in the middle of inferences.

\subsection{Nearest Neighbors Query Implementation}

Amalgam has a rich query engine which utilizes data structures common among fast databases, such as bit vectors, high performance hashes, column oriented data structures, and bidirectional mapping between values and entities.  In this section we focus on its nearest neighbor queries, as this algorithm is a novel combination of techniques.  This algorithm computes exact nearest neighbors and is designed to be fast when using \emph{any} combination of features to find data, not just one set of features.  It is not nearly as fast as techniques designed for simpler Euclidean spaces (e.g.,~\cite{johnson2019billion} or table 3 from~\cite{ukey2023survey}), or approximate nearest neighbors (e.g.,~\cite{gao2023high}), but we are unaware of other nearest neighbor techniques optimized for querying arbitrary subsets of features with surprisal based distances.

Algorithm~\ref{alg:find_nearest_cases} describes the outline of the algorithm we have developed to quickly find nearest neighbors, exactly to the level of numerical precision requested.  We note that the query engine is designed to support a variety of distances and surprisals, and so describe the algorithm with regard to distances here.  We refer to partially computed distances as ``partial sums'', which are the sums of each distance term, measured in surprisal, for each feature.  First, the partial sums are accumulated for exact matches of any values for features, as well as similar matches that expand out to 97\% of the uncertainty mass of the deviation, as computed by expanding out to include values less than or equal 5 times the deviation (each doubling of the deviation yields half of the probability mass, and $1 - 0.5^5 \approx 0.97$).  The largest computed distances for each feature are stored to allow calculation of the minimal possible partial distance sum for any unresolved feature set, thereby enabling efficient elimination of cases from consideration.

After the initial data structures have been populated, our algorithm then finds cases that have a large number of features populated with partial sums to use as initial good match cases, found by sampling from the entire set of cases and populating a priority queue with the largest counts.  Those potentially good match candidates have the remaining distance terms computed to obtain fully resolved distances and are used to populate a priority queue that contains the smallest distances.  After these potential good matches have been populated, the remaining cases are evaluated with early rejections if the remainder of uncomputed features would increase the distance beyond the current rejection distance, which is the largest distance of the top cases.  At the end, the most similar cases are selected from the subset of the top cases based on the statistical bandwidth as described in Section~\ref{ssec:influential_cases}.

\begin{algorithm}[H]
    \caption{Find Nearest Cases}
    \label{alg:find_nearest_cases}
    \textbf{Procedure} find\_nearest\_cases()
    \begin{algorithmic}[1]
        \State Precompute nominal and interned distance terms
        \For{$f \in features$}
            \State Accumulate distance terms to partial sums for exact matches
            \If{$f$ is continuous or ordinal}
                \State Accumulate distance terms to partial sums to expanded surprisal quantile (0.97)
            \EndIf
            \State $min\_unpopulated\_distances[f] = GetMaximumDistanceComputed(f)$
        \EndFor

        \State Initialize $min\_distance\_by\_unpopulated\_count$
        \For{$i = 1$ to $min\_distance\_by\_unpopulated\_count.size()$}
            \State $min\_distance\_by\_unpopulated\_count[i] += min\_distance\_by\_unpopulated\_count[i - 1]$
        \EndFor

        \State $potential\_good\_match\_cases = FindCasesWithApproximatelyLargestAccumCount(k\_upper\_bound)$

        \State Initialize priority queue with stochastic tie-breaking: $top\_k(k\_upper\_bound)$

        \For{$case \in potential\_good\_match\_cases$}
            \State $distance = ResolveDistanceToTarget(case)$
            \State $top\_k.Push(case, distance)$
            \State $remaining\_cases.remove(case)$
        \EndFor

        \State Get largest distance in priority queue: $reject\_distance = top\_k.GetLargestDistance()$

        \For{$case \in remaining\_cases$}
            \State $distance = GetPartialSum(case)$
            \State $num\_uncalculated\_features = GetNumUncalculatedFeatures(case)$
            \State $distance += min\_distance\_by\_unpopulated\_count[num\_uncalculated\_features]$

            \If{$distance > reject\_distance$}
                \State \textbf{continue}
            \EndIf

            \For{$feature \in GetUncalculatedFeatures(case)$}
				\State $num\_uncalculated\_features = num\_uncalculated\_features - 1$
                \State $distance -= min\_unpopulated\_distances[num\_uncalculated\_features]$
                \State $distance += ComputeDistanceTerm(case)$

                \If{$distance > reject\_distance$}
                    \State \textbf{break}
                \EndIf
            \EndFor

            \If{$distance \leq reject\_distance$}
                \State $top\_k.PushAndPopLargest(case, distance)$
                \State $reject\_distance = top\_k.GetLargestDistance()$
            \EndIf
        \EndFor

        \State \textbf{return} $TruncateCasesByBandwidth(top\_k)$
    \end{algorithmic}
\end{algorithm}

Though Algorithm~\ref{alg:find_nearest_cases} describes the high level elements of our algorithm, there are many details and enhancements that are beyond the scope of this paper.  For example, the nearest cases of the previous query are stored per thread and their distances are resolved after the potential good matches but before the remainder of the cases.  This accelerates finding nearest neighbors by finding better early matches in the common cases when repeated queries on the same thread refer to similar data.  When the number of unique numeric and string values are low relative to the total number of cases the values themselves are interned, meaning that the values are stored in a table and a zero-based integer index is stored to the table location.  This allows for fast array lookups of distance terms in inner loops of resolving distance terms.  Additionally, many of the set inclusion operations use bit vectors with optimized iteration including population count instructions and other fast bitwise manipulations.  We note that although this algorithm technically has a computational complexity of $O(|\mathcal{F}| \cdot |\mathcal{C}|)$, at least on some data sets, in practice the portion of the algorithm that iterates over every case is optimized and performs as well or better than some tree-based implementations on similar data sizes of less than 100,000 cases.  At larger data volumes, data reduction and hierarchical approaches described in Section~\ref{sec:compression_hierarchy} remain important for scalability.

\section{Empirical Results}

This section includes empirical results across a wide variety of tasks.  Our current goal is to demonstrate that our technique has competence in all of these tasks rather than conclusively exceed state-of-the-art on any or all of them.  Over the years that we have developed these techniques, we have noticed that improving the results on one task often (but not always) improves performance on other tasks.  Future work includes continuing to improve results and push towards maximizing the performance on these tasks while maintaining the mathematical consistency and universal applicability of the core techniques, as well as increasing the number of data sets and benchmarks run.

\subsection{Supervised Learning}
\label{ssec:supervised_results}

\noindent We evaluated Howso's supervised learning performance against XGBoost~\citep{chen2016xgboost} and LightGBM~\citep{shi2022quantized} across the Penn Machine Learning Benchmarks (PMLB) repository~\citep{olson2017pmlb}. The evaluation encompasses 404 datasets spanning both classification and regression tasks with dataset sizes ranging from tens to over one million cases and dimensionalities from 2 to 1,000 features.  We chose the Matthews correlation coefficient (MCC) metric as it is highly sensitive to all classification error types, and chose the Spearman correlation coefficient for regression tasks since it is more robust across distributions than $R^2$.

% Auto-generated supervised learning main results tables
% Generated by generate_supervised_main_tables.py
% Using pre-computed aggregate files (no math performed here)

\begin{table}[!ht]
    \centering
    \begin{tabular}{lcccc}
    \toprule
    {Classification} & Howso (targetless) & Howso (single-target) & XGBoost & LightGBM \\
    \midrule
    Mean MCC $\uparrow$ & 0.599 & 0.650 & \textbf{0.659} & 0.651 \\
    Mean Accuracy $\uparrow$ & 0.807 & \textbf{0.829} & \textbf{0.829} & 0.825 \\
    Mean Precision $\uparrow$ & 0.755 & 0.785 & \textbf{0.793} & 0.772 \\
    Mean Recall $\uparrow$ & 0.737 & 0.769 & \textbf{0.779} & 0.767 \\
    Mean F1 $\uparrow$ & 0.744 & 0.777 & \textbf{0.785} & 0.768 \\
    \bottomrule
    \end{tabular}
    \caption{Classification Results across 145 PMLB Datasets.  Top score bolded.}
    \label{tab:pmlb_classification}
\end{table}

\begin{table}[!ht]
    \centering
    \begin{tabular}{lcccc}
    \toprule
    {Regression} & Howso (targetless) & Howso (single-target) & XGBoost & LightGBM \\
    \midrule
    Spearman $\uparrow$ & 0.918 & \textbf{0.924} & 0.919 & 0.916 \\
    $R^2$ $\uparrow$ & \textbf{0.848} & 0.834 & 0.808 & 0.787 \\
    \bottomrule
    \end{tabular}
    \caption{Regression Results across 258 PMLB Datasets.  Top score bolded.}
    \label{tab:pmlb_regression}
\end{table}

When optimizing for a single target feature as described in Section~\ref{ssec:single_targeted} (as XGBoost and LightGBM do), Howso demonstrates competitive performance on MCC (0.652 vs XGBoost 0.661/LightGBM 0.654) and accuracy (0.830 vs XGBoost 0.830/LightGBM 0.826) as shown in Table~\ref{tab:pmlb_classification}, and exceeds both competitors on regression as measured by Spearman correlation (0.924 vs XGBoost 0.919/LightGBM 0.916) as shown in Table~\ref{tab:pmlb_regression}. Additionally, Howso can operate in a targetless configuration as described in Section~\ref{ssec:deviations} (a capability not available in other supervised learning approaches) where, despite not optimizing to predict a specific target feature, it maintains competitive performance with accuracy of 0.808 on classification (vs XGBoost 0.830/LightGBM 0.826) and regression (0.918 vs XGBoost 0.919/LightGBM 0.916).  Experimental details are provided in Appendix~\ref{ssec:appendix_supervised_learning}.

Even though PMLB specifies whether each feature is categorical (nominal only) or continuous, for these experiments, we let all estimators determine this on their own based on their own heuristics.  We ran these experiments this way because in our experience, many practitioners would prefer the tools figure this out when possible.\footnote{One example of a typical oddity we have run into is using special extreme values to indicate nominal values.  One commercial data set used in production we encountered used specific dates in January of the year 2655 to indicate particular nominal conditions such as unknown or specific types of indefinite values.}  Running all of these experiments with specified nominal versus continuous features would benefit all estimators, and it may help Howso more than XGBoost or LightGBM.  For example, there are data sets in PMLB that have floating point values of -1.0, 0.0, and 1.0 where the values indicate clearly nominal data and have observed that representing 0 being closer to 1 than -1 in continuous space reduces the quality of Howso's estimations.  Future work involves improving our ``infer\_feature\_attributes'' heuristics to automatically infer these types from the data better.  We also will continue to work on improving the deviation and feature probability computations; preliminary results show that we are able to improve classification tasks with targetless analysis beyond that of the current targeted algorithm, but at the reduction of regression tasks, and vice versa.  We believe it is very plausible that we will be able to entirely merge these flows to include one targetless algorithm of superior quality inference.

\subsection{Feature Importance}

\begin{figure}[!ht]
\centering
\includegraphics[width=0.49\textwidth]{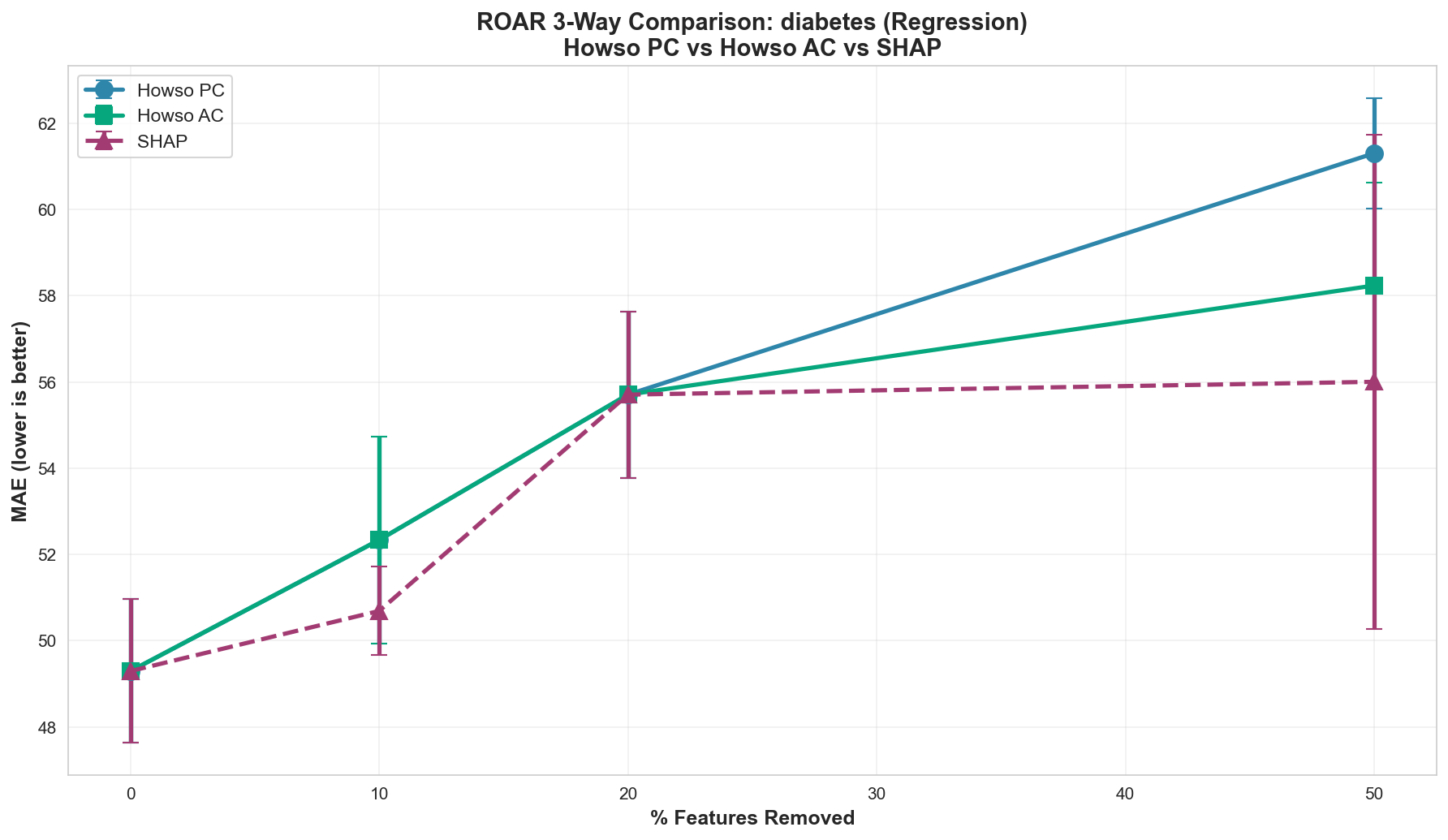}
\includegraphics[width=0.49\textwidth]{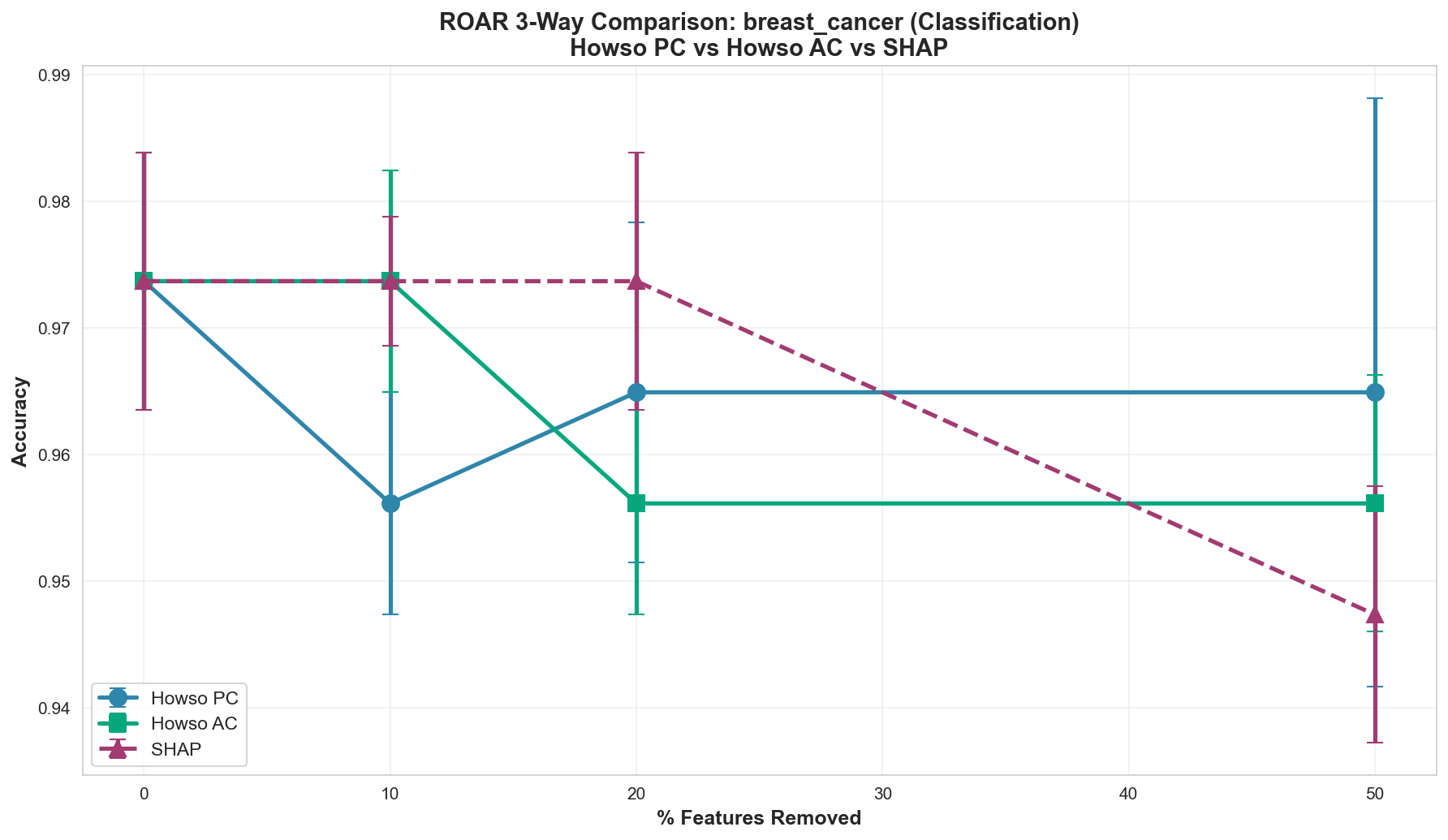}
\caption{ROAR curves comparing Howso Prediction Contributions, Howso Accuracy Contributions, and SHAP on diabetes (MAE) and breast cancer (Accuracy) with XGBoost; 3 seeds, no tuning.}
\label{fig:roar_comparison}
\end{figure}

We applied  ROAR (RemOve And Retrain) validation~\citep{hooker2019benchmark} to compare Howso's Prediction Contributions (PC) as described in Section~\ref{ssec:prediction_contributions} and Accuracy Contributions (AC) described in Section~\ref{ssec:accuracy_contributions} against XGBoost using SHAP~\citep{lundberg2017unifiedapproachinterpretingmodel}.  On the diabetes data set, Prediction Contributions shows a slightly steeper degradation (+24.4\% at 50\%) than Accuracy Contributions (+18.0\%) or SHAP (+13.6\%), while all methods remain robust on the breast cancer data set (less than 2pp degradation at 50\%). Full experimental details are provided in Appendix~\ref{sec:appendix_roar}.

\begin{figure}[!ht]
\centering
\includegraphics[width=0.7\textwidth]{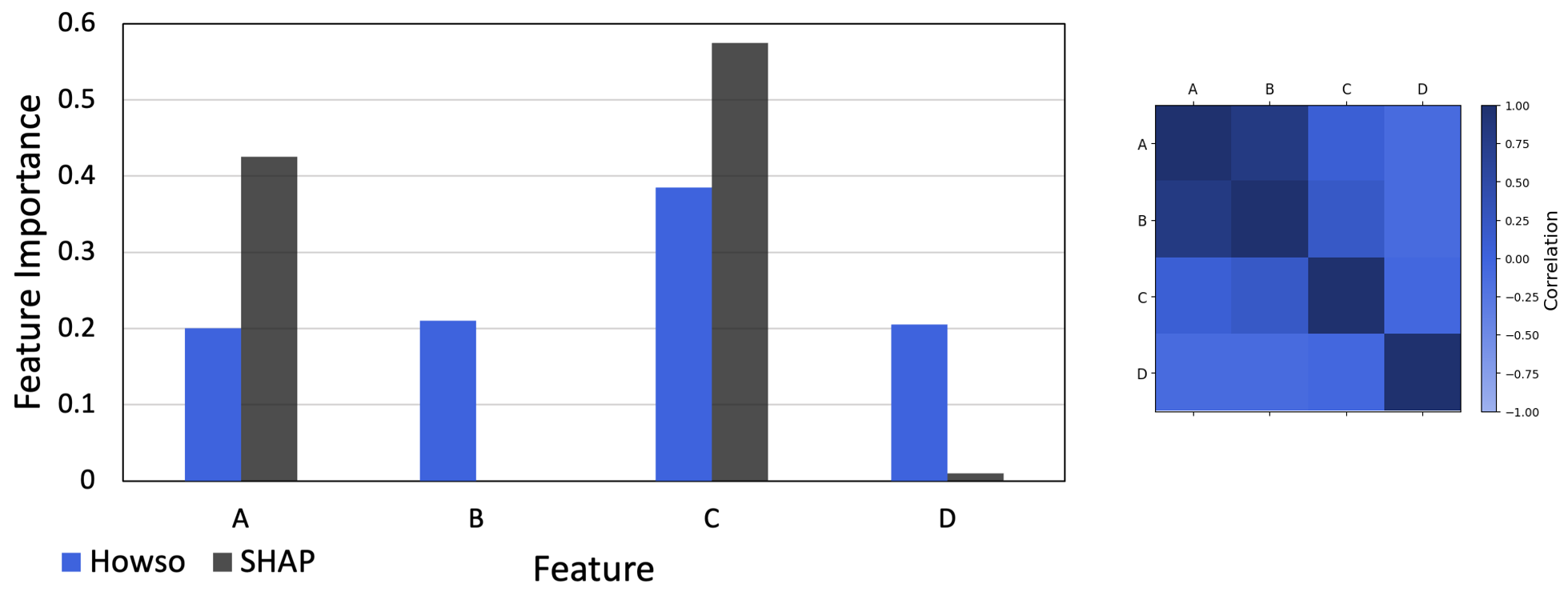}
\caption{An example of inductive bias influencing feature importance, LightGBM SHAP versus Howso prediction contributions on a generated data set.}
\label{fig:shap_inductive_bias}
\end{figure}

Additionally, we created a data set to experiment to test inductive bias.  In this experiment, features A, C, and D were generated from uniform random distributions of different scales, C being the largest and D being the smallest.  Feature B was just a copy of feature A with a very small amount of noise added.  We trained Howso and LightGBM on the data and measured feature importance.  As depicted in Figure~\ref{fig:shap_inductive_bias}, the feature importance of SHAP indicated how the LightGBM model was built and operating, reflecting the model's inductive bias of eliminating feature B.  When we changed the feature order, it indicated that feature B was important and that A had zero importance.  However, when running Howso's Prediction Contributions, we found that it did not incur inductive bias but accurately reflected the importance and usefulness of the features in predicting the target.

\subsection{Anomaly Detection}
\label{ssec:empirical_anomaly_detection}

Since anomalies occur in many different forms, benchmarking a variety of anomaly types yields a more representative and robust assessment of real-world data. ADBench~\citep{han2022adbench} is widely recognized as a versatile benchmark for tabular anomaly detection.  It categorizes anomalies into the four types of local, global, dependency, and cluster.  The benchmark system constructs benchmark datasets by removing prelabeled anomalies from third-party datasets, then regenerates an equal number of synthetic anomalies using one of the four generation methods.  Our benchmark extends ADBench's approach by generating all four anomaly types within each dataset rather than limiting each dataset to a single anomaly category.  Specifically, this benchmark applies ADBench's four generation methods to create one-quarter of the total anomalies from each category, yielding datasets with a greater diversity of anomaly types that better reflects the diversity of real-world anomaly behavior.

\begin{table}[!ht]
\centering
\begin{tabular}{lrr}
\toprule
Metric & PR-AUC & ROC-AUC \\
\midrule
CBLOF~\citep{he2003discovering} & 0.1357 & 0.5851 \\
COF~\citep{pokrajac2008incremental} & 0.2361 & 0.7783 \\
COPOD~\citep{li2020copod} & 0.2847 & 0.8262 \\
ECOD~\citep{li2022ecod} & 0.2895 & 0.8331 \\
FeatureBagging~\citep{lazarevic2005feature} & 0.2818 & 0.8253 \\
HBOS~\citep{goldstein2012histogram} & 0.2826 & 0.8011 \\
IForest~\citep{liu2008isolation} & 0.2865 & 0.8510 \\
kNN~\citep{ramaswamy2000efficient} & 0.3029 & 0.8799 \\
LODA~\citep{pevny2016loda} & 0.2539 & 0.7554 \\
LOF~\citep{breunig2000lof} & 0.2814 & 0.8135 \\
MCD~\citep{hubert2010minimum} & 0.3005 & 0.8780 \\
MOGAAL~\citep{liu2019generative} & 0.0976 & 0.5432 \\
OCSVM~\citep{scholkopf2001estimating} & 0.2932 & 0.8367 \\
PCA~\citep{shyu2003novel} & 0.2821 & 0.8134 \\
SOD~\citep{zuo2023novel} & 0.2849 & 0.8557 \\
SOGAAL~\citep{wang2021new} & 0.0959 & 0.6052 \\
Howso Anomaly Detection & \textbf{0.3265} & \textbf{0.9060} \\
Howso Surprisal Contribution Conviction Only & 0.2977 & 0.8861 \\
Howso Similarity Conviction Only & 0.2550 & 0.7804 \\
\bottomrule
\end{tabular}
\caption{Aggregated Anomaly Detection Results.  Top scores are in bold.}
\label{tab:anomaly_results}
\end{table}

Our results, summarized in Table~\ref{tab:anomaly_results}, are based on five independent runs for each of the 41 datasets. Compared to ADBench's aggregated results, which report average model performance across all four anomaly categories, our findings are generally consistent.  The top three models by ROC-AUC in our evaluation, kNN, IForest, and SOD, all achieve above-average performance in ADBench's benchmark results, with kNN and IForest ranking as the second and third best methods, respectively, in ADBench.  Howso Anomaly Detection, described in Section~\ref{ssec:anomaly_detection}, achieved the highest overall performance, with a PR-AUC of 0.3265 and an ROC-AUC of 0.9060.  The Howso Surprisal Contribution configuration, which is conceptually similar to kNN, exhibited comparable performance to kNN itself.  Notably, CBLOF, which was the top-performing model in ADBench's aggregate evaluation, performs substantially worse under these more heterogeneous conditions.  These results emphasize the increased difficulty of real-world datasets containing a mixture of anomaly types and suggests that methods tuned for controlled or single-type anomaly distributions may struggle with the complexity of mixed anomaly datasets.  In contrast, the adaptability of Howso's anomaly detection approach enables stronger generalization and more robust performance across realistic, heterogeneous anomaly distributions.  Full experimental details are provided in Appendix~\ref{ssec:appendix_anomaly}.  Preliminary results suggest that improving the clustering techniques positively impacts Howso's anomaly detection scores.  Making sure that the clustering is robust to different data scales using the surprisal-as-distance metrics by Howso is future work that may improve these results even further.

\subsection{Learning from Demonstration}

The results in this section are more qualitative.  We describe two experiments performed in real-time environments.  In both, our software was running on the same laptop as the real-time environment, in both cases connected via a C\# API.  These experiments were performed earlier in the development of this software.

The first learning from demonstration experiment was performed on a drone simulation platform developed by Hazardous Software Inc.  The sensory inputs to the agent were simulated preprocessed LiDAR data scanning the area in front of the drone, with a minimum filter to reduce the feature set to around 100 samples each representing the closest distance to an object, and that distance was divided by the velocity in that direction.  This meant that the sensory inputs were time-to-impact values.  Additionally, random navigation points were generated within a certain radius, though the navigation points were generated regardless of objects or obstacles, as long as the navigation point did not intersect with an object.  The last set of input data was Xbox 360 controller input.

\begin{figure}[!ht]
\centering
\includegraphics[width=0.5\textwidth]{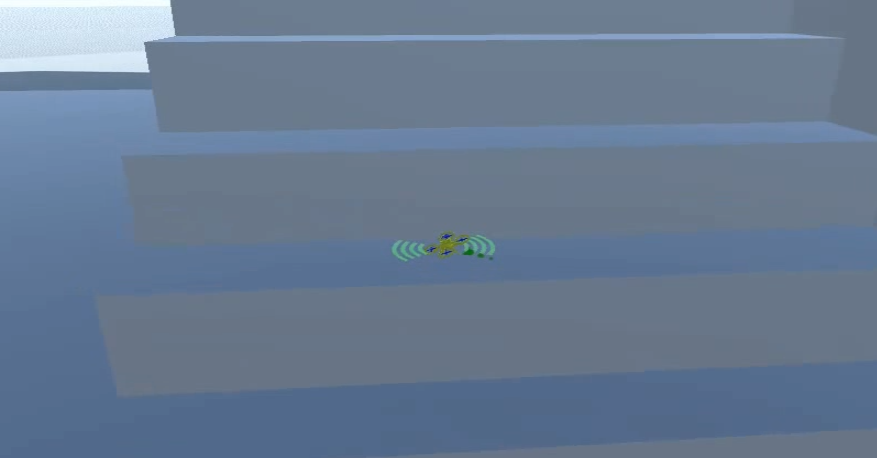}
\caption{Drone approaching a navigation point behind stairs.}
\label{fig:drone_infront_stairs}
\end{figure}

\begin{figure}[!ht]
\centering
\includegraphics[width=0.5\textwidth]{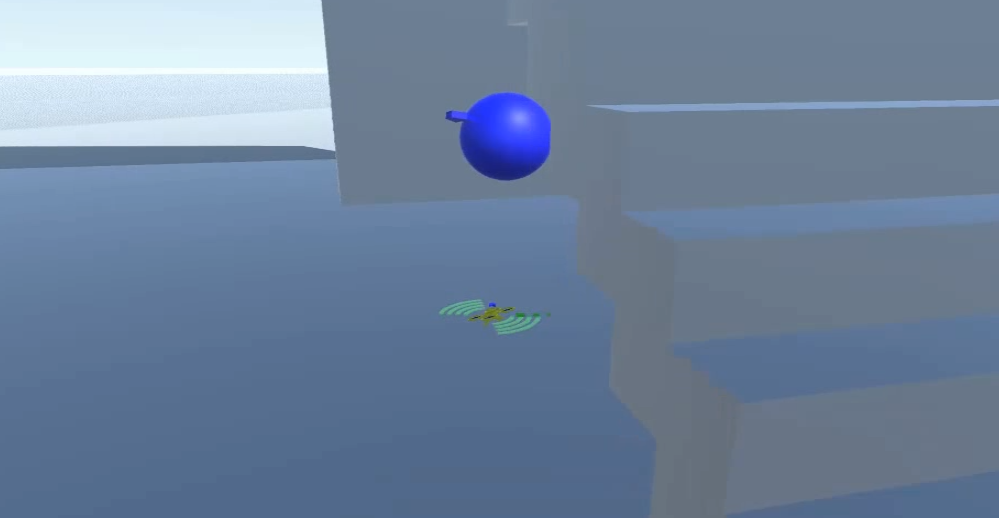}
\caption{Drone navigating around the stairs.}
\label{fig:drone_navigating_stairs}
\end{figure}

As is shown in Figures~\ref{fig:drone_infront_stairs} and~\ref{fig:drone_navigating_stairs}, the drone was shown how to navigate around obstacles, out of windows, around buildings, etc.  Typical sessions comprised of 1 to 2 minutes of initial demonstration followed by an additional 4-8 minutes of switching between the agent flying and the human taking over to make corrections or to redo a maneuver demonstrating how to navigate the obstacle properly.  We also implemented a button where the user could delete the most influential training data point, which proved useful on occasion when the drone was, for example, flying in a fixed direction away from a navigation point.  After this typical 5 to 10 minute session, our software was able to successfully navigate the drone from navigation point to navigation point.  Due to the sensitivity of the drone controls, the system needed to respond within a single digit number of milliseconds some of the time in order to avoid collisions, especially when the navigation points appeared in a challenging relative location.  We also observed that the drone's flying style matched that of the user.  For example, if the user flew cautiously or tended to zig-zag, the drone would perform similar behavior.  The environment was complex enough that we could tell that the behavior was generalizing to some degree, and the drone did not have access to global coordinates.

\begin{figure}[!ht]
\centering
\includegraphics[width=0.5\textwidth]{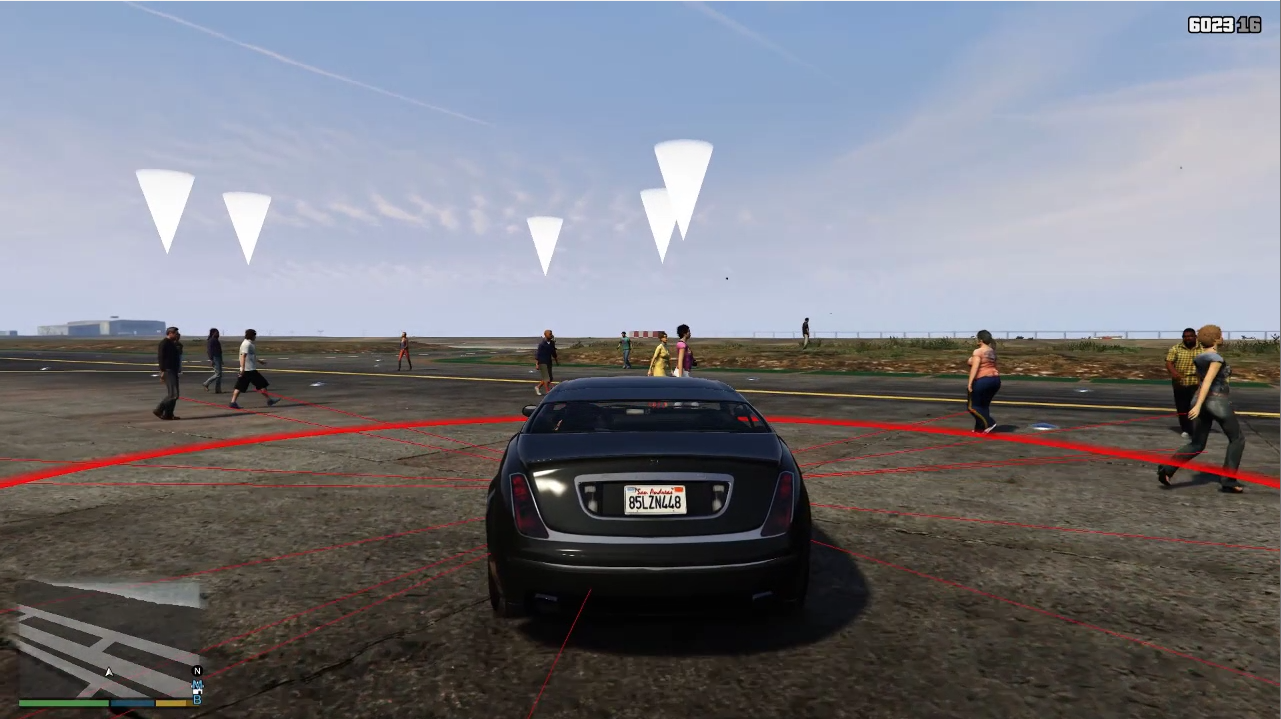}
\caption{Driving while avoiding crowd of pedestrians.}
\label{fig:gta_avoidance}
\end{figure}

\begin{figure}[!ht]
\centering
\includegraphics[width=0.5\textwidth]{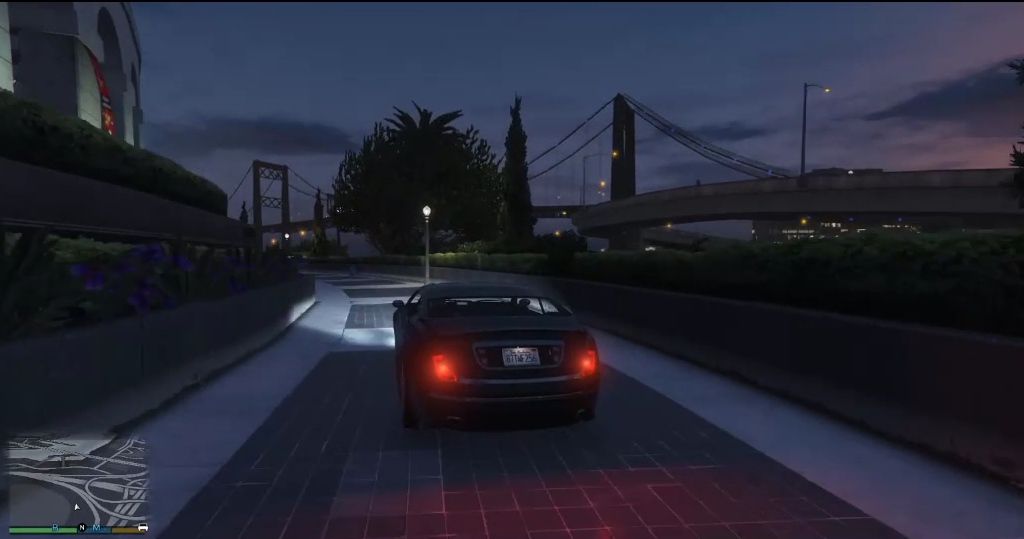}
\caption{Driving while staying within tight bounds.}
\label{fig:gta_path}
\end{figure}

The second learning from demonstration experiment was performed on the video game Grand Theft Auto V (GTAV).  In this setting, we used similar inputs, though the simulated LiDAR was cast out circling the car without variation in height.  In the experiment shown in Figure~\ref{fig:gta_avoidance}, the agent was taught to drive to faster and faster speeds while avoiding pedestrians who were walking about in a large but sparse crowd.  The agent was able to avoid pedestrians successfully at moderately high speeds while avoiding pedestrians in a period of time similar to the drone example of about 5 to 10 minutes.  In the second experiment shown in Figure~\ref{fig:gta_path}, the agent was taught to drive the car around a stadium in a tight lane, an area in GTAV similar to a race track.  About 2 minutes of training data was generally sufficient to train the agent to successfully navigate the oval path at high speed.  This scenario may have relied on memorization to some degree, but the small amount of training data (usually two laps, sometimes one lap was sufficient) indicate that some generalization was likely occurring because the path for any two laps is slightly different given human error and the controls.

\subsection{Reinforcement Learning}
\label{ssec:rl_results}

On CartPole~\citep{brockman2016openaigym}, our implementation meets the environment's built-in benchmark, having an average score of at least 195 over the last 100 episodes, by a median of 273 episodes across 5 runs, with a 100\% solve rate.  For context, a small study reported the REINFORCE algorithm needing around 650 episodes to solve CartPole, which improved to about 400 when employing a model-free approach with model dynamics~\citep{shaikh2019towards}, placing Howso in a comparable or faster sample-efficiency range.

On the Wafer-Thin-Mints game~\citep{bontrager2019superstition} with a 150-round evaluation, the average score ranges from -10 to 9, with 9 being the highest possible.  Our Howso agent achieved a mean score of 6.07 across 20 runs (median 6.59), with a 70\% win rate; among winning runs the mean was 6.87 (median 6.98). In prior literature, A2C agents trained from pixels typically failed to converge on this environment (mean score around $-6$).  Only the best planning-based approach, NovelTS, reported scores above Howso with a mean of 8.75, and MCTS averaged a score of 5.73.

Howso performs competently in both environments using a conviction of $\rho = 1$, which means it is simply exploring the part of the space with as much uncertainty as is in the data.  We believe this is somewhat profound in that it is able to perform well without an explicit learning rate, in addition to the system being richly interpretable back to the training data.  However, we found that setting $\rho = 3$ improved results for CartPole, so we used that for the performance metrics mentioned.  Wafer-Thin-Mints was measured with conviction of $\rho = 1$.  Full experimental details are provided in Appendix~\ref{ssec:appendix_rl}.

\subsection{Causal Discovery}
\label{ssec:empirical_causal_discovery}

\begin{table}[ht]
\centering
\begin{tabular}{lccccc}
\hline
\textbf{Dataset} & \textbf{Measure} & \textbf{GES} & \textbf{Howso} & \textbf{LiNGAM} & \textbf{PC} \\
\hline
GCM Microservice Architecture & F1 & 0.750 & 0.805 & 0.867 & \textbf{0.963} \\
 & Jaccard & 0.600 & 0.678 & 0.765 & \textbf{0.929} \\
GCM Online Shop & F1 & 0.552 & \textbf{0.665} & 0.581 & 0.556 \\
 & Jaccard & 0.381 & \textbf{0.498} & 0.409 & 0.385 \\
Sachs & F1 & 0.549 & \textbf{0.641} & 0.635 & 0.533 \\
 & Jaccard & 0.378 & \textbf{0.472} & 0.465 & 0.364 \\
GCM Supply Chain & F1 & 0.667 & 0.390 & 0.615 & \textbf{0.889} \\
 & Jaccard & 0.500 & 0.244 & 0.444 & \textbf{0.800} \\
Fluid Dynamics & F1 & 0.571 & 0.697 & 0.667 & \textbf{0.714} \\
 & Jaccard & 0.400 & 0.536 & 0.500 & \textbf{0.556} \\
\midrule
\textbf{average} & F1 & 0.618 & 0.639 & 0.673 & \textbf{0.731} \\
  & Jaccard & 0.452 & 0.486 & 0.517 & \textbf{0.606} \\
\hline
\end{tabular}
\caption{Causal discovery results scored according to undirected relationships.}
\label{tab:undirected_causal_results}
\end{table}

\begin{table}[ht]
\centering
\begin{tabular}{lccccc}
\hline
\textbf{Dataset} & \textbf{Measure} & \textbf{GES} & \textbf{Howso} & \textbf{LiNGAM} & \textbf{PC} \\
\hline
GCM Microservice Architecture & F1 & 0.667 & 0.758 & 0.828 & \textbf{0.923} \\
 & Jaccard & 0.500 & 0.615 & 0.706 & \textbf{0.857} \\
GCM Online Shop & F1 & 0.240 & 0.435 & 0.308 & \textbf{0.471} \\
 & Jaccard & 0.136 & 0.279 & 0.182 & \textbf{0.308} \\
Sachs & F1 & 0.267 & 0.434 & 0.436 & \textbf{0.455} \\
 & Jaccard & 0.154 & 0.279 & 0.279 & \textbf{0.294} \\
GCM Supply Chain & F1 & 0.545 & 0.250 & 0.500 & \textbf{0.889} \\
 & Jaccard & 0.375 & 0.144 & 0.333 & \textbf{0.800} \\
Fluid Dynamics & F1 & 0.143 & \textbf{0.636} & 0.571 & 0.462 \\
 & Jaccard & 0.077 & \textbf{0.467} & 0.400 & 0.300 \\
\midrule
\textbf{average} & F1 & 0.372 & 0.503 & 0.529 & \textbf{0.640} \\
  & Jaccard & 0.248 & 0.357 & 0.380 & \textbf{0.512} \\
\hline
\end{tabular}
\caption{Causal discovery results scored according to directed relationships.}
\label{tab:directed_causal_results}
\end{table}

We evaluated our causal discovery as described in Section~\ref{ssec:guided_feature_discovery} against the data sets of Graphical Causal Model (GCM) Microservice Architecture\footnote{\url{https://www.pywhy.org/dowhy/v0.11.1/example_notebooks/gcm_rca_microservice_architecture.html}}, GCM Online Shop\footnote{\url{https://www.pywhy.org/dowhy/v0.11.1/example_notebooks/gcm_online_shop.html}}, microbiology protein signaling network (Sachs)~\citep{sachs2005causal}, GCM Supply Chain\footnote{\url{https://www.pywhy.org/dowhy/v0.11.1/example_notebooks/gcm_supply_chain_dist_change.html}}, and data from a fluid dynamics simulation which has an inflow, outflow, control valve, level set point, and tank.  We evaluated our results versus Greedy Equivalent Search (GES)~\citep{chickering2020statistically}, Peter-Clark Algorithm (PC)~\citep{spirtes2000causation}, and Linear Non-Gaussian Acyclic Models (LiNGAM)~\citep{shimizu2014lingam}.  Figure~\ref{tab:undirected_causal_results} shows the results for undirected causal discovery and Figure~\ref{tab:directed_causal_results} shows the results for directed causal discovery.

PC is generally the strongest performer, but our system was strongest performer of the algorithms whenever LiNGAM was better than PC.  These data sets were selected to give a cross section of real-world problems and simulations with known ground truths and GCM implementations.  Given that our algorithm looks for signal asymmetries in uncertainty, it is not surprising that systems that look for noise direction could outperform ours in some situations, as we see in the empirical results.  Strengthening our methods on these types of problems, as well as further investigating when to use missing information ratios versus entropy asymmetries to decide directionality are avenues for future work.  Additional future work includes further causal benchmarking against more data sets~\citep{wang2024causalbench} and employing richer evaluation metrics~\citep{shimoni2018benchmarking}.

\begin{figure}[!ht]
\centering
\includegraphics[width=0.7\textwidth]{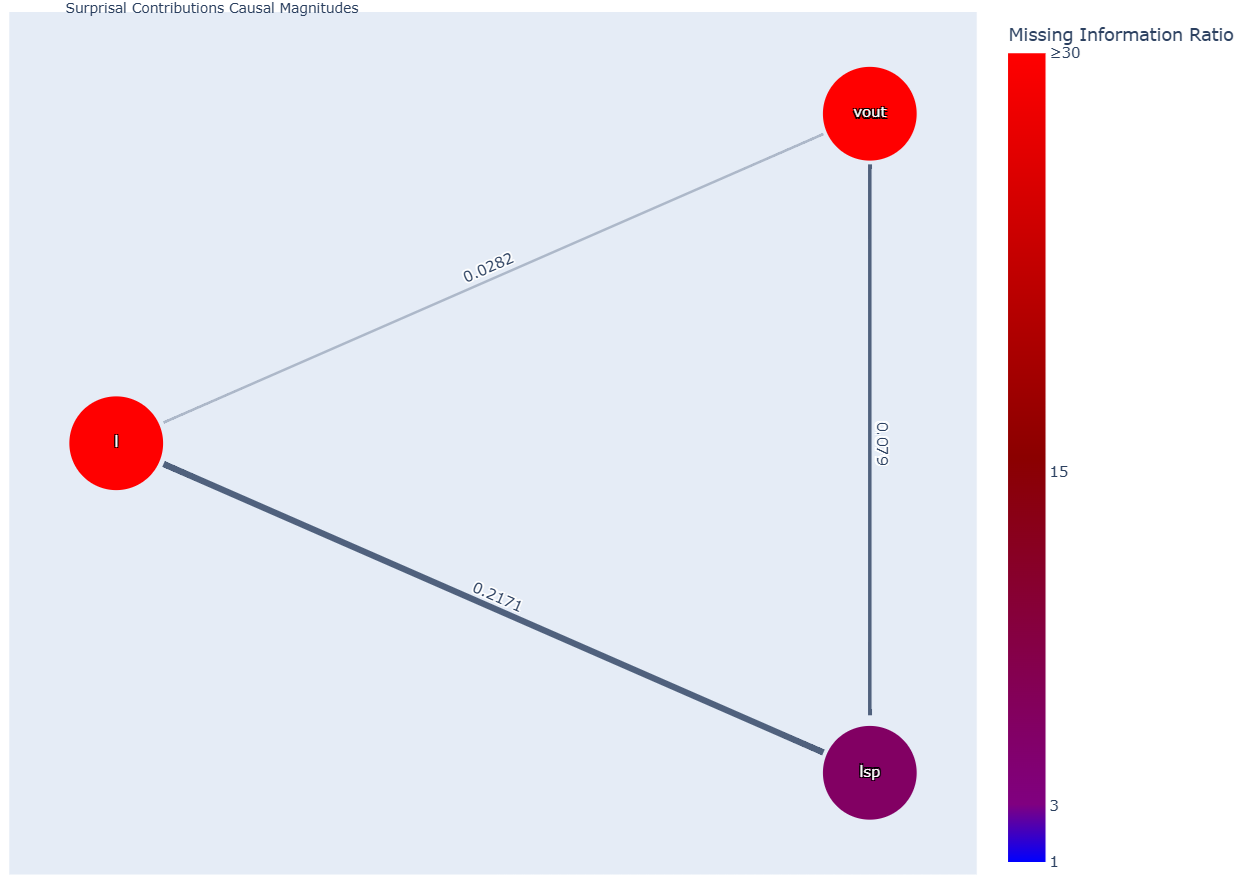}
\caption{A causal graph discovered for a fluid dynamics simulation using a small subset of features.}
\label{fig:missing_feature_discovery}
\end{figure}

We have found that our missing feature discovery and our certainty of edges is also useful for many data sets.  We are unaware of any current benchmarks to accurately assess these capabilities, and few causal discovery techniques support this.  Figure~\ref{fig:missing_feature_discovery} depicts an example of what the plots look like.  The edges are undirected because the missing information ratio is high, meaning that our system is not confident in the direction of the edges.  Further, the strongest relationship is between $l$ which is the level of water in the tank and $lsp$ which is the level set point that controls the valve.  The level is largely affected by fluid coming into the tank, and the volume of fluid going out $vout$ is also affected by this.  In this example, our system correctly assesses that there are additional causal features that are related and important but not included in the system, such as the volume coming in, $vin$, and the valve control position $c$, not included in this graph or in the training data for this test.  Qualitatively we have tested this system against commercial data sets and have found it to be very useful, but determining and measuring benchmarks for these capabilities remains a task for future work.  Anecdotally, our system seems to give more stable causal graphs than we have experienced with other causal techniques, a notable problem in causal discovery~\citep{hulse2025shaky}.  Characterizing this stability and determining whether and when our system is more stable also remains a task for future work.

\subsection{Data Synthesis}
\label{ssec:synth_results}

We evaluated the quality of synthetically generated data using the techniques described in Section~\ref{ssec:synth_data} by training models on 9 datasets from PMLB (4 regression, 5 classification), generating synthetic versions, and comparing model performance trained on original versus synthetic data across 10-fold cross-validation using LightGBM regressors/classifiers.  The synthetic data was generated with desired conviction of $5$ and new case generation enabled.  Results are shown in Table~\ref{tab:synth_summary} and we found that there was generally little loss of predictable signal in the synthesized data.  Full experimental details are provided in Appendix~\ref{ssec:appendix_synth}.  Previous work has found that an earlier version of our synthetic data preserved the most utility of all algorithms compared in the study~\citep{ling2024trading}.

% Auto-generated synthetic data generation summary table
% Generated by generate_synth_results.py

\begin{table}[!ht]
    \centering
    \begin{tabular}{lcc}
    \toprule
    Metric & Original & Generated \\
    \midrule
    Accuracy & 0.959 & 0.925 \\
    F1-Score & 0.959 & 0.927 \\
    MCC & 0.923 & 0.868 \\
    Spearman & 0.921 & 0.908 \\
    \bottomrule
    \end{tabular}
    \caption{Synthetic Data Generation Results: Overall Performance Comparison. Aggregated across 5 classification and 4 regression datasets.}
    \label{tab:synth_summary}
\end{table}

\subsection{Data Reduction and Ablation}

We evaluated the data reduction and ablation techniques described in Section~\ref{ssec:case_weights} across multiple benchmark datasets.  These data sets were selected from the PMLB data sets and are generally more challenging to perform data reduction on than most data sets we encounter in industry which tends to have significantly more redundancy.  The ablation process retained cases that contributed significant information while redistributing the probability mass of redundant cases to their influential neighbors.

\begin{table}[ht]
\centering
\begin{tabular}{lrrrr}
\hline
Dataset & Reduction (\%) & Howso Ablated & Howso With Subsample$^*$ & Howso With All Data\\
\hline
connect\_4 & 63.12 & 0.4968 & 0.4765 & 0.5696 \\
adult & 66.67 & 0.4907 & 0.5085 & 0.5022 \\
sleep & 87.89 & 0.5731 & 0.6441 & 0.6353 \\
fars & 51.75 & 0.6918 & 0.6860 & 0.7082 \\
krkopt & 24.91 & 0.6094 & 0.6248 & 0.6580 \\
letter & 27.74 & 0.9420 & 0.9524 & 0.9548 \\
\hline
\multicolumn{5}{l}{$^*$Randomly subsampled to the same size as ablated.} \\
\end{tabular}
\caption{Classification results comparing ablated Howso to Howso with all data using Matthews Correlation Coefficient (MCC).}
\label{tab:ablation_classification}
\end{table}

\begin{table}[ht]
\centering
\begin{tabular}{lrrrr}
\hline
Dataset & Reduction (\%) & Howso Ablated & Howso With Subsample$^*$ & Howso With All Data\\
\hline
1201\_BNG\_breastTumor & 53.29 & 0.3348 & 0.2996 & 0.3539 \\
feynman\_test\_1 & 35.95 & 0.9915 & 0.9889 & 0.9948 \\
feynman\_test\_20 & 23.22 & 0.9866 & 0.9780 & 0.9855 \\
215\_2dplanes & 61.97 & 0.9653 & 0.9699 & 0.9703 \\
344\_mv & 47.04 & 0.9973 & 0.9986 & 0.9987 \\
564\_fried & 62.52 & 0.9538 & 0.9566 & 0.9655 \\
1193\_BNG\_lowbwt & 59.23 & 0.7443 & 0.7440 & 0.7484 \\
574\_house\_16H & 21.84 & 0.7920 & 0.7909 & 0.7953 \\
\hline
\multicolumn{5}{l}{$^*$Randomly subsampled to the same size as ablated.} \\
\end{tabular}
\caption{Regression results comparing ablated Howso to Howso with all data using Spearman correlation coefficient.}
\label{tab:ablation_regression}
\end{table}

Tables~\ref{tab:ablation_classification} and~\ref{tab:ablation_regression} show the results for classification and regression tasks respectively.  We chose the Matthews correlation coefficient (MCC) metric as it is highly sensitive to all classification error types, and chose the Spearman correlation coefficient for regression tasks since it is more robust across distributions than $R^2$.

Across the evaluated datasets, case ablation achieved substantial data reduction while maintaining performance comparable to the full dataset, though the results are stronger over random subsampling for regression tasks than classification tasks.  The ablated results demonstrate that appropriate data is retained for accurate predictions.  Even though in some cases the subsampling performed as well as or slightly better than ablation, the data ablation process was able to find a subsample size that was still able to maintain quality.  We emphasize that this method of data ablation and data reduction does not focus on any individual target, but rather is designed to address general representative data selection.  Future work includes improving the quality of these results and deeper data reduction.

\subsection{Adversarial Robustness}

\begin{table}[!ht]
\centering
\begin{tabular}{lccc}
\toprule
Model & Clean Acc & K=3 Adv Acc & K=1 Adv Acc \\
\midrule
Howso & 0.819 & 0.596 ($-0.223$) & 0.737 ($-0.082$) \\
Logistic Regression & 0.795 & 0.684 ($-0.111$) & 0.742 ($-0.053$) \\
Decision Tree & 0.783 & 0.564 ($-0.220$) & 0.641 ($-0.142$) \\
kNN & 0.796 & 0.436 ($-0.360$) & 0.757 ($-0.039$) \\
LightGBM & 0.835 & 0.619 ($-0.216$) & 0.717 ($-0.119$) \\
XGBoost & 0.834 & 0.633 ($-0.201$) & 0.691 ($-0.143$) \\
\bottomrule
\end{tabular}
\caption{Mean accuracy under adversarial attack (27 datasets)}
\label{tab:robustness_main}
\end{table}

\begin{figure}[!ht]
\centering
\includegraphics[width=0.49\textwidth]{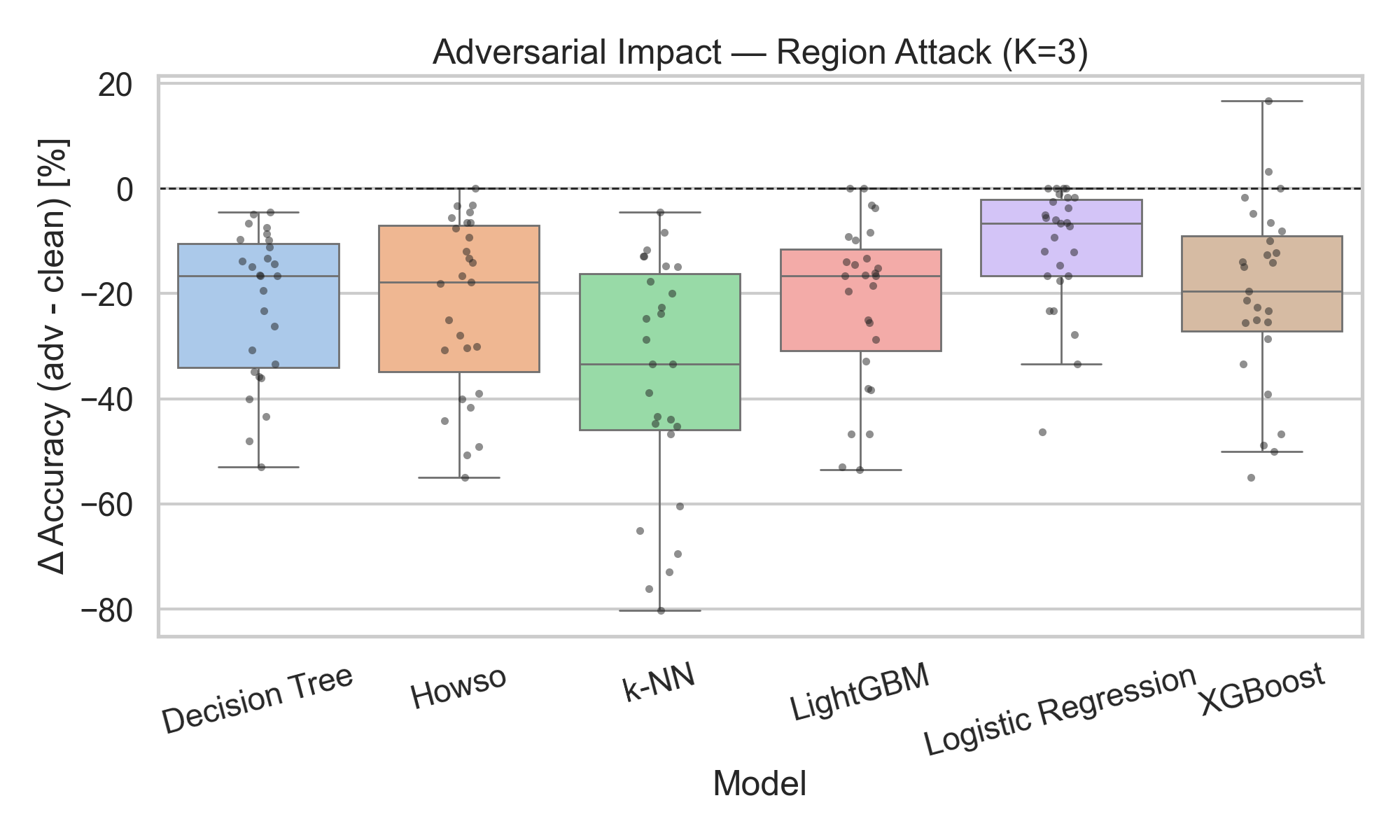}
\includegraphics[width=0.49\textwidth]{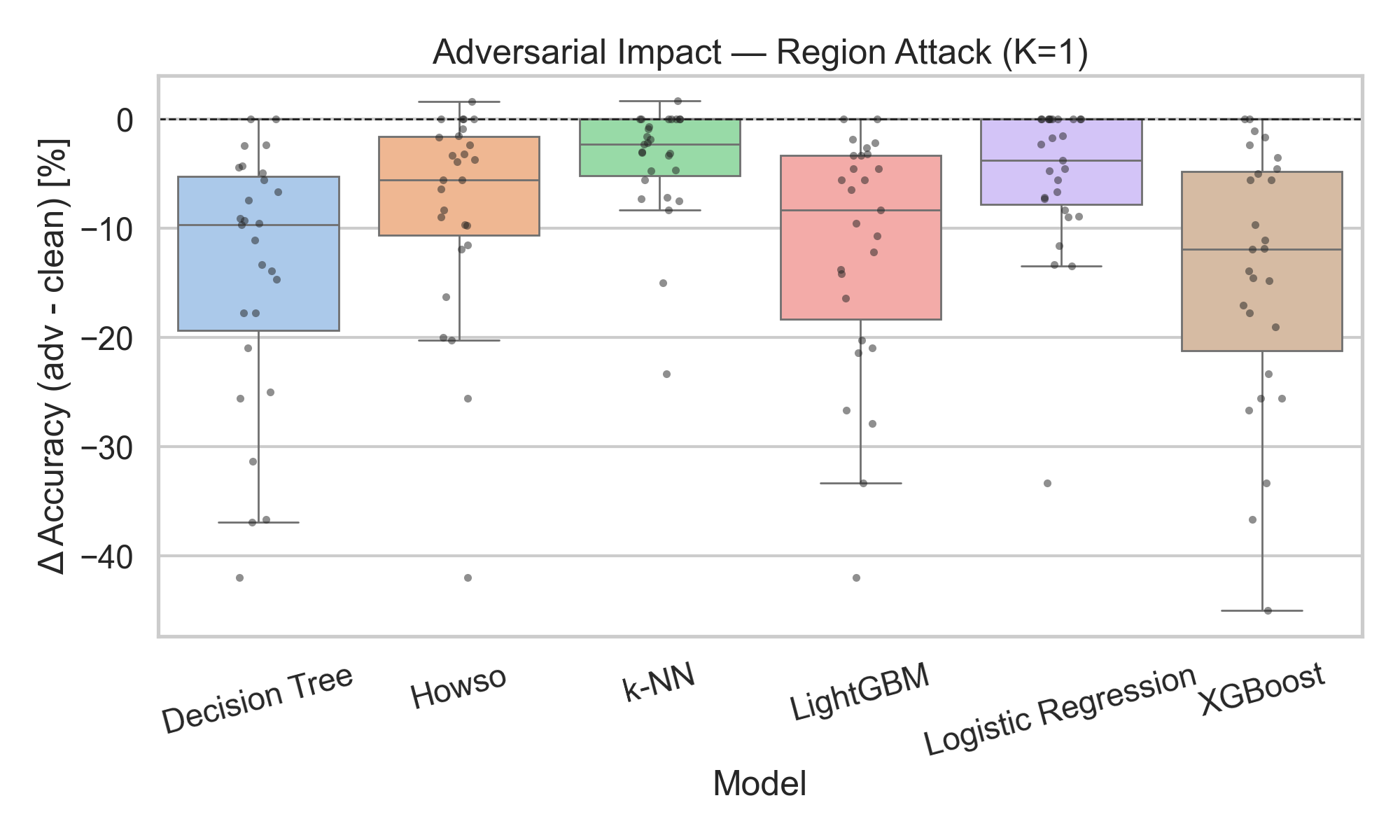}
\caption{Distribution of accuracy degradation ($\Delta$ Accuracy) by model and attack type across 27 datasets. Box plots show median, IQR, and individual dataset results.}
\label{fig:robustness_boxplots}
\end{figure}

We evaluated Howso's robustness to adversarial attacks using Region-Based Adversarial (RBA) attacks~\citep{yang2020robustness} across 27 tabular classification datasets.  The RBA attack constrains perturbations to the convex region defined by the K most similar training cases, generating adversarial examples that remain within locally plausible feature space.  Two attack configurations were tested: K=1 (single-neighbor, tighter constraint) and K=3 (three-neighbor, looser constraint).  Howso was compared against five baselines: k-Nearest Neighbors, Decision Tree, Logistic Regression, LightGBM, and XGBoost.  The results are shown in Table~\ref{tab:robustness_main} and Figure~\ref{fig:robustness_boxplots}.

The $K=3$ attack was more challenging to all models, with instance-based learners showing greater vulnerability than the parametric logistic regression model.  Under the tighter K=1 constraint, all models demonstrated improved robustness.  Howso's was less vulnerable to attacks with K=3 than the other instance-based learner, kNN, and Howso was less vulnerable to attacks with K=1 than everything except the two simpler algorithms, Logistic Regression and kNN.  Detailed per-dataset results and analysis are provided in Appendix~\ref{sec:appendix_robustness}.

\subsection{Compute Performance}

\begin{figure}[!ht]
\centering
\includegraphics[width=0.9\textwidth]{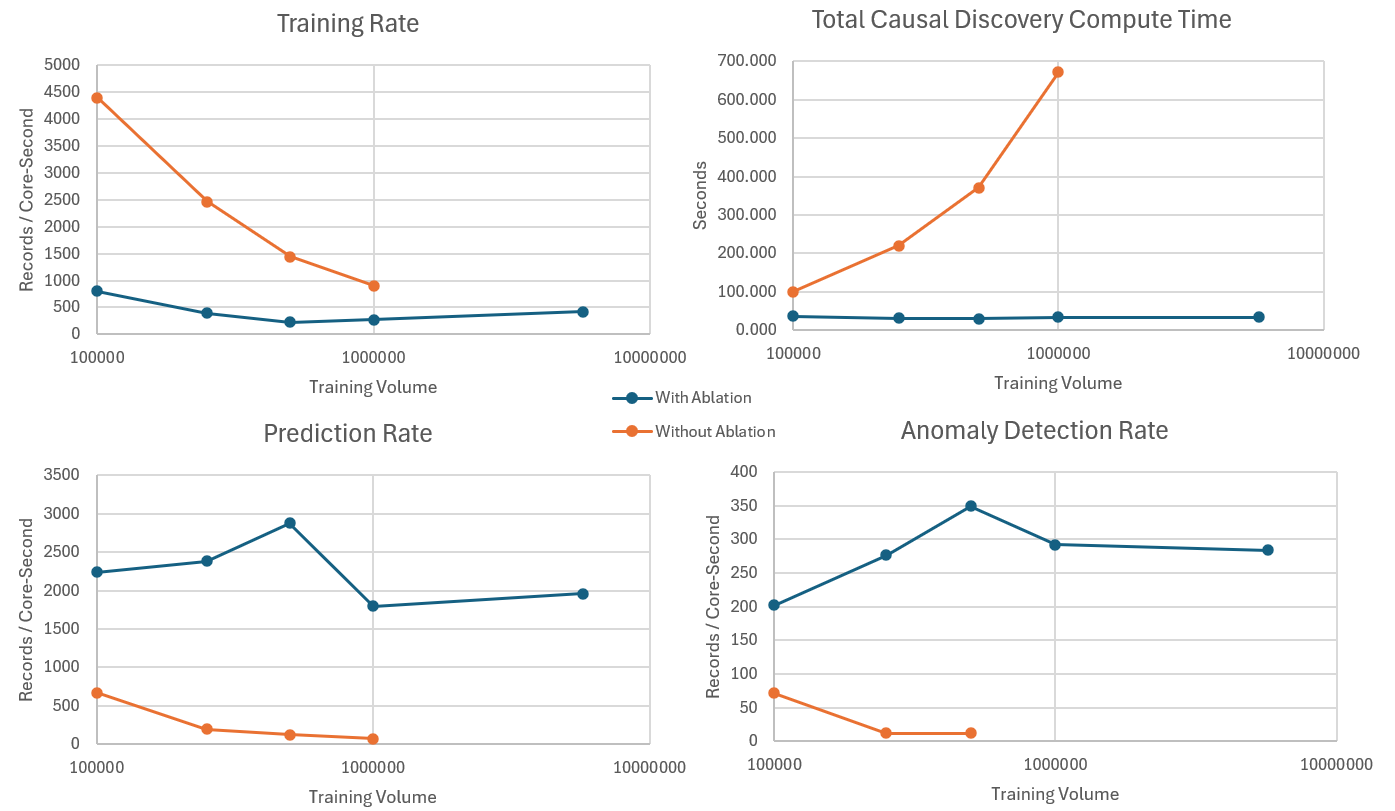}
\caption{Four measures of performance on different volumes of the NYC arrest dataset.}
\label{fig:performance}
\end{figure}

Figure~\ref{fig:performance} depicts some performance metrics with different amounts of training data using the New York City historical arrest dataset,\footnote{\url{https://data.cityofnewyork.us/Public-Safety/NYPD-Arrests-Data-Historic-/8h9b-rp9u/about_data}} which is approximately 6 million records by 18 features, and measured on a modern 16 core processor.  The environment and performance metrics include the overhead of loading the data via Python and Python scripts driving all of the interactions, serializing and deserializing the data back and forth via JSON.  The orange lines indicate performance without data ablation and data reduction, and the blue lines indicate with.  The training time graph includes the relevant analyze, ablation, and data reduction processes in the per-record overhead.  Without ablation, the performance degrades, but with ablation, the performance starts notably lower but increases as the scale goes up.  With data ablation, causal discovery, anomaly detection, and prediction rates stay steady independent of original data volume, as long as additional data does not add richness to the complexity of the model.  In other words, more complex data may require a larger ceiling of data kept to maintain inference quality.

Future work includes advancing the use of hierarchical data sharding as described in Section~\ref{ssec:hierarchical_sharding} which may logarithmically increase inference times with data.  Further, as of the time of this writing, we have focused optimization efforts on inference, and are now working toward a number of optimizations that should continue to improve performance for training times.  

\section{Future Work: Language, Images, and Creativity}

There are many avenues for future work which would enhance usability and applicability, including improving the quality of inference, causal discovery, anomaly detection, etc., as well as adding support for more data types, and adding more types of dynamic solving of complex expressions (such as combinatorial optimization).  In addition to those tactical areas of future work, here we highlight future work to move beyond structured, semistructured, and time series data. 

\subsection{Text}

\begin{table}[ht]
\centering
\begin{tabularx}{\textwidth}{|>{\hsize=0.2\hsize}X|>{\hsize=0.4\hsize}X|>{\hsize=0.4\hsize}X|}
\hline
Prompt & Completed Sentence & Notes \\
\hline
The Solar System has & The Solar System has been relatively well studied, initially through telescopes and then later by spacecraft & The source was the article titled Astronomy, verbatim, except stopping before the period at the end. \\
\hline
The Solar System has & The Solar System has been fairly stable for billions of years, it is technically chaotic, and may be considered a planet. & The source was the article Solar\_System, changed at the end, gramatically correct, topically correct, but factually incorrect. \\
\hline
Jupiter is the & Jupiter is the most important deity in the Guatemalan Highlands at this time was Q'umarkaj, the capital of the aggressive K'iche' kingdom. & The source was a mixture of articles on the Maya people, and the sentence has strong factual basis, though the grammar is a little off.  The K'iche' people and the Maya people in general tracked Jupiter's orbit and incorporated it into their calendars.  The K'iche' kingdom was described in the Wikipedia texts as having subdued other Mayan peoples. \\
\hline
Jupiter is the & Jupiter is the largest of the Western Roman Empire, and the the king of the planet's connection with Easter & Example of a nonsensical sentence that is grammatically correct except for the missing period. \\
\hline
\end{tabularx}
\caption{Typical sentence completions based on approximately 1 million words trained from Wikipedia.}
\label{tab:text_examples}
\end{table}

A major future direction for this technology is to apply all of the relevant capabilities discussed in this work to unstructured text.  We have begun implementing some early experiments in this domain which have shown promise.  In general, we leverage the time series capabilities, and the number of time series lags is effectively the context window.  Using basic tokenization around word and punctuation boundaries on small data sets, we find that it is able to effectively memorize text and generalize enough to typically stay topical and produce mostly reasonable grammar with only a few mistakes per sentence.  Table~\ref{tab:text_examples} shows the results of completing sentences given an initial prompt given a training data set of approximately 1 million words from Wikipedia covering a range of science and history topics.  We note that when to terminate a sentence is learned entirely from the training data.

Early experiments with sentiment classification look promising as well.  On the MTEB semantic classification benchmark (tweet sentiment extraction)~\citep{enevoldsen2025mmteb} treating the text as time series but the sentiment as a stationary feature as described in Section~\ref{ssec:time_series}, our system's accuracy ranges is currently in the vicinity of 54\%.  Compared to the public leaderboards,\footnote{\url{https://huggingface.co/spaces/mteb/leaderboard}} this is in the upper portion of the lower half of all results.  However, we are obtaining these results with zero pretraining data and virtually every other technique on the leaderboard has significant pretraining.  There is even a column on the leaderboard indicating how many millions of parameters the models have.  Future work will include pretraining to boost this score.

Other capabilities interact in interesting ways with text.  For example, sparse deviation matrices described in Section~\ref{ssec:deviations} act essentially as a confusion matrix between interchangeability of words.  Data ablation on time series data, described in Section~\ref{ssec:time_series}, may offer a human-readable alternative to embeddings where unnecessary words are cut out.  It is hard to say how the performance and effective compression ratio of these embedding alternatives will compare to embeddings.  Simple concepts could be reflected in fewer words than a full vector, however constraints related to grammar may make it harder to remove chunks to improve the compression.  Future work will investigate the efficacy of such an implementation.

Another direction is to find ways to vastly improve the time series context length beyond what is currently maximally feasible, which currently is in the realm of 1000-2000 tokens on a single machine.  We believe this can be accomplished by performing ablation and data reduction in the context, specifically to dynamically compress the earliest data in the context the most.  Using surprisal to mark areas for compression has been shown to vastly improve context windows in LLMs~\citep{fountas2025humaninspired}, and we believe similar techniques of improving LLM context lengths will have analogues with this system.

Longer term, the hierarchical techniques described in Section~\ref{sec:compression_hierarchy} will be required to improve performance.  The dynamic derivation of code during inference, as described in Section~\ref{ssec:derived_and_dependent_feat}, may offer paths for tighter, closer tool integration to match rules and derive results for arithmetic, logic, and potentially optimization and symbolic reasoning.  Because these mechanisms themselves do not need to be learned, training can be vastly reduced to only include training data to indicate when these tools should be employed.  This could potentially significantly reduce training data as compared to LLMs require in order to learn these capabilities due to the vast reduction in dimensionality.

\subsection{Images}

\begin{figure}[!ht]
\centering
\includegraphics[width=0.1\textwidth]{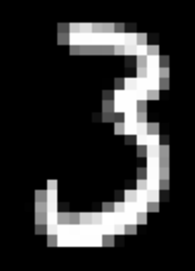}
\caption{An example of a number generated using the MNIST handwritten digits data set.}
\label{fig:generated_3}
\end{figure}

Another future direction for the methods described here are to work on image data.  Though the scores for the MNIST hand written data set are in the 91\% to 92\% range, with  a little manual tuning we have achieved results in the 96\% to 97\% range, and we are hopeful that additional work toward improving deviation calculations and feature influence probabilities may be able to make those results achievable automatically.  Additionally, preliminary experiments with adversarial attacks on MNIST have shown Howso to be more robust and maintain higher accuracy than any technique other than basic kNN.  And we have had some success generating images from the MNIST data set.  Figure~\ref{fig:generated_3} depicts an example of a generated image using Howso that is slightly better than average but not atypical.

Furthering the work of hierarchical implementation as described in Section~\ref{sec:compression_hierarchy} will be essential.  That capability may help with some of the general problem for some tasks, but including additional features by employing understandable and possibly reversible convolutional transformations may be required to achieve high quality results.  KNN has had a long history of being useful in images, even at scale with billions of images~\citep{johnson2019billion}, and we hope that our work here can help advance these applications.  Treating images as time series data may also provide a path for eventually supporting video.

\subsection{Creativity}

Creativity is often defined as being novel compared to what is known and having value~\citep{runco2012standard}, and though there are psychological tests to measure creativity in human subjects via proxy tasks~\citep{kim2006can}, there has been little research into measuring creativity from an information theory perspective (rare exceptions do exist, such as that by~\cite{coffman1992measuring}).  As our system can dynamically balance exploration versus exploitation in online learning without specifying a learning or exploration rate (as discussed in Section~\ref{ssec:rl_results}), a natural extension for future work is to apply our techniques towards measuring and governing creativity relative to a base data set of knowledge.

Suppose we have a particular configuration or data record, $x_F$, and we would like to measure its creativity against a set of known data $X_F$.  The relative novelty of $x_F$ to some other configuration $x^{\prime}_F$ can be expressed as $I(x_F | x^{\prime}_F)$.  Further suppose we have a set of $v \in V$ functions that measure value of a configuration, as $v(x_F)$.  Using the geometric mean to combine measurements of achieving different goals has been shown to be an effective objective function for multicriteria optimization which is also scale invariant to the measurements~\citep{harrington1965desirability}, and so using it to provide contrast between different similarities is a natural use.  \cite{leinster2016maximizing} show that the generalized diversity index, which can be parameterized to measure the Shannon entropy, is nothing more than the reciprocal of the generalized mean when substituting $p-1$ for $p$ when dealing with probabilities~\citep{tuomisto2010consistent}.  Finding a geometric mean of values is the same as finding the arithmetic mean of their logarithms and reversing the transformation.  In light of this, we can combine the arithmetic mean of the logarithms with measurement uncertainty to bring these values into surprisal as described in Section~\ref{sec:foundational_statistics}.  This surprisal of value, $I\left(v(x_F) \mid v(x^{\prime}_F)\right)$ can be expressed in terms of the deviation of value, $\delta(x_F)$, as
\begin{equation}
\label{eq:surprisal_of_value}
I\left(v(x_F) \mid x_F, v(x^{\prime}_F)\right) = \frac{1}{|V|} \sum_{v \in V} \frac{\left|v(x_F) - v(x^{\prime}_F)\right|}{\delta(x_F)}.
\end{equation}

We can condition Equation~\ref{eq:surprisal_of_value} based on whether the outcome is positive to represent how much surprisal of beneficial value would be gained from data $x_F$.  This caps the loss at zero and can be expressed as
\begin{equation}
\label{eq:surprisal_of_beneficial_value}
I\left(v_+(x_F) \mid x_F, v(x^{\prime}_F)\right) = \frac{1}{|V|} \sum_{v \in V} \frac{\max\left\{ v(x_F) - v(x^{\prime}_F), 0 \right\}}{\delta(x_F)}.
\end{equation}
Recognizing that the surprisal of a loss is just reversed, we can combine both positive and negative into a real valued quasi-surprisal, $I^*$, as
\begin{equation}
\label{eq:directional_surprisal_of_value}
I^*\left(v_{\pm}(x_F) \mid x_F, v(x^{\prime}_F)\right) = \frac{1}{|V|} \sum_{v \in V} \frac{v(x_F) - v(x^{\prime}_F)}{\delta(x_F)}.
\end{equation}

Additionally, in many domains, complexity is a factor of value, and elegance is often associated with creativity in some domains.  The difference in complexity of a configuration $x_F$ relative to $x^{\prime}_F$ can be computed as $I(x_F) - I(x^{\prime}_F)$.

Putting all of these elements together, we can find the difference between the novelty and value of a given configuration and the best competing configuration of what is known.  Measuring the creativity of configuration $x_F$ against a data set of knowledge $X_F$ can be expressed as
\begin{equation}
\label{eq:measuring_creativity}
%I^*(x_F | X_F) = \min_{x^{\prime}_F \in X_F} \left\{ I(x_F | x^{\prime}_F) - \left( I(x_F) - I(x^{\prime}_F) \right) + \frac{1}{|V|} \sum_{v \in V} \left( \ln v(x_F) - \ln v(x^{\prime}_F) \right) \right\}.
I^*(x_F | X_F) = \min_{x^{\prime}_F \in X_F} \left\{ I(x_F | x^{\prime}_F) - \left( I(x_F) - I(x^{\prime}_F) \right) + \frac{1}{|V|} \sum_{v \in V} \frac{v(x_F) - v(x^{\prime}_F)}{\delta(x_F)} \right\}.
\end{equation}
We note that the inclusion of goal features as described in Section~\ref{ssec:constraining_and_goals} already implicitly incorporate some aspects of these value functions into the queries, and aspects of Equation~\ref{eq:measuring_creativity}, are already implicitly performed during reinforcement learning (as described in Section~\ref{ssec:reinforcement_learning_active_inference}).  Future work involves validating this approach to measuring creativity, understanding how best to trade-off novelty versus utility, and applying it further to reinforcement learning and active inference.

\section{Conclusion}

In conclusion, we have demonstrated that instance-based learning built on information theory is a useful and mathematically powerful way to recontextualize machine learning into probability and statistics while maintaining strong and robust inference capabilities.  Our hope is that this work inspires others to continue research on how we can get understanding and knowledge directly from data, where humans can understand every step of the process to debug, diagnose, remediate, and control more complex systems.  At larger scales of data with more complex data sets, there will be different emergent layers that will need to be modeled.  Our hierarchical approach may be able to decouple these layers when they are easily separable, and future work may help human understandability for complex systems with information leakage between layers of emergence such as molecular biology~\citep{rosas2024software, rosas2025characterising}.  Rather than continuing to build larger and larger models that are increasingly opaque, we have demonstrated another potential path to help discover the physics of information.

\section*{Acknowledgments}
\addcontentsline{toc}{section}{Acknowledgments}
\phantomsection
This work is the culmination of continuous effort from the entire research and engineering team at Howso Incorporated, formerly Diveplane Corporation.  Diveplane was founded in 2017 upon the technology of Hazardous Software Inc., which had began research and development on this technology in 2011.  We gratefully acknowledge the effort of our dedicated teams including current and former employees, our government and commercial customers, our individual and institutional investors, and our friends and families on this journey for all of their support.

\bibliographystyle{abbrvnat}
\bibliography{diveplane}

\appendix
\appendixpage
\section{Supporting Mathematics}

\subsection{Derivation of {\L}ukaszyk–Karmowski (LK) with Laplace Distributions}
\label{subsec:lk_laplace_derivation}

\begin{theorem}
The expected distance between two random variables $X$ and $Y$ given that they both follow the Laplace distribution with a distributional parameter $b$ and with respective means $\mu_1$ and $\mu_2$ can be found as $d(X, Y) = |\mu_2 - \mu_1| + \frac{1}{2} e^\frac{-|\mu_2 - \mu_1|}{b} \left( 3 b + |\mu_2 - \mu_1| \right)$.
\end{theorem}

\begin{proof}
We begin with the expected distance between two random variables $X$ and $Y$ given two probability density functions, $f(x)$ and $g(y)$ as
\begin{equation}
d(X, Y) = \int_{-\infty}^\infty \int_{-\infty}^\infty |x-y| f(x) g(y) \, dx\, dy.
\end{equation}

Using two Laplace distributions with means $\mu_1$ and $\mu_2$ and expected distance from the mean $b_1$ and $b_2$, we can express the probability density functions as
\begin{equation}
f(x) = \frac{1}{2 b_1} e^{\frac{-|x - \mu_1|}{b_1}}
\end{equation}
and 
\begin{equation}
g(y) = \frac{1}{2 b_2} e^{\frac{-|y - \mu_2|}{b_2}}
\end{equation}
respectively.

Substituting in the Laplace distributions into the expected distance, we can simplify this slightly as
\begin{align*}
d(X, Y) &= \int_{-\infty}^\infty \int_{-\infty}^\infty |x-y| \cdot \frac{1}{2 b_1} e^{\frac{-|x - \mu_1|}{b_1}} \cdot \frac{1}{2 b_2} e^{\frac{-|y - \mu_2|}{b_2}} \, dx\, dy \\
&= \frac{1}{4 b_1 b_2} \int_{-\infty}^\infty \int_{-\infty}^\infty |x-y| \cdot e^{\frac{-|x - \mu_1|}{b_1}} \cdot e^{\frac{-|y - \mu_2|}{b_2}} \, dx\, dy.
\end{align*}

We further assume that $b_1 = b_2$, and use $b$ in place of both, which assumes that the error is the same throughout the space and simplify further as
\begin{equation}
\label{eq:laplace_lk_symmetric_b}
% d_{v \leq \mu_1}(X, Y) = \frac{1}{4 b^2} \int_{-\infty}^{\infty} \int_{-\infty}^{\infty} |x-y| \cdot e^{\frac{-|\mu_1 - x|}{b}} \cdot e^{\frac{-|\mu_2 - y|}{b}} \, dx\, dy.
d(X, Y) = \frac{1}{4 b^2} \int_{-\infty}^{\infty} \int_{-\infty}^{\infty} |x-y| \cdot e^{\frac{-|\mu_1 - x|}{b}} \cdot e^{\frac{-|\mu_2 - y|}{b}} \, dx\, dy.
\end{equation}

Because we only have one value for $b$, we can assume that $\mu_1 \leq \mu_2$ without loss of generality because we can just exchange the values if this is not true, and in the end we will adjust the formula to remove this assumption.  There exist 3 regions of the space for $x$ which are $x \leq \mu_1$, $\mu_1 < x \leq \mu_2$, and $\mu_2 < x$.

Rewriting Equation~\ref{eq:laplace_lk_symmetric_b} for the part of the space where $x < \mu_1, y < \mu_1$ is
\begin{align*}
d_{x \leq \mu_1, y \leq \mu_1}(X, Y) &= \frac{1}{4 b^2} \int_{-\infty}^{\mu_1} \int_{-\infty}^{\mu_1} |y - x| \cdot e^{\frac{-(\mu_1 - x)}{b}} \cdot e^{\frac{-(\mu_2 - y)}{b}} \, dx\, dy \\
 &= \frac{1}{4 b^2} \int_{-\infty}^{\mu_1} \left( \int_{-\infty}^{\mu_1} |y - x| \cdot e^{\frac{-(\mu_1 - x)}{b}} \, dx \right) \cdot e^{\frac{-(\mu_2 - y)}{b}} \, dy \\
  &= \frac{1}{4 b^2} \int_{-\infty}^{\mu_1} \left( \int_{-\infty}^{y} (y - x) \cdot e^{\frac{-(\mu_1 - x)}{b}} \, dx + \int_{y}^{\mu_1} (x - y) \cdot e^{\frac{-(\mu_1 - x)}{b}} \, dx \right) \cdot e^{\frac{-(\mu_2 - y)}{b}} \, dy \\
  &= \frac{1}{4 b^2} \int_{-\infty}^{\mu_1} \left( b^2 e^\frac{y-\mu_1}{b} + b^2 e^\frac{y-\mu_1}{b} - b y + b \mu_1 - b^2 \right) \cdot e^{\frac{-(\mu_2 - y)}{b}} \, dy \\
  &= \frac{1}{4 b} \int_{-\infty}^{\mu_1} \left( 2 b e^\frac{y-\mu_1}{b} - y + \mu_1 - b \right) \cdot e^{\frac{-(\mu_2 - y)}{b}} \, dy \\
  &= \frac{1}{4 b} b^2 e^\frac{-(\mu_2 - \mu_1)}{b}\\
  &= \frac{1}{4} b e^\frac{-(\mu_2 - \mu_1)}{b}.
\end{align*}

Rewriting Equation~\ref{eq:laplace_lk_symmetric_b} for the part of the space where $x < \mu_1, \mu_1 < y < \mu_2$ is
\begin{align*}
d_{x \leq \mu_1, \mu_1 < y \leq \mu_2}(X, Y) &= \frac{1}{4 b^2} \int_{\mu_1}^{\mu_2} \int_{-\infty}^{\mu_1} (y - x) \cdot e^{\frac{-(\mu_1 - x)}{b}} \cdot e^{\frac{-(\mu_2 - y)}{b}} \, dx\, dy \\
&= \frac{1}{4 b^2} \int_{\mu_1}^{\mu_2} \left( b y - b \mu_1 + b^2\right) \cdot e^{\frac{-(\mu_2 - y)}{b}} \, dy \\
&= \frac{1}{4 b} \int_{\mu_1}^{\mu_2} \left( y - \mu_1 + b\right) \cdot e^{\frac{-(\mu_2 - y)}{b}} \, dy \\
&= \frac{1}{4 b} \left( b \mu_2 - b \mu_1 \right) \\
&= \frac{1}{4} \left( \mu_2 - \mu_1 \right) .
\end{align*}

Rewriting Equation~\ref{eq:laplace_lk_symmetric_b} for the part of the space where $x < \mu_1, \mu_2 < y$ is
\begin{align*}
d_{x \leq \mu_1, \mu_2 < y}(X, Y) &= \frac{1}{4 b^2} \int_{\mu_2}^{\infty} \int_{-\infty}^{\mu_1} (y - x) \cdot e^{\frac{-(\mu_1 - x)}{b}} \cdot e^{\frac{-(y - \mu_2)}{b}} \, dx\, dy\\
&= \frac{1}{4 b^2} \int_{\mu_2}^{\infty} \left( \int_{-\infty}^{\mu_1} (y - x) \cdot e^{\frac{-(\mu_1 - x)}{b}} \, dx\, \right) \cdot e^{\frac{-(y - \mu_2)}{b}} dy\\
&= \frac{1}{4 b^2} \int_{\mu_2}^{\infty} \left( b y - b \mu_1 + b^2  \right) \cdot e^{\frac{-(y - \mu_2)}{b}} dy\\
&= \frac{1}{4 b} \int_{\mu_2}^{\infty} \left( y - \mu_1 + b  \right) \cdot e^{\frac{-(y - \mu_2)}{b}} dy\\
 &= \frac{1}{4b} \left( b \mu_2 - b \mu_1 + 2 b^2 \right) \\
 &= \frac{1}{4} \left( \mu_2 - \mu_1 + 2 b \right).
\end{align*}

Rewriting Equation~\ref{eq:laplace_lk_symmetric_b} for the part of the space where $\mu_1 < x < \mu_2, y < \mu_1$ is
\begin{align*}
d_{\mu_1 < x \leq \mu_2, y \leq \mu_1}(X, Y) &= \frac{1}{4 b^2} \int_{-\infty}^{\mu_1} \int_{\mu_1}^{\mu_2} (x - y) \cdot e^{\frac{-(x - \mu_1)}{b}} \cdot e^{\frac{-(\mu_2 - y)}{b}} \, dx\, dy\\
 &= \frac{1}{4 b^2} \int_{-\infty}^{\mu_1} \left( \int_{\mu_1}^{\mu_2} (x - y) \cdot e^{\frac{-(x - \mu_1)}{b}} \, dx \right) \cdot e^{\frac{-(\mu_2 - y)}{b}} \, dy\\
 &= \frac{1}{4 b^2} \int_{-\infty}^{\mu_1} \left( e^\frac{\mu_1 - \mu_2}{b} \left( b y - b \mu_2 - b^2 \right) - b y + b \mu_1 + b^2 \right) \cdot e^{\frac{-(\mu_2 - y)}{b}} \, dy\\
 &= \frac{1}{4 b} \int_{-\infty}^{\mu_1} \left( e^\frac{\mu_1 - \mu_2}{b} \left( y - \mu_2 - b \right) - y + \mu_1 + b \right) \cdot e^{\frac{-(\mu_2 - y)}{b}} \, dy\\
 &= \frac{1}{4 b} \left( 2 b^2 e^\frac{\mu_1-\mu_2}{b} + e^\frac{2 \mu_1 - 2 \mu_2}{b} \left(-b \mu_2 + b \mu_1 - 2 b^2 \right) \right) \\
 &= \frac{1}{4} \left( 2 b e^\frac{\mu_1-\mu_2}{b} + e^\frac{2 \mu_1 - 2 \mu_2}{b} \left( \mu_1 - \mu_2 - 2 b \right) \right). 
\end{align*}

Rewriting Equation~\ref{eq:laplace_lk_symmetric_b} for the part of the space where $\mu_1 < x < \mu_2, \mu_1 < y < \mu_2$ is
\begin{align*}
d_{\mu_1 < x \leq \mu_2, \mu_1 < y \leq \mu_2}(X, Y) &= \frac{1}{4 b^2} \int_{\mu_1}^{\mu_2} \int_{\mu_1}^{\mu_2} |y - x| \cdot e^{\frac{-(x - \mu_1)}{b}} \cdot e^{\frac{-(\mu_2 - y)}{b}} \, dx\, dy\\
 &= \frac{1}{4 b^2} \int_{\mu_1}^{\mu_2} \left( \int_{\mu_1}^{\mu_2} |y - x| \cdot e^{\frac{-(x - \mu_1)}{b}} \, dx \right) \cdot e^{\frac{-(\mu_2 - y)}{b}} \, dy\\
 &= \frac{1}{4 b^2} \int_{\mu_1}^{\mu_2} \left( \int_{\mu_1}^{y} (y - x) \cdot e^{\frac{-(x - \mu_1)}{b}} \, dx  + \int_{y}^{\mu_2} (x - y) \cdot e^{\frac{-(x - \mu_1)}{b}} \, dx \right) \cdot e^{\frac{-(\mu_2 - y)}{b}} \, dy\\
 &= \frac{1}{4 b^2} \int_{\mu_1}^{\mu_2} \left( \left( b^2 e^\frac{\mu_1 - y}{b} + b y - b \mu_1 - b^2 \right) + \left( b^2 e^\frac{\mu_1 - y}{b} + e^\frac{\mu_1 - \mu_2}{b} \left( b y - b \mu_2 - b^2 \right) \right) \right) \cdot e^{\frac{-(\mu_2 - y)}{b}} \, dy\\
 &= \frac{1}{4 b} \int_{\mu_1}^{\mu_2} \left( 2 b e^\frac{\mu_1 - y}{b} + y - \mu_1 - b + e^\frac{\mu_1 - \mu_2}{b} \left( y - \mu_2 - b \right) \right) \cdot e^{\frac{-(\mu_2 - y)}{b}} \, dy\\
 &= \frac{1}{4 b} \left( b \mu_2 - b \mu_1 - 2 b^2 + e^\frac{\mu_1 - \mu_2}{b} \left( 2 b \mu_2 - 2 b^2 \right) + e^\frac{\mu_1 - \mu_2}{b} \left( -2 b \mu_1 + 2 b^2\right) + e^\frac{2 \mu_1 - 2 \mu_2}{b} \left( b \mu_2 - b \mu_1 + 2 b^2 \right) \right) \\
 &= \frac{1}{4} \left( \mu_2 - \mu_1 - 2 b + e^\frac{\mu_1 - \mu_2}{b} \left( 2 \mu_2 - 2 \mu_1 \right) + e^\frac{2 \mu_1 - 2 \mu_2}{b} \left(\mu_2 - \mu_1 + 2 b \right) \right).
\end{align*}

Rewriting Equation~\ref{eq:laplace_lk_symmetric_b} for the part of the space where $\mu_1 < x < \mu_2, \mu_2 < y$ is
\begin{align*}
d_{\mu_1 < x \leq \mu_2, \mu_2 < y}(X, Y) &= \frac{1}{4 b^2} \int_{\mu_2}^{\infty} \int_{\mu_1}^{\mu_2} (y - x) \cdot e^{\frac{-(x - \mu_1)}{b}} \cdot e^{\frac{-(y - \mu_2)}{b}} \, dx\, dy\\
 &= \frac{1}{4 b^2} \int_{\mu_2}^{\infty} \left( \int_{\mu_1}^{\mu_2} (y - x) \cdot e^{\frac{-(x - \mu_1)}{b}} \, dx \right) \cdot e^{\frac{-(y - \mu_2)}{b}} \, dy\\
 &= \frac{1}{4 b^2} \int_{\mu_2}^{\infty} \left( b y - b \mu_1 - b^2 - e^\frac{\mu_1 - \mu_2}{b} \left(b y - b \mu_2 - b^2 \right) \right) \cdot e^{\frac{-(y - \mu_2)}{b}} \, dy\\
 &= \frac{1}{4 b} \int_{\mu_2}^{\infty} \left( y - \mu_1 - b - e^\frac{\mu_1 - \mu_2}{b} \left(y - \mu_2 - b \right) \right) \cdot e^{\frac{-(y - \mu_2)}{b}} \, dy\\
 &= \frac{1}{4 b} \left( b \mu_2 - b \mu_1 \right) \\
 &= \frac{1}{4} \left( \mu_2 - \mu_1 \right)
\end{align*}

Rewriting Equation~\ref{eq:laplace_lk_symmetric_b} for the part of the space where $\mu_2 < x, y < \mu_1$ is
\begin{align*}
d_{\mu_2 < x, y \leq \mu_1}(X, Y) &= \frac{1}{4 b^2} \int_{-\infty}^{\mu_1} \int_{\mu_2}^{\infty} (x - y) \cdot e^{\frac{-(x - \mu_1)}{b}} \cdot e^{\frac{-(\mu_2 - y)}{b}} \, dx\, dy\\
 &= \frac{1}{4 b^2} \int_{-\infty}^{\mu_1} \left( \int_{\mu_2}^{\infty} (x - y) \cdot e^{\frac{-(x - \mu_1)}{b}} \, dx \right) \cdot e^{\frac{-(\mu_2 - y)}{b}} \, dy\\
 &= \frac{1}{4 b^2} \int_{-\infty}^{\mu_1} e^\frac{\mu_1 - \mu_2}{b} \left( -b y + b \mu_2 + b^2 \right) \cdot e^{\frac{-(\mu_2 - y)}{b}} \, dy\\
 &= \frac{1}{4 b} \int_{-\infty}^{\mu_1} e^\frac{\mu_1 - \mu_2}{b} \left( -y + \mu_2 + b \right) \cdot e^{\frac{-(\mu_2 - y)}{b}} \, dy\\
 &= \frac{1}{4 b} e^\frac{2 \mu_1 - 2 \mu_2}{b} \left( b \mu_2 + 2 b^2 - b \mu_1 \right)\\
 &= \frac{1}{4} e^\frac{2 \mu_1 - 2 \mu_2}{b} \left( \mu_2 + 2 b - \mu_1 \right)\\
 &= \frac{1}{4} e^\frac{-2 (\mu_2 - \mu_1)}{b} \left(2 b + \mu_2 - \mu_1 \right)
\end{align*}

Rewriting Equation~\ref{eq:laplace_lk_symmetric_b} for the part of the space where $\mu_2 < x, \mu_1 < y < \mu_2$ is
\begin{align*}
d_{\mu_2 < x, \mu_1 < y \leq \mu_2}(X, Y) &= \frac{1}{4 b^2} \int_{\mu_1}^{\mu_2} \int_{\mu_2}^{\infty} (x - y) \cdot e^{\frac{-(x - \mu_1)}{b}} \cdot e^{\frac{-(\mu_2 - y)}{b}} \, dx\, dy\\
 &= \frac{1}{4 b^2} \int_{\mu_1}^{\mu_2} \left( \int_{\mu_2}^{\infty} (x - y) \cdot e^{\frac{-(x - \mu_1)}{b}} \, dx \right) \cdot e^{\frac{-(\mu_2 - y)}{b}} \, dy\\
 &= \frac{1}{4 b^2} \int_{\mu_1}^{\mu_2} \left( e^\frac{\mu_1-\mu_2}{b} \left( -b y + b \mu_2 + b^2 \right) \right) \cdot e^{\frac{-(\mu_2 - y)}{b}} \, dy\\
 &= \frac{1}{4 b} \int_{\mu_1}^{\mu_2} \left( e^\frac{\mu_1-\mu_2}{b} \left( -y + \mu_2 + b \right) \right) \cdot e^{\frac{-(\mu_2 - y)}{b}} \, dy\\
 &= \frac{1}{4 b} \left( 2 b^2 e^\frac{\mu_1 - \mu_2}{b} - e^\frac{2 \mu_1 - 2 \mu_2}{b} \left( b \mu_2 - b \mu_1 + 2 b^2 \right) \right) \\
 &= \frac{1}{4} \left( 2 b e^\frac{\mu_1 - \mu_2}{b} + e^\frac{2 \mu_1 - 2 \mu_2}{b} \left( \mu_1 - \mu_2 - 2 b \right) \right)
\end{align*}

Rewriting Equation~\ref{eq:laplace_lk_symmetric_b} for the part of the space where $\mu_2 < x, \mu_2 < y$ is
\begin{align*}
d_{\mu_2 < x, \mu_2 < y}(X, Y) &= \frac{1}{4 b^2} \int_{\mu_2}^{\infty} \int_{\mu_2}^{\infty} |x - y| \cdot e^{\frac{-(x - \mu_1)}{b}} \cdot e^{\frac{-(y - \mu_2)}{b}} \, dx\, dy\\
 &= \frac{1}{4 b^2} \int_{\mu_2}^{\infty} \left( \int_{\mu_2}^{\infty} |x - y| \cdot e^{\frac{-(x - \mu_1)}{b}} \, dx \right) \cdot e^{\frac{-(y - \mu_2)}{b}} \, dy\\
 &= \frac{1}{4 b^2} \int_{\mu_2}^{\infty} \left( \int_{\mu_2}^{y} (y - x) \cdot e^{\frac{-(x - \mu_1)}{b}} \, dx + \int_{y}^{\infty} (x - y) \cdot e^{\frac{-(x - \mu_1)}{b}} \, dx \right) \cdot e^{\frac{-(y - \mu_2)}{b}} \, dy\\
 &= \frac{1}{4 b^2} \int_{\mu_2}^{\infty} \left( \left( b^2 e^\frac{\mu_1 - y}{b} + e^\frac{\mu_1 - \mu_2}{b} \left(b y - b \mu_2 - b^2 \right) \right) + \left( b^2 e^\frac{\mu_1 - y}{b} \right) \right) \cdot e^{\frac{-(y - \mu_2)}{b}} \, dy\\
 &= \frac{1}{4 b} \int_{\mu_2}^{\infty} \left( 2 b e^\frac{\mu_1 - y}{b} + e^\frac{\mu_1 - \mu_2}{b} \left(y - \mu_2 - b \right) \right) \cdot e^{\frac{-(y - \mu_2)}{b}} \, dy\\
 &= \frac{1}{4 b} b^2 e^\frac{\mu_1 - \mu_2}{b}\\
 &= \frac{1}{4} b e^\frac{-(\mu_2 - \mu_1)}{b}.
\end{align*}

We can combine each of the probability weighted distances as
\begin{align*}
d(X, Y) &= d_{x \leq \mu_1, y \leq \mu_1}(X, Y) + d_{x \leq \mu_1, \mu_1 < y \leq \mu_2}(X, Y) + d_{x \leq \mu_1, \mu_2 < y}(X, Y) + d_{\mu_1 < x \leq \mu_2, y \leq \mu_1}(X, Y) + d_{\mu_1 < x \leq \mu_2, \mu_1 < y \leq \mu_2}(X, Y) \\
&\quad + d_{\mu_1 < x \leq \mu_2, \mu_2 < y}(X, Y) + d_{\mu_2 < x, y \leq \mu_1}(X, Y) + d_{\mu_2 < x, \mu_1 < y \leq \mu_2}(X, Y) + d_{\mu_2 < x, \mu_2 < y}(X, Y)\\
&= \frac{1}{4} b e^\frac{-(\mu_2 - \mu_1)}{b} \\
&\quad+ \frac{1}{4} \left( \mu_2 - \mu_1 \right) \\
&\quad+ \frac{1}{4} \left( \mu_2 - \mu_1 + 2 b \right) \\
&\quad+ \frac{1}{4} \left( 2 b e^\frac{\mu_1-\mu_2}{b} + e^\frac{2 \mu_1 - 2 \mu_2}{b} \left( \mu_1 - \mu_2 - 2 b \right) \right) \\
&\quad+ \frac{1}{4} \left( \mu_2 - \mu_1 - 2 b + e^\frac{\mu_1 - \mu_2}{b} \left( 2 \mu_2 - 2 \mu_1 \right) + e^\frac{2 \mu_1 - 2 \mu_2}{b} \left(\mu_2 - \mu_1 + 2 b \right) \right)\\
&\quad+ \frac{1}{4} \left( \mu_2 - \mu_1 \right) \\
&\quad+ \frac{1}{4} e^\frac{-2 (\mu_2 - \mu_1)}{b} \left(2 b + \mu_2 - \mu_1 \right) \\
&\quad+ \frac{1}{4} \left( 2 b e^\frac{\mu_1 - \mu_2}{b} + e^\frac{2 \mu_1 - 2 \mu_2}{b} \left( \mu_1 - \mu_2 - 2 b \right) \right) \\
&\quad+ \frac{1}{4} b e^\frac{-(\mu_2 - \mu_1)}{b} \\
&= b e^\frac{-(\mu_2 - \mu_1)}{b} \\
&\quad+ \frac{1}{2} \left( \mu_2 - \mu_1 + b \right) \\
&\quad+ \frac{1}{4} \left( 2 b e^\frac{\mu_1-\mu_2}{b} + e^\frac{2 \mu_1 - 2 \mu_2}{b} \left( \mu_1 - \mu_2 - 2 b \right) \right) \\
&\quad+ \frac{1}{4} \left( \mu_2 - \mu_1 - 2 b + e^\frac{\mu_1 - \mu_2}{b} \left( 2 \mu_2 - 2 \mu_1 \right) + e^\frac{2 \mu_1 - 2 \mu_2}{b} \left(\mu_2 - \mu_1 + 2 b \right) \right) \\
&\quad+ \frac{1}{4} \left( \mu_2 - \mu_1 \right) \\
&= b e^\frac{-(\mu_2 - \mu_1)}{b} \\
&\quad+ \left( \mu_2 - \mu_1 \right) \\
&\quad+ \frac{1}{4} \left( 2 b e^\frac{\mu_1-\mu_2}{b} \right) \\
&\quad+ \frac{1}{4} \left( e^\frac{\mu_1 - \mu_2}{b} \left( 2 \mu_2 - 2 \mu_1 \right)  \right) \\
&= b e^\frac{-(\mu_2 - \mu_1)}{b} + \left( \mu_2 - \mu_1 \right) + \frac{1}{2} b e^\frac{\mu_1-\mu_2}{b} + \frac{1}{2} e^\frac{\mu_1 - \mu_2}{b} \left( \mu_2 - \mu_1 \right) \\
d(X, Y) &= \mu_2 - \mu_1 + \frac{1}{2} e^\frac{-(\mu_2 - \mu_1)}{b} \left( 3 b + \mu_2 - \mu_1 \right).
\end{align*}

To remove the assumption that $\mu_1 \leq \mu_2$, we can rewrite this result as
\begin{equation}
\label{eq:lk_laplace}
d(X, Y) = |\mu_2 - \mu_1| + \frac{1}{2} e^\frac{-|\mu_2 - \mu_1|}{b} \left( 3 b + |\mu_2 - \mu_1| \right). 
\end{equation}

This completes the derivation of LK distance with Laplace Distributions.
\end{proof}

\subsection{Surprisal Approximates a Distance Metric At Scale}
\label{subsec:surprisal_as_distance}

\begin{proposition}
The function $d(x, y)$ defined as
$d(x, y) = \sum_{j \in J} \left( \frac{ |x_j - y_j| + \frac{1}{2} e^\frac{-|x_j - y_j|}{b_j} \left( 3 b_j + |x_j - y_j| \right) }{b_j} - 1.5 \right)$, is a symmetric, non-negative dissimilarity measure for $x_j, y_j \in \mathbb{R}$ and $b_j > 0$.  For differences significantly larger than $b_j$ ($|x_j - y_j| \gg b_j$), the function asymptotically approximates a standard weighted Manhattan distance metric as $d(x, y) \approx \sum_{j \in J} \frac{|x_j - y_j|}{b_j}$.
\end{proposition}

\begin{proof}
We analyze the properties of $d(x, y)$ by examining the individual components of the sum. Let $\Delta_j = |x_j - y_j|$ represent the absolute difference in dimension $j$, and let $u_j = \frac{\Delta_j}{b_j}$ represent the normalized distance relative to the parameter $b_j$. The term for each dimension can be expressed as
\begin{equation}
    f(u_j) = u_j + \frac{1}{2} e^{-u_j}(3 + u_j) - 1.5.
\end{equation}

\textbf{Symmetry:}
The distance function depends only on the term $\Delta_j = |x_j - y_j|$. Because the absolute difference is symmetric, $|x_j - y_j| = |y_j - x_j|$ for all $j \in J$. Therefore, $d(x, y) = d(y, x)$, satisfying the symmetry requirement for a dissimilarity measure.

\textbf{Nonnegativity and Identity of Indiscernibles:}
To satisfy these properties, we require $d(x, y) \geq 0$ and $d(x, y) = 0 \iff x = y$. 
As $\Delta_j \to 0$, the normalized distance $u_j$ approaches $0$. In this limit:
\begin{align*}
    f(u_j) &= 0 + \frac{1}{2} e^{0}(3+0) - 1.5 \\
    &= \frac{1}{2}(3) - 1.5 = 0.
\end{align*}
Thus, $d(x, y) = 0$ if and only if $\Delta_j = 0$ for all $j$, satisfying the identity of indiscernibles. Furthermore, since $f(u_j)$ is a convex function with its global minimum at $u_j = 0$, it follows that $f(u_j) \geq 0$ for all $u_j \geq 0$. Consequently, $d(x, y) \geq 0$ for all $x, y$, satisfying nonnegativity.

\textbf{Asymptotic Approximation:}
We examine the behavior of $d(x,y)$ as the distance becomes large relative to the scale parameter ($u_j \to \infty$). In this regime:
\begin{enumerate}
    \item The term $\frac{|x_j - y_j|}{b_j}$ grows linearly with the distance.
    \item The term $\frac{1}{2} e^{-u_j}(3 + u_j)$ decays exponentially toward zero because the exponential decay $e^{-u_j}$ dominates the linear growth of $(3+u_j)$.
\end{enumerate}
As $u_j \to \infty$, the exponential term becomes negligible, and the constant offset $-1.5$ remains fixed relative to the growing $|x_j - y_j|$. Thus, for $\Delta_j \gg b_j$:
\begin{equation}
    d(x,y) = \sum_{j \in J} (u_j + \text{negligible terms}) \approx \sum_{j \in J} \frac{|x_j - y_j|}{b_j}.
\end{equation}
The expression $\sum_{j \in J} \frac{|x_j - y_j|}{b_j}$ is a weighted Manhattan distance, which is a standard metric. Therefore, while the inclusion of the $e^{-u}$ term provides local surprisal smoothing (non-metric behavior in the neighborhood of zero), the function asymptotically converges to a valid distance metric as the scale of observation increases.
\end{proof}

\section{Numeric Example of Deriving a Prediction Using Iris Dataset}
This appendix shows each numerical step in predicting a class in the classic iris data set.

\begin{figure}[H]
  \centering
  \includegraphics[width=\linewidth,page=1]{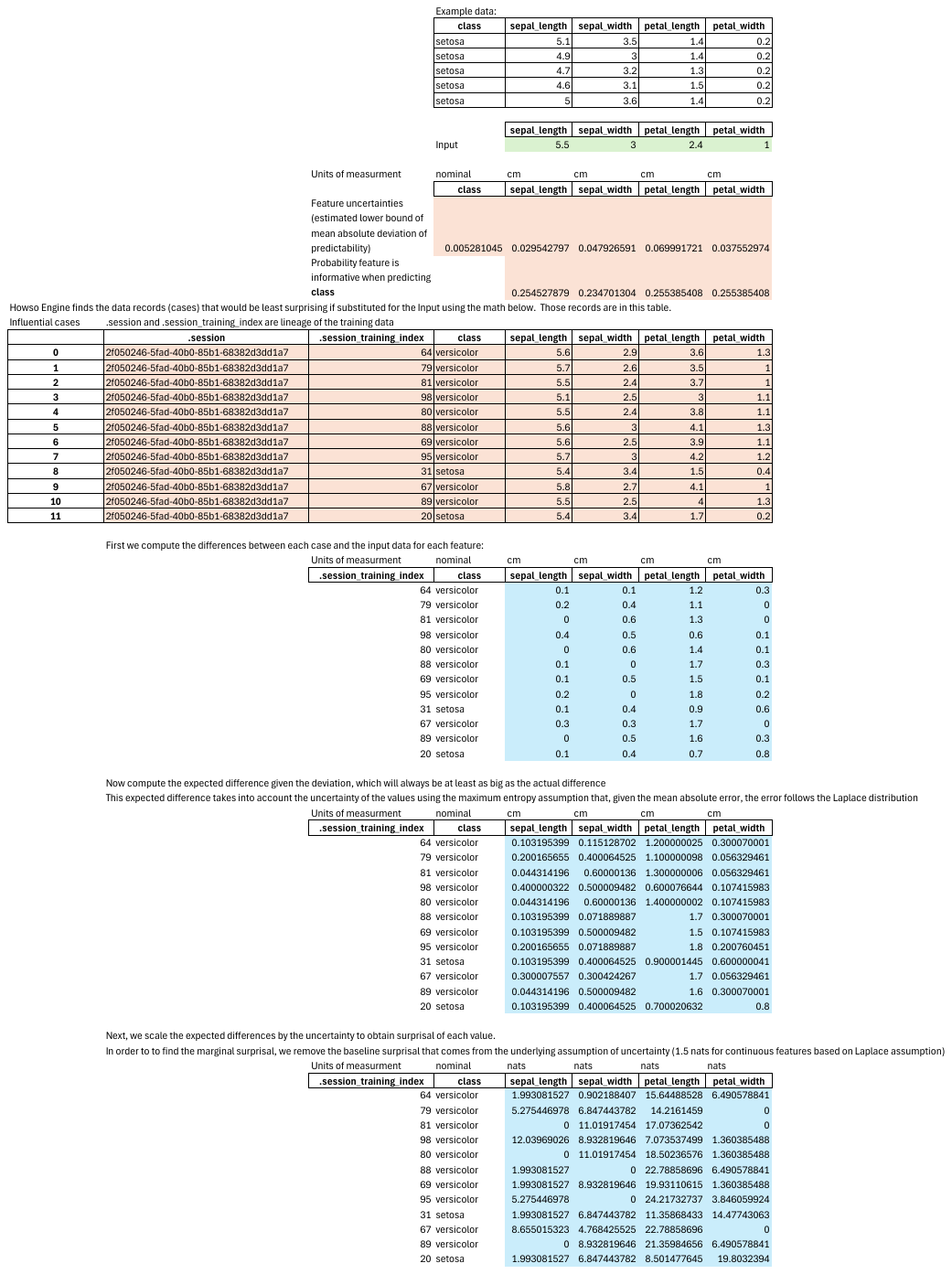}
\end{figure}

\begin{figure}[H]
  \centering
  \includegraphics[width=\linewidth,page=2]{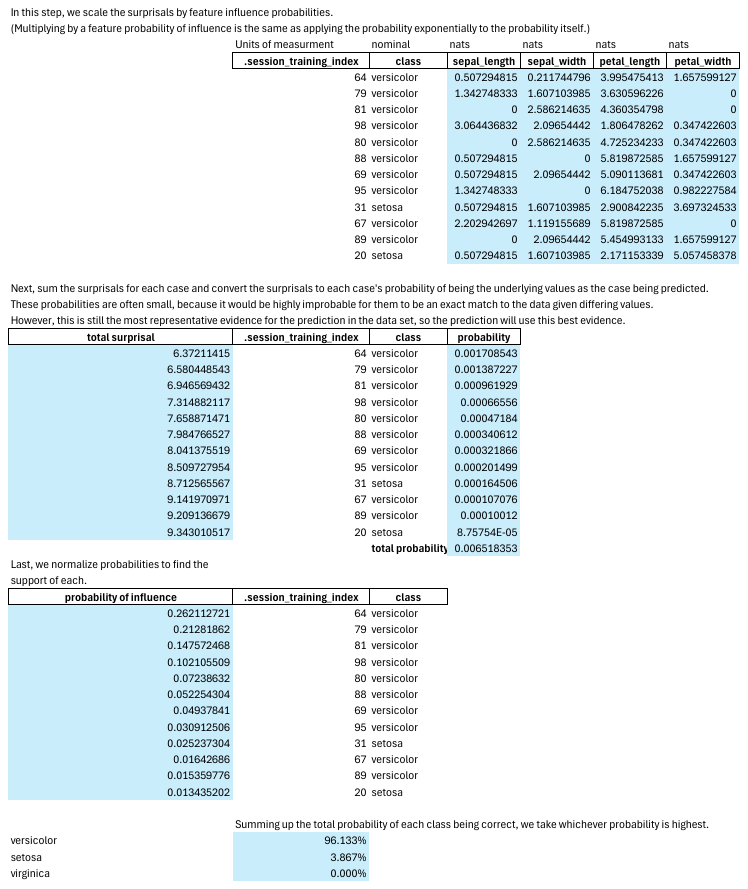}
\end{figure}

\section{Details of Empirical Evidence}
% Auto-generated supervised learning appendix tables
% Generated by generate_supervised_appendix.py
% Using pre-computed aggregate files (no math performed here)

\subsection{Supervised Learning Classification Results}
\label{ssec:appendix_supervised_learning}

\noindent Five small regression datasets (first\_principles\_ideal\_gas, first\_principles\_newton, 1089\_USCrime, 659\_sleuth\_ex1714, 485\_analcatdata\_vehicle) were excluded from all algorithms due to LightGBM failures (undefined Spearman correlation from constant predictions).

\subsubsection{Howso (Targetless)}

\begin{longtable}{lrrrr}
\caption{Classification Results across 146 Datasets}
\label{tab:howso_targetless_classification} \\
\toprule
Dataset & Accuracy & Precision & Recall & MCC \\
\midrule
\endfirsthead

\multicolumn{5}{c}{\textit{(Continued from previous page)}}\\ 
\toprule
Dataset & Accuracy & Precision & Recall & MCC \\
\midrule
\endhead

\midrule
\multicolumn{5}{r}{\textit{(Continued on next page)}} \\
\endfoot

\bottomrule
\endlastfoot

GAMETES\_Epistasis\_2\_Way\_1000atts\_0.4H\_EDM\_1\_EDM\_1\_1 & 0.5172 & 0.5202 & 0.5196 & 0.0398 \\
GAMETES\_Epistasis\_2\_Way\_20atts\_0.1H\_EDM\_1\_1 & 0.5792 & 0.5853 & 0.5803 & 0.1656 \\
GAMETES\_Epistasis\_2\_Way\_20atts\_0.4H\_EDM\_1\_1 & 0.7694 & 0.7797 & 0.7714 & 0.5510 \\
GAMETES\_Epistasis\_3\_Way\_20atts\_0.2H\_EDM\_1\_1 & 0.5528 & 0.5578 & 0.5541 & 0.1118 \\
GAMETES\_Heterogeneity\_20atts\_1600\_Het\_0.4\_0.2\_50\_EDM\_2\_001 & 0.6316 & 0.6351 & 0.6336 & 0.2687 \\
GAMETES\_Heterogeneity\_20atts\_1600\_Het\_0.4\_0.2\_75\_EDM\_2\_001 & 0.6809 & 0.6839 & 0.6821 & 0.3660 \\
Hill\_Valley\_with\_noise & 0.5263 & 0.5272 & 0.5268 & 0.0540 \\
Hill\_Valley\_without\_noise & 0.5588 & 0.5604 & 0.5601 & 0.1205 \\
\end{longtable}

\subsubsection{Howso (Single-Target)}

\begin{longtable}{lrrrr}
\caption{Classification Results across 146 Datasets}
\label{tab:howso_singletarget_classification} \\
\toprule
Dataset & Accuracy & Precision & Recall & MCC \\
\midrule
\endfirsthead

\multicolumn{5}{c}{\textit{(Continued from previous page)}}\\ 
\toprule
Dataset & Accuracy & Precision & Recall & MCC \\
\midrule
\endhead

\midrule
\multicolumn{5}{r}{\textit{(Continued on next page)}} \\
\endfoot

\bottomrule
\endlastfoot

GAMETES\_Epistasis\_2\_Way\_1000atts\_0.4H\_EDM\_1\_EDM\_1\_1 & 0.5164 & 0.5173 & 0.5144 & 0.0316 \\
GAMETES\_Epistasis\_2\_Way\_20atts\_0.1H\_EDM\_1\_1 & 0.6493 & 0.6604 & 0.6504 & 0.3106 \\
GAMETES\_Epistasis\_2\_Way\_20atts\_0.4H\_EDM\_1\_1 & 0.7910 & 0.8215 & 0.7917 & 0.6124 \\
GAMETES\_Epistasis\_3\_Way\_20atts\_0.2H\_EDM\_1\_1 & 0.6569 & 0.6598 & 0.6577 & 0.3175 \\
GAMETES\_Heterogeneity\_20atts\_1600\_Het\_0.4\_0.2\_50\_EDM\_2\_001 & 0.6806 & 0.6810 & 0.6805 & 0.3615 \\
GAMETES\_Heterogeneity\_20atts\_1600\_Het\_0.4\_0.2\_75\_EDM\_2\_001 & 0.7188 & 0.7295 & 0.7192 & 0.4485 \\
Hill\_Valley\_with\_noise & 0.5203 & 0.5261 & 0.5233 & 0.0493 \\
Hill\_Valley\_without\_noise & 0.5734 & 0.5769 & 0.5765 & 0.1534 \\
\end{longtable}

\subsubsection{XGBoost}

\begin{longtable}{lrrrr}
\caption{Classification Results across 146 Datasets}
\label{tab:xgboost_classification} \\
\toprule
Dataset & Accuracy & Precision & Recall & MCC \\
\midrule
\endfirsthead

\multicolumn{5}{c}{\textit{(Continued from previous page)}}\\ 
\toprule
Dataset & Accuracy & Precision & Recall & MCC \\
\midrule
\endhead

\midrule
\multicolumn{5}{r}{\textit{(Continued on next page)}} \\
\endfoot

\bottomrule
\endlastfoot

GAMETES\_Epistasis\_2\_Way\_1000atts\_0.4H\_EDM\_1\_EDM\_1\_1 & 0.4930 & 0.4941 & 0.4941 & -0.0118 \\
GAMETES\_Epistasis\_2\_Way\_20atts\_0.1H\_EDM\_1\_1 & 0.5920 & 0.5916 & 0.5914 & 0.1831 \\
GAMETES\_Epistasis\_2\_Way\_20atts\_0.4H\_EDM\_1\_1 & 0.7339 & 0.7343 & 0.7342 & 0.4686 \\
GAMETES\_Epistasis\_3\_Way\_20atts\_0.2H\_EDM\_1\_1 & 0.5603 & 0.5606 & 0.5607 & 0.1213 \\
GAMETES\_Heterogeneity\_20atts\_1600\_Het\_0.4\_0.2\_50\_EDM\_2\_001 & 0.6393 & 0.6391 & 0.6392 & 0.2783 \\
GAMETES\_Heterogeneity\_20atts\_1600\_Het\_0.4\_0.2\_75\_EDM\_2\_001 & 0.6813 & 0.6816 & 0.6817 & 0.3633 \\
Hill\_Valley\_with\_noise & 0.5556 & 0.5555 & 0.5555 & 0.1110 \\
Hill\_Valley\_without\_noise & 0.6032 & 0.6045 & 0.6042 & 0.2087 \\
\end{longtable}

\subsubsection{LightGBM}

\begin{longtable}{lrrrr}
\caption{Classification Results across 146 Datasets}
\label{tab:lightgbm_classification} \\
\toprule
Dataset & Accuracy & Precision & Recall & MCC \\
\midrule
\endfirsthead

\multicolumn{5}{c}{\textit{(Continued from previous page)}}\\ 
\toprule
Dataset & Accuracy & Precision & Recall & MCC \\
\midrule
\endhead

\midrule
\multicolumn{5}{r}{\textit{(Continued on next page)}} \\
\endfoot

\bottomrule
\endlastfoot

GAMETES\_Epistasis\_2\_Way\_1000atts\_0.4H\_EDM\_1\_EDM\_1\_1 & 0.4916 & 0.4924 & 0.4924 & -0.0151 \\
GAMETES\_Epistasis\_2\_Way\_20atts\_0.1H\_EDM\_1\_1 & 0.6112 & 0.6117 & 0.6115 & 0.2232 \\
GAMETES\_Epistasis\_2\_Way\_20atts\_0.4H\_EDM\_1\_1 & 0.7288 & 0.7301 & 0.7283 & 0.4583 \\
GAMETES\_Epistasis\_3\_Way\_20atts\_0.2H\_EDM\_1\_1 & 0.5522 & 0.5528 & 0.5528 & 0.1055 \\
GAMETES\_Heterogeneity\_20atts\_1600\_Het\_0.4\_0.2\_50\_EDM\_2\_001 & 0.6666 & 0.6673 & 0.6669 & 0.3343 \\
GAMETES\_Heterogeneity\_20atts\_1600\_Het\_0.4\_0.2\_75\_EDM\_2\_001 & 0.6919 & 0.6919 & 0.6914 & 0.3834 \\
Hill\_Valley\_with\_noise & 0.5988 & 0.5987 & 0.5985 & 0.1973 \\
Hill\_Valley\_without\_noise & 0.5712 & 0.5707 & 0.5704 & 0.1410 \\
\end{longtable}

\clearpage
\subsection{Supervised Learning Regression Results}

\subsubsection{Howso (Targetless)}

\begin{longtable}{lrr}
\caption{Regression Results across 258 Datasets}
\label{tab:howso_targetless_regression} \\
\toprule
Dataset & $R^2$ & Spearman \\
\midrule
\endfirsthead

\multicolumn{3}{c}{\textit{(Continued from previous page)}}\\ 
\toprule
Dataset & $R^2$ & Spearman \\
\midrule
\endhead

\midrule
\multicolumn{3}{r}{\textit{(Continued on next page)}} \\
\endfoot

\bottomrule
\endlastfoot

1027\_ESL & 0.8033 & 0.9118 \\
1028\_SWD & 0.2067 & 0.5568 \\
1029\_LEV & 0.4665 & 0.7092 \\
1030\_ERA & 0.3372 & 0.5449 \\
1089\_USCrime & 0.7240 & 0.7908 \\
1096\_FacultySalaries & 0.5929 & 0.8690 \\
1191\_BNG\_pbc & 0.3086 & 0.5740 \\
1193\_BNG\_lowbwt & 0.5818 & 0.7210 \\
\end{longtable}

\subsubsection{Howso (Single-Target)}

\begin{longtable}{lrr}
\caption{Regression Results across 258 Datasets}
\label{tab:howso_singletarget_regression} \\
\toprule
Dataset & $R^2$ & Spearman \\
\midrule
\endfirsthead

\multicolumn{3}{c}{\textit{(Continued from previous page)}}\\ 
\toprule
Dataset & $R^2$ & Spearman \\
\midrule
\endhead

\midrule
\multicolumn{3}{r}{\textit{(Continued on next page)}} \\
\endfoot

\bottomrule
\endlastfoot

1027\_ESL & 0.8081 & 0.9116 \\
1028\_SWD & 0.2573 & 0.5668 \\
1029\_LEV & 0.4653 & 0.7088 \\
1030\_ERA & 0.3304 & 0.5447 \\
1089\_USCrime & 0.7283 & 0.7728 \\
1096\_FacultySalaries & 0.6008 & 0.8553 \\
1191\_BNG\_pbc & 0.3829 & 0.6331 \\
1193\_BNG\_lowbwt & 0.6047 & 0.7484 \\
\end{longtable}

\subsubsection{XGBoost}

\begin{longtable}{lrr}
\caption{Regression Results across 258 Datasets}
\label{tab:xgboost_regression} \\
\toprule
Dataset & $R^2$ & Spearman \\
\midrule
\endfirsthead

\multicolumn{3}{c}{\textit{(Continued from previous page)}}\\ 
\toprule
Dataset & $R^2$ & Spearman \\
\midrule
\endhead

\midrule
\multicolumn{3}{r}{\textit{(Continued on next page)}} \\
\endfoot

\bottomrule
\endlastfoot

1027\_ESL & 0.8024 & 0.9148 \\
1028\_SWD & 0.3270 & 0.5942 \\
1029\_LEV & 0.5360 & 0.7324 \\
1030\_ERA & 0.3501 & 0.5588 \\
1089\_USCrime & 0.5860 & 0.7279 \\
1096\_FacultySalaries & 0.6673 & 0.8437 \\
1191\_BNG\_pbc & 0.4416 & 0.6771 \\
1193\_BNG\_lowbwt & 0.5954 & 0.7430 \\
\end{longtable}

\subsubsection{LightGBM}

\begin{longtable}{lrr}
\caption{Regression Results across 253 Datasets}
\label{tab:lightgbm_regression} \\
\toprule
Dataset & $R^2$ & Spearman \\
\midrule
\endfirsthead

\multicolumn{3}{c}{\textit{(Continued from previous page)}}\\ 
\toprule
Dataset & $R^2$ & Spearman \\
\midrule
\endhead

\midrule
\multicolumn{3}{r}{\textit{(Continued on next page)}} \\
\endfoot

\bottomrule
\endlastfoot

1027\_ESL & 0.8477 & 0.9292 \\
1028\_SWD & 0.3713 & 0.6231 \\
1029\_LEV & 0.5498 & 0.7410 \\
1030\_ERA & 0.3559 & 0.5597 \\
1096\_FacultySalaries & 0.1822 & 0.6449 \\
1191\_BNG\_pbc & 0.4284 & 0.6674 \\
1193\_BNG\_lowbwt & 0.6136 & 0.7585 \\
1196\_BNG\_pharynx & 0.5136 & 0.6685 \\
\end{longtable}

\clearpage

\subsection{Feature Importance Validation Results}
\label{sec:appendix_roar}

This section provides detailed results from the ROAR (RemOve And Retrain) validation~\citep{hooker2019benchmark} comparing Howso's feature importance methods (Prediction Contributions and Accuracy Contributions) with SHAP~\citep{lundberg2017unifiedapproachinterpretingmodel}.

\begin{table}[!ht]
\centering
\caption{ROAR Comparison Results: Howso Prediction Contributions, Howso Accuracy Contributions, and SHAP Performance by Dataset and Removal Level}
\label{tab:roar_detailed}
\begin{tabular}{llcrrrrrr}
\toprule
Dataset & Task & Method & Removed & Baseline & After & Change & Relative & Seeds \\
\midrule
\multirow{12}{*}{diabetes} & \multirow{12}{*}{Regression (MAE)} & Prediction Contrib. & 0\% (0) & 49.3 & 49.3 & 0.0 & 0.0\% & 3 \\
                          &                                    & Prediction Contrib. & 10\% (1) & 49.3 & 52.3 & +3.6 & +7.2\% & 3 \\
                          &                                    & Prediction Contrib. & 20\% (2) & 49.3 & 55.7 & +6.4 & +13.0\% & 3 \\
                          &                                    & Prediction Contrib. & 50\% (5) & 49.3 & 61.3 & +12.0 & +24.4\% & 3 \\
\cmidrule{3-9}
                          &                                    & Accuracy Contrib. & 0\% (0) & 49.3 & 49.3 & 0.0 & 0.0\% & 3 \\
                          &                                    & Accuracy Contrib. & 10\% (1) & 49.3 & 52.3 & +3.0 & +6.1\% & 3 \\
                          &                                    & Accuracy Contrib. & 20\% (2) & 49.3 & 55.7 & +6.4 & +13.0\% & 3 \\
\bottomrule
\end{tabular}

\vspace{0.3cm}
\footnotesize
\textit{Note}: Values are medians across 3 random seeds. Numbers in parentheses indicate the count of features removed. MAE = Mean Absolute Error (lower is better). Acc = Accuracy (higher is better). pp = percentage points. All tests performed on XGBoost models.
\end{table}

% Auto-generated anomaly detection appendix table
% Generated by generate_anomaly_appendix.py

\subsection{Anomaly Detection Results}

\label{ssec:appendix_anomaly}

\begin{longtable}{llrr}
\caption{Anomaly Detection Results across Datasets}
\label{tab:anomaly_detection_results} \\
\toprule
Metric & Dataset & PR-AUC & ROC-AUC \\
\midrule
\endfirsthead

\multicolumn{4}{c}{\textit{(Continued from previous page)}} \\
\toprule
Metric & Dataset & PR-AUC & ROC-AUC \\
\midrule
\endhead

\midrule
\multicolumn{4}{r}{\textit{(Continued on next page)}} \\
\endfoot

\bottomrule
\endlastfoot

CBLOF & 10\_cover & 0.0932 & 0.5130 \\
CBLOF & 12\_fault & 0.2136 & 0.5605 \\
CBLOF & 14\_glass & 0.0621 & 0.6300 \\
CBLOF & 15\_Hepatitis & 0.1970 & 0.6559 \\
CBLOF & 16\_http & 0.0200 & 0.4207 \\
CBLOF & 18\_Ionosphere & 0.2368 & 0.7047 \\
CBLOF & 19\_landsat & 0.2585 & 0.7408 \\
CBLOF & 1\_ALOI & 0.0328 & 0.4428 \\
\end{longtable}

\clearpage

% Auto-generated RL results appendix tables
% Generated by generate_rl_appendix.py

\subsection{Reinforcement Learning: Detailed Results}
\label{ssec:appendix_rl}

\subsubsection{CartPole Results}

\begin{table}[!ht]
    \centering
\begin{tabular}{lr}
\toprule
Metric & Value \\
\midrule
Total Runs & 5 \\
Winning Runs & 5 \\
Win Rate (\%) & 100.00 \\
Median Rounds to Win & 273 \\
IQR Rounds to Win & 38.00 \\
Median High Score & 500.00 \\
Median Total Cases & 26564 \\
\bottomrule
\end{tabular}
    \caption{CartPole Summary Statistics}
    \label{tab:rl_cartpole_summary}
\end{table}

\begin{table}[!ht]
    \centering
    \scriptsize
\begin{tabular}{lcccccc}
\toprule
Iteration & Win & Rounds & Avg Score (Final) & Avg Score (Total) & High Score & Total Cases \\
\midrule
0 & Yes & 100 & 252.42 & 252.42 & 500 & 22191 \\
1 & Yes & 273 & 195.09 & 116.14 & 500 & 25564 \\
2 & Yes & 791 & 196.79 & 123.35 & 500 & 71117 \\
3 & Yes & 287 & 195.11 & 143.84 & 500 & 30427 \\
4 & Yes & 249 & 196.16 & 136.71 & 500 & 26564 \\
\bottomrule
\end{tabular}
    \caption{CartPole Detailed Run Results}
    \label{tab:rl_cartpole_runs}
\end{table}

\subsubsection{Wafer-Thin-Mints Results}

\begin{table}[!ht]
    \centering
\begin{tabular}{lr}
\toprule
Metric & Value \\
\midrule
Total Runs & 20 \\
Winning Runs & 14 \\
Win Rate (\%) & 70.00 \\
Median Rounds to Win & 150.00 \\
IQR Rounds to Win & 0.00 \\
Median High Score & 9.00 \\
Median Total Cases & 1189.50 \\
\bottomrule
\end{tabular}
    \caption{Wafer-Thin-Mints Summary Statistics}
    \label{tab:rl_wtm_summary}
\end{table}

\begin{table}[!ht]
    \centering
    \scriptsize
\begin{tabular}{lcccccc}
\toprule
Iteration & Win & Rounds & Avg Score (Final) & Avg Score (Total) & High Score & Total Cases \\
\midrule
0 & No & 150 & 3.92 & 3.92 & 5 & 738 \\
1 & Yes & 150 & 6.15 & 6.15 & 9 & 1130 \\
2 & No & 150 & 4.43 & 4.43 & 8 & 815 \\
3 & No & 150 & 5.74 & 5.74 & 9 & 1182 \\
4 & Yes & 150 & 7.31 & 7.31 & 9 & 1342 \\
5 & Yes & 150 & 7.44 & 7.44 & 9 & 1361 \\
6 & Yes & 150 & 6.99 & 6.99 & 9 & 1237 \\
7 & Yes & 150 & 7.10 & 7.10 & 9 & 1329 \\
8 & Yes & 150 & 6.97 & 6.97 & 9 & 1253 \\
9 & Yes & 150 & 7.10 & 7.10 & 9 & 1291 \\
10 & Yes & 150 & 6.35 & 6.35 & 9 & 1197 \\
11 & Yes & 150 & 7.00 & 7.00 & 9 & 1352 \\
12 & Yes & 150 & 6.20 & 6.20 & 9 & 1156 \\
13 & Yes & 150 & 6.53 & 6.53 & 9 & 1168 \\
14 & No & 150 & 5.01 & 5.01 & 6 & 901 \\
15 & No & 150 & 4.51 & 4.51 & 6 & 845 \\
16 & Yes & 150 & 6.92 & 6.92 & 9 & 1264 \\
17 & Yes & 150 & 7.50 & 7.50 & 9 & 1351 \\
18 & No & 150 & 1.52 & 1.52 & 2 & 378 \\
19 & Yes & 150 & 6.65 & 6.65 & 9 & 1166 \\
\bottomrule
\end{tabular}
    \caption{Wafer-Thin-Mints Detailed Run Results}
    \label{tab:rl_wtm_runs}
\end{table}

\clearpage

% Auto-generated synthetic data generation appendix tables
% Generated by generate_synth_results.py

\subsection{Synthetic Data Generation: Detailed Results}
\label{ssec:appendix_synth}

\begin{longtable}{lcc}
\caption{Synthetic Data Generation: Regression Datasets (Spearman Correlation)}
\label{tab:synth_continuous} \\
\toprule
Dataset & Original & Generated \\
\midrule
\endfirsthead

\multicolumn{3}{c}{\textit{(Continued from previous page)}}\\ 
\toprule
Dataset & Original & Generated \\
\midrule
\endhead

\midrule
\multicolumn{3}{r}{\textit{(Continued on next page)}} \\
\endfoot

\bottomrule
\endlastfoot

1199\_BNG\_echoMonths & 0.8296 & 0.8083 \\
197\_cpu\_act & 0.9832 & 0.9754 \\
537\_houses & 0.9553 & 0.9464 \\
574\_house\_16H & 0.9144 & 0.9018 \\
\end{longtable}

\vspace{1em}

\begin{longtable}{lcccccc}
\caption{Synthetic Data Generation: Classification Datasets}
\label{tab:synth_nominal} \\
\toprule
Dataset & \multicolumn{2}{c}{Accuracy} & \multicolumn{2}{c}{F1-Score} & \multicolumn{2}{c}{MCC} \\
        & Orig & Gen & Orig & Gen & Orig & Gen \\
\midrule
\endfirsthead

\multicolumn{7}{c}{\textit{(Continued from previous page)}}\\ 
\toprule
Dataset & \multicolumn{2}{c}{Accuracy} & \multicolumn{2}{c}{F1-Score} & \multicolumn{2}{c}{MCC} \\
        & Orig & Gen & Orig & Gen & Orig & Gen \\
\midrule
\endhead

\midrule
\multicolumn{7}{r}{\textit{(Continued on next page)}} \\
\endfoot

\bottomrule
\endlastfoot

letter & 0.953 & 0.901 & 0.954 & 0.902 & 0.952 & 0.897 \\
magic & 0.879 & 0.851 & 0.879 & 0.852 & 0.732 & 0.667 \\
mushroom & 1.000 & 1.000 & 1.000 & 1.000 & 1.000 & 0.999 \\
pendigits & 0.992 & 0.986 & 0.992 & 0.986 & 0.992 & 0.984 \\
ring & 0.971 & 0.886 & 0.971 & 0.896 & 0.942 & 0.791 \\
\end{longtable}

% Auto-generated robustness results appendix tables
% Generated by generate_robustness_appendix.py

\subsection{Adversarial Robustness Results}
\label{sec:appendix_robustness}

\subsubsection{Aggregate Results}

\begin{table}[!ht]
\centering
\caption{Mean accuracy degradation under adversarial attack across 27 datasets}
\label{tab:robustness_aggregate}
\begin{tabular}{lcccc}
\toprule
Attack & Model & Clean Acc & Adv Acc & $\Delta$ (Adv - Clean) \\
\midrule
\multirow{6}{*}{K=3 (approx)} & Decision Tree & 0.783 & 0.564 & $-0.220$ \\
                               & Howso & 0.819 & 0.596 & $-0.223$ \\
                               & k-NN & 0.796 & 0.436 & $-0.360$ \\
                               & LightGBM & 0.835 & 0.619 & $-0.216$ \\
                               & Logistic Regression & 0.795 & 0.684 & $-0.111$ \\
                               & XGBoost & 0.834 & 0.633 & $-0.201$ \\
\midrule
\multirow{6}{*}{K=1 (exact)} & Decision Tree & 0.783 & 0.641 & $-0.142$ \\
                             & Howso & 0.819 & 0.737 & $-0.082$ \\
                             & k-NN & 0.796 & 0.757 & $-0.039$ \\
                             & LightGBM & 0.835 & 0.717 & $-0.119$ \\
                             & Logistic Regression & 0.795 & 0.742 & $-0.053$ \\
                             & XGBoost & 0.834 & 0.691 & $-0.143$ \\
\bottomrule
\end{tabular}
\end{table}

\begin{table}[!ht]
\centering
\caption{Summary statistics of accuracy degradation by model and attack}
\label{tab:robustness_summary_stats}
\begin{tabular}{llcccccc}
\toprule
Attack & Model & Median $\Delta$ & Mean $\Delta$ & Min $\Delta$ & Q1 & Q3 & Worst-5 Mean \\
\midrule
\multirow{6}{*}{K=3} & Dec. Tree & $-0.167$ & $-0.220$ & $-0.530$ & $-0.341$ & $-0.104$ & $-0.441$ \\
                     & Howso & $-0.179$ & $-0.223$ & $-0.550$ & $-0.349$ & $-0.070$ & $-0.482$ \\
                     & k-NN & $-0.333$ & $-0.360$ & $-0.803$ & $-0.460$ & $-0.163$ & $-0.728$ \\
                     & LightGBM & $-0.166$ & $-0.216$ & $-0.535$ & $-0.308$ & $-0.116$ & $-0.476$ \\
                     & Log. Reg. & $-0.067$ & $-0.111$ & $-0.463$ & $-0.167$ & $-0.021$ & $-0.308$ \\
                     & XGBoost & $-0.195$ & $-0.201$ & $-0.550$ & $-0.271$ & $-0.090$ & $-0.479$ \\
\midrule
\multirow{6}{*}{K=1} & Dec. Tree & $-0.097$ & $-0.142$ & $-0.420$ & $-0.193$ & $-0.052$ & $-0.345$ \\
                     & Howso & $-0.056$ & $-0.082$ & $-0.420$ & $-0.107$ & $-0.016$ & $-0.248$ \\
                     & k-NN & $-0.023$ & $-0.039$ & $-0.233$ & $-0.052$ & $0.000$ & $-0.123$ \\
                     & LightGBM & $-0.083$ & $-0.119$ & $-0.420$ & $-0.183$ & $-0.033$ & $-0.303$ \\
                     & Log. Reg. & $-0.038$ & $-0.053$ & $-0.333$ & $-0.078$ & $0.000$ & $-0.161$ \\
                     & XGBoost & $-0.119$ & $-0.143$ & $-0.450$ & $-0.212$ & $-0.048$ & $-0.334$ \\
\bottomrule
\end{tabular}
\end{table}

\subsubsection{Per-Dataset Results}

Complete per-dataset results for Howso are provided below.

\begin{table}[!ht]
\centering
\caption{Howso per-dataset robustness results}
\label{tab:robustness_per_dataset_howso}
\scriptsize
\begin{tabular}{llcccc}
\toprule
Dataset & Attack & Clean Acc & Adv Acc & $\Delta$ & N Test \\
\midrule
biomed & K=3 & 1.000 & 0.833 & $-0.167$ & 42 \\
breast\_w & K=3 & 0.950 & 0.771 & $-0.179$ & 140 \\
cars & K=3 & 1.000 & 0.696 & $-0.304$ & 79 \\
cleveland & K=3 & 0.443 & 0.262 & $-0.180$ & 61 \\
cloud & K=3 & 0.364 & 0.318 & $-0.045$ & 22 \\
confidence & K=3 & 0.933 & 0.533 & $-0.400$ & 15 \\
d\_breast\_cancer & K=3 & 0.974 & 0.667 & $-0.307$ & 114 \\
\bottomrule
\end{tabular}
\end{table}

\clearpage

\end{document}